\newif\ifarxiv
\arxivtrue

\newif\ifopt
\optfalse 

\ifarxiv

\newif\ificlrfinal
\iclrfinaltrue

\documentclass{article}

\usepackage[margin=2.25cm]{geometry}
\usepackage{palatino}
\usepackage[longnamesfirst]{natbib}
\else
\ifopt

\documentclass[anon]{opt2026} 

\else

\documentclass{article} 
\usepackage{iclr2027_conference,times}

\fi
\fi

\usepackage{microtype}

\usepackage{array}
\usepackage{booktabs}
\usepackage{multirow}

\ifopt\else
\ifarxiv
\usepackage[pagebackref,colorlinks,linkcolor=purple,citecolor=teal,urlcolor=violet]{hyperref}
\else
\usepackage[pagebackref]{hyperref}
\fi
\usepackage{url}
\fi

\usepackage{doi}
\usepackage[shortlabels]{enumitem}

\usepackage{mismath}
\usepackage{amsfonts}
\ifopt\else
\usepackage{amssymb}
\usepackage{amsthm}
\fi
\usepackage{thm-restate}
\usepackage{yhmath}

\usepackage{cleveref}

\usepackage{pgfplots}
\pgfplotsset{compat=1.18}
\ifopt\else
\makeatletter
\@ifundefined{ICLRtikzexternal}{}{\iclrfinaltrue}
\makeatother
\fi
\usepgfplotslibrary{external}
\tikzset{
  external/system call={
    lualatex \tikzexternalcheckshellescape
    -halt-on-error
    -interaction=batchmode
    -jobname "\image"
    \ifopt
    "\texsource"
    \else
    "\string\def\string\ICLRtikzexternal{1}\texsource"
    \fi
  }
}
\tikzexternaldisable
\usepgfplotslibrary{fillbetween}

\pgfplotsset{
  colormap={warpedviridis}{
    rgb255=(68,1,84)
    rgb255=(72,32,113)
    rgb255=(71,45,123)
    rgb255=(69,53,129)
    rgb255=(67,61,132)
    rgb255=(65,66,135)
    rgb255=(63,72,137)
    rgb255=(61,78,138)
    rgb255=(59,82,139)
    rgb255=(57,86,140)
    rgb255=(55,90,140)
    rgb255=(53,94,141)
    rgb255=(52,97,141)
    rgb255=(50,100,142)
    rgb255=(49,103,142)
    rgb255=(48,106,142)
    rgb255=(46,109,142)
    rgb255=(45,112,142)
    rgb255=(44,114,142)
    rgb255=(43,117,142)
    rgb255=(42,119,142)
    rgb255=(41,122,142)
    rgb255=(40,124,142)
    rgb255=(39,126,142)
    rgb255=(38,129,142)
    rgb255=(38,130,142)
    rgb255=(37,132,142)
    rgb255=(36,135,142)
    rgb255=(35,137,142)
    rgb255=(34,139,141)
    rgb255=(34,141,141)
    rgb255=(33,143,141)
    rgb255=(33,145,140)
    rgb255=(32,146,140)
    rgb255=(31,148,140)
    rgb255=(31,150,139)
    rgb255=(31,152,139)
    rgb255=(31,154,138)
    rgb255=(30,156,137)
    rgb255=(30,157,137)
    rgb255=(31,159,136)
    rgb255=(31,161,136)
    rgb255=(31,162,135)
    rgb255=(32,163,134)
    rgb255=(33,165,133)
    rgb255=(34,167,133)
    rgb255=(35,169,131)
    rgb255=(36,170,131)
    rgb255=(37,172,130)
    rgb255=(39,173,129)
    rgb255=(40,174,128)
    rgb255=(42,176,127)
    rgb255=(44,177,126)
    rgb255=(46,179,124)
    rgb255=(47,180,124)
    rgb255=(50,182,122)
    rgb255=(52,182,121)
    rgb255=(55,184,120)
    rgb255=(58,186,118)
    rgb255=(59,187,117)
    rgb255=(61,188,116)
    rgb255=(64,189,114)
    rgb255=(66,190,113)
    rgb255=(70,192,111)
    rgb255=(72,193,110)
    rgb255=(76,194,108)
    rgb255=(78,195,107)
    rgb255=(80,196,106)
    rgb255=(84,197,104)
    rgb255=(86,198,103)
    rgb255=(90,200,100)
    rgb255=(92,200,99)
    rgb255=(94,201,98)
    rgb255=(99,203,95)
    rgb255=(101,203,94)
    rgb255=(103,204,92)
    rgb255=(108,205,90)
    rgb255=(110,206,88)
    rgb255=(112,207,87)
    rgb255=(115,208,86)
    rgb255=(119,209,83)
    rgb255=(122,209,81)
    rgb255=(124,210,80)
    rgb255=(127,211,78)
    rgb255=(132,212,75)
    rgb255=(134,213,73)
    rgb255=(137,213,72)
    rgb255=(139,214,70)
    rgb255=(144,215,67)
    rgb255=(147,215,65)
    rgb255=(149,216,64)
    rgb255=(152,216,62)
    rgb255=(155,217,60)
    rgb255=(160,218,57)
    rgb255=(162,218,55)
    rgb255=(165,219,54)
    rgb255=(168,219,52)
    rgb255=(170,220,50)
    rgb255=(173,220,48)
    rgb255=(178,221,45)
    rgb255=(181,222,43)
    rgb255=(184,222,41)
    rgb255=(186,222,40)
    rgb255=(189,223,38)
    rgb255=(192,223,37)
    rgb255=(194,223,35)
    rgb255=(197,224,33)
    rgb255=(200,224,32)
    rgb255=(205,225,29)
    rgb255=(208,225,28)
    rgb255=(210,226,27)
    rgb255=(213,226,26)
    rgb255=(216,226,25)
    rgb255=(218,227,25)
    rgb255=(221,227,24)
    rgb255=(223,227,24)
    rgb255=(226,228,24)
    rgb255=(229,228,25)
    rgb255=(231,228,25)
    rgb255=(234,229,26)
    rgb255=(236,229,27)
    rgb255=(239,229,28)
    rgb255=(241,229,29)
    rgb255=(244,230,30)
    rgb255=(246,230,32)
    rgb255=(248,230,33)
    rgb255=(251,231,35)
    rgb255=(253,231,37)
  }
}

\pgfplotsset{
  colormap={warped3viridis}{
    rgb255=(68,1,84)
    rgb255=(65,66,135)
    rgb255=(59,82,139)
    rgb255=(54,92,141)
    rgb255=(51,99,141)
    rgb255=(48,106,142)
    rgb255=(46,111,142)
    rgb255=(44,115,142)
    rgb255=(42,119,142)
    rgb255=(41,123,142)
    rgb255=(39,127,142)
    rgb255=(38,130,142)
    rgb255=(37,133,142)
    rgb255=(35,136,142)
    rgb255=(34,139,141)
    rgb255=(33,142,141)
    rgb255=(33,145,140)
    rgb255=(32,146,140)
    rgb255=(31,149,139)
    rgb255=(31,151,139)
    rgb255=(31,154,138)
    rgb255=(30,156,137)
    rgb255=(31,158,137)
    rgb255=(31,160,136)
    rgb255=(31,161,135)
    rgb255=(32,163,134)
    rgb255=(33,165,133)
    rgb255=(34,167,133)
    rgb255=(35,169,131)
    rgb255=(37,171,130)
    rgb255=(38,173,129)
    rgb255=(39,173,129)
    rgb255=(41,175,127)
    rgb255=(44,177,126)
    rgb255=(45,178,125)
    rgb255=(47,180,124)
    rgb255=(50,182,122)
    rgb255=(52,182,121)
    rgb255=(55,184,120)
    rgb255=(56,185,119)
    rgb255=(59,187,117)
    rgb255=(61,188,116)
    rgb255=(64,189,114)
    rgb255=(66,190,113)
    rgb255=(68,191,112)
    rgb255=(72,193,110)
    rgb255=(74,193,109)
    rgb255=(76,194,108)
    rgb255=(80,196,106)
    rgb255=(82,197,105)
    rgb255=(84,197,104)
    rgb255=(86,198,103)
    rgb255=(90,200,100)
    rgb255=(92,200,99)
    rgb255=(94,201,98)
    rgb255=(96,202,96)
    rgb255=(99,203,95)
    rgb255=(103,204,92)
    rgb255=(105,205,91)
    rgb255=(108,205,90)
    rgb255=(110,206,88)
    rgb255=(112,207,87)
    rgb255=(115,208,86)
    rgb255=(117,208,84)
    rgb255=(119,209,83)
    rgb255=(122,209,81)
    rgb255=(124,210,80)
    rgb255=(127,211,78)
    rgb255=(129,211,77)
    rgb255=(132,212,75)
    rgb255=(134,213,73)
    rgb255=(137,213,72)
    rgb255=(139,214,70)
    rgb255=(142,214,69)
    rgb255=(144,215,67)
    rgb255=(147,215,65)
    rgb255=(149,216,64)
    rgb255=(152,216,62)
    rgb255=(155,217,60)
    rgb255=(157,217,59)
    rgb255=(160,218,57)
    rgb255=(162,218,55)
    rgb255=(165,219,54)
    rgb255=(168,219,52)
    rgb255=(170,220,50)
    rgb255=(170,220,50)
    rgb255=(173,220,48)
    rgb255=(176,221,47)
    rgb255=(178,221,45)
    rgb255=(181,222,43)
    rgb255=(184,222,41)
    rgb255=(186,222,40)
    rgb255=(186,222,40)
    rgb255=(189,223,38)
    rgb255=(192,223,37)
    rgb255=(194,223,35)
    rgb255=(197,224,33)
    rgb255=(200,224,32)
    rgb255=(200,224,32)
    rgb255=(202,225,31)
    rgb255=(205,225,29)
    rgb255=(208,225,28)
    rgb255=(208,225,28)
    rgb255=(210,226,27)
    rgb255=(213,226,26)
    rgb255=(216,226,25)
    rgb255=(218,227,25)
    rgb255=(218,227,25)
    rgb255=(221,227,24)
    rgb255=(223,227,24)
    rgb255=(226,228,24)
    rgb255=(226,228,24)
    rgb255=(229,228,25)
    rgb255=(231,228,25)
    rgb255=(231,228,25)
    rgb255=(234,229,26)
    rgb255=(236,229,27)
    rgb255=(239,229,28)
    rgb255=(239,229,28)
    rgb255=(241,229,29)
    rgb255=(244,230,30)
    rgb255=(244,230,30)
    rgb255=(246,230,32)
    rgb255=(248,230,33)
    rgb255=(248,230,33)
    rgb255=(251,231,35)
    rgb255=(253,231,37)
    rgb255=(253,231,37)
  }
}

\tikzset{point/.style={circle,fill=black,inner sep=0mm,minimum size=1.75mm}}
\tikzset{plus/.pic={\draw[green!67!black,line width=1mm]
                    (-1mm, 0mm) -- (1mm, 0mm) (0mm, -1mm) -- (0mm, 1mm);}}
\usetikzlibrary{arrows.meta}
\tikzset{>={Stealth[length=3mm,width=2mm]}}

\expandafter\def\csname ver@etex.sty\endcsname{3000/12/31}

\usepackage{autonum}

\usepackage[tight,k-tight]{minitoc}
\mtcsetdepth{parttoc}{3}
\PassOptionsToPackage{toc,page,header}{appendix}

\mtcsetfeature{parttoc}{before}{}
\mtcsetfeature{parttoc}{pagestyle}{}
\mtcsetfeature{parttoc}{after}{}

\ifarxiv
\usepackage{parskip}
\usepackage{etoolbox}
\makeatletter
\patchcmd{\deferred@thm@head}{\addvspace{-\parskip}}{}{}{\PackageWarning{mydocument}{Patching deferred theorem head failed}}
\makeatother
\else
\fi

\newtheorem{ass}{Assumption}
\crefname{ass}{Assumption}{Assumptions}
\newtheorem{cor}{Corollary}
\crefname{cor}{Corollary}{Corollaries}
\newtheorem{lem}{Lemma}
\crefname{lem}{Lemma}{Lemmas}
\newtheorem{prop}{Proposition}
\crefname{prop}{Proposition}{Propositions}
\newtheorem{thm}{Theorem}
\crefname{thm}{Theorem}{Theorems}

\ifopt
\theorembodyfont{\upshape}
\theoremheaderfont{\itshape}
\else
\theoremstyle{remark}
\fi
\newtheorem{exl}{Example}
\crefname{exl}{Example}{Examples}
\newtheorem{rem}{Remark}
\crefname{rem}{Remark}{Remarks}

\crefformat{prob}{prob.~(#2#1#3)}
\Crefformat{prob}{Problem~(#2#1#3)}
\crefmultiformat{prob}{probs.~(#2#1#3)}{ and~(#2#1#3)}{, (#2#1#3)}{ and~(#2#1#3)}
\Crefmultiformat{prob}{Problems~(#2#1#3)}{ and~(#2#1#3)}{, (#2#1#3)}{ and~(#2#1#3)}

\DeclareMathOperator{\conv}{conv}

\allowdisplaybreaks[4]

\ifopt

\title[Minimal-norm univariate two-layer ReLU classification]{Minimal-norm univariate two-layer ReLU classification: \\ Exact solutions and global optimality with skip connections}

\else
\ifarxiv
\title{Minimal-Norm Univariate Two-Layer ReLU Classification: \\ Exact Solutions and Global Optimality with Skip Connections}
\else
\title{Minimal-Norm Univariate Two-Layer ReLU \\ Classification: Exact Solutions and \\ Global Optimality with Skip Connections}
\fi

\author{%
\begin{minipage}[t]{.9\textwidth}
\ifarxiv\centering\else\fi
Karolina Drabik\textsuperscript{1}\thanks{Equal contribution.} \hspace{1em}
Ben Lewis\textsuperscript{2}\footnotemark[\value{footnote}] \hspace{1em}
Antoni Puch\textsuperscript{1}\footnotemark[\value{footnote}] \hspace{1em}
Etienne Boursier\textsuperscript{3} \\[.2ex]
Piotr Hofman\textsuperscript{1} \hspace{1em}
Matthias Englert\textsuperscript{2} \hspace{1em}
Ranko Lazi\'c\textsuperscript{2}
\end{minipage} \\[\ifarxiv 4.2\else 2.7\fi ex]
\textsuperscript{1}
Department of Mathematics, Informatics and Mechanics,
University of Warsaw, Poland \\
\textsuperscript{2}
Department of Computer Science,
University of Warwick, UK \\
\textsuperscript{3}
INRIA \&
LMO, Universit\'e Paris-Saclay, France \\[\ifarxiv 1\else 0\fi ex]
\small
\texttt{\{kw.drabik2,a.puch,piotr.hofman\}@uw.edu.pl} \\
\small
\texttt{\{ben.lewis,m.englert,r.s.lazic\}@warwick.ac.uk} \\
\small
\texttt{etienne.boursier@inria.fr}}

\fi

\ifarxiv
\date{}
\else
\fi

\begin{document}

\doparttoc 
\faketableofcontents 

\maketitle

\begin{abstract}
We study minimal-norm interpolation and $\ell_2$-regularized logistic-loss minimization for binary classification by univariate two-layer ReLU networks. We give complete geometric characterizations of the optimal classifiers in function space, resolving how the solutions depend on whether hidden-layer biases are included in the parameter norm. When biases are unpenalized, the minimal-norm interpolators are exactly the continuous piecewise-affine functions that hug every label switch and have kinks of the appropriate convexity. When biases are penalized, the minimizer is unique in function space, has exactly one kink in each intermediate same-label segment, and is therefore a sparsest positive-margin classifier. We further show that adding a free affine skip connection leaves these function-space solutions unchanged but fundamentally improves the parameter-space landscape: every KKT point of the constrained problem becomes globally optimal, whereas suboptimal KKT points can occur without the skip connection. We establish analogous global-optimality and geometric results for sufficiently weak $\ell_2$-regularization of the logistic loss. In the unpenalized-bias case, we identify an additional sparsity-like restriction, implying that most minimal-norm interpolators cannot arise as small-regularization limits of margin-normalized logistic-loss minimizers. Numerical experiments across varying dataset complexity and network width support the predicted landscape and sparsity phenomena.
\end{abstract}

\section{Introduction}
\label{s:intro}

In practice, a key factor in the success of modern large language models is regularization of the norm of their parameters during training~\citep[e.g.,][]{WangA25,BergsmaDGGSH25}, by means of (decoupled) weight decay such as in the AdamW optimizer~\citep{LoshchilovH19}. 
Meanwhile, in theory, minimization of the $\mathcal{F}_1$~norm~\citep{KurkovaS01}, which is the counterpart in function space of the $\ell_2$~norm of two-layer networks in parameter space, has been proposed as of major significance in understanding generalization of modern neural networks~\citep{neyshabur2014search,bach2017breaking}.  Moreover, many prominent works studying implicit bias~\citep[cf.][]{Vardi23c} of gradient based algorithms on overparameterized neural networks have established that the algorithms often find solutions that generalize well precisely because they implicitly minimize the $\ell_2$~norm (or a closely related measure) of the model parameters, even when regularization is not explicitly enforced~\citep[e.g.,][]{LyuL20,JiT20,ChizatB20,BoursierPF22,CaiZWMLB25}.  For these and further reasons including its foundational role since the 1960s~\citep{tikhonov1963regularization}, research on optimization with $\ell_2$~norm minimization continues to attract considerable attention in the machine learning community.

A substantial part of that research has focused on the setting of univariate data (the basic case of low dimensionality) and two-layer ReLU networks (the canonical universal approximators~\citep[e.g.,][]{LeshnoLPS93}, with many challenging features of deep neural networks such as nonconvexity).  Two questions have been central:
\begin{enumerate}[label=\arabic*.,ref=\arabic*]
\item
\label{q:repr}
Given a function, what is the minimal%
\footnote{In this work, when we write just ``minimal'', ``minimum'', ``minimizer'', etc., we mean global; and we write ``local'' explicitly where that meaning is needed.}
parameter norm of a network that represents it?
\item
\label{q:inter}
Given a finite dataset, what functions are represented by its interpolators whose parameter norm is minimal?
\end{enumerate}
Considerable progress has been made in answering both, which has also identified two important distinctions, one optimizational and the other architectural:
\begin{itemize}
\item
whether biases are penalized (i.e., included in the parameter norm);
\item
whether the network has a skip connection (i.e., an additional affine term whose parameters are not penalized).
\end{itemize}
Studying all four combinations is well motivated, e.g.: the seminal implicit margin maximization results of \citet{LyuL20} and \citet{JiT20} involve penalizing the (first layer) biases, but do not apply in the presence of the skip connection (which breaks the homogeneity of the network); whereas, excluding the biases from weight decay has been standard in practice~\citep[cf.][Chapter~7]{Goodfellow-et-al-2016}; also, modern neural architectures such as the transformer~\citep{VaswaniSPUJGKP17} typically feature skip connections.

In this context, question~\ref{q:repr} has been resolved~\citep{savarese2019infinite,OngieWSS20,BoursierF23}, and question~\ref{q:inter} for regression to a large extent~\citep{parhi2021banach,Hanin21,DebarreDUF22,BoursierF23,kim2025exploring}.  Furthermore, building on these works, \citet{ZenoOBWS23} studied generalization of minimal-norm denoiser networks and \citet{JoshiVS24} characterized the nature of overfitting noisy data by minimal-norm interpolators.

For classification (where interpolation means $y \, f(x) \geq 1$ for each input~$x$ whose label is $y \in \{\pm 1\}$), convex-duality approaches have shown that in one dimension a globally optimal hinge-loss classifier can be chosen from a finite dictionary of data-aligned ReLU features~\citep{ErgenP21}, while \citet{SafranVL22}, \citet{BoursierF23}, and \citet{KornowskiYS23} obtained complementary structural and dynamical results. However, a complete characterization of all finite-sample minimum-norm classifier functions, including their dependence on bias penalization and skip connections, has remained open. In this work, we fill that gap, and do much more by answering also the following questions for classification:
\begin{enumerate}[resume]
\item
\label{q:skip}
How does adding a skip connection change the optimization landscape?
\item
\label{q:reg.loss.lim}
Given a finite dataset, what functions are represented by minimizers of the $\ell_2$-regularized logistic loss; and how are those functions, with margin normalization and small regularization strengths, related to those of the minimal-norm interpolators?
\end{enumerate}
More broadly, question~\ref{q:skip} is pertinent since, in spite of a great deal of research~\citep[e.g.,][]{HardtM17,OrhanP18,BarzilaiGGB23,MacDonaldVSL23,BelferGGB24}, theoretical understanding of skip connections is still a long way behind empirical demonstrations such as in the seminal works of \citet{HeZRS16cvpr} and \citet{Li0TSG18}.  Similarly, question~\ref{q:reg.loss.lim} addresses the theory-practice gap between minimal-norm interpolation and regularized-loss minimization, including how small regularization strengths~$\lambda$ need to be.  Here is a summary of our main contributions.
\begin{itemize}[wide,label=\textendash]
\item \textbf{Exact function space characterizations.}
We establish complete geometric characterizations of the minimal-norm interpolators and of the minimizers of the regularized loss in function space, show that they are not affected by the absence or presence of the skip connection, and show how they differ depending on whether the biases are penalized.
\item \textbf{Skip connections guarantee global optimality.}
We prove that, with a skip connection, and regardless of whether the biases are penalized, every KKT point of the interpolator norm minimization problem and every positive-margin stationary point of the regularized loss are (global) minimizers, and show that this is not the case without the skip connection.
\item \textbf{Regularized loss promotes sparsity.}
We prove that, when the biases are not penalized, the minimizers of the regularized loss (for any~$\lambda$ below a mild bound) in function space have a sparsity-like geometric property which minimal-norm interpolators do not necessarily have, and that therefore most minimal-norm interpolators cannot be obtained as limits of margin-normalized minimizers of the regularized loss by letting~$\lambda$ tend to~$0$.
\item \textbf{Validation by numerical experiments.}
For our theoretical results on how the optimization landscapes change with the skip connection, and on sparsity, we performed experiments which validate them on a range of synthetic datasets and with varying degrees of overparameterization.
\end{itemize}

Implications of our results include the following improvements and extensions of algorithms and results in prior works.
\begin{itemize}[wide,label=\textendash]
\item
The $32 r + 67$ bound of \citet[Theorem~4.2]{SafranVL22} on the number of linear regions of KKT points in function space when the biases are penalized reduces to~$r$ (where $r$~is the number of label switches in the dataset), provided either a skip connection is added or minimizers are considered instead of KKT points.
\item
We prove the conjecture of \citet[cf.][Corollary~1]{BoursierF23} that minimal-norm interpolators when the biases are penalized and with the skip connection are unique in function space, and show both uniqueness and optimal sparsity also without the skip connection.
\item
The privacy attack of \citet{Smorodinsky26} that extracts from any local minimizer, when the biases are penalized and without knowing the training dataset, a set of points at least $1 / 4$ of which are guaranteed to be in the dataset, can be improved to guarantee that all extracted points are in the dataset and moreover that every same-label segment of the dataset is represented, provided either a skip connection is added (in which case it suffices to consider KKT points) or global minimizers are considered.%
\footnote{Also, neither knowledge of the network's parameters nor existence of an omniactive neuron are needed.}
\end{itemize}

We provide proofs of all our results in the appendix, which also contains further results, background material, examples, figures, plots, remarks, and details of our experiments.

\section{Setting}

\paragraph{Networks.}

We consider two-layer ReLU networks with a skip connection and univariate inputs $x \in \mathbb{R}$, whose output is
\begin{equation}
\textstyle
f_{\theta, a_0, b_0}(x) \coloneqq
a_0 x + b_0 + \sum_{j =1}^m a_j \, \sigma(w_j x + b_j),
\label{eq:f}
\end{equation}  
where $\sigma(z) \coloneqq \max \{0, z\}$ is the ReLU nonlinearity, $m$~is the network width, $a_0$~is the weight and $b_0$~is the bias for the skip connection, and the remaining parameters $\theta = (a_j, w_j, b_j)_{j = 1}^m \in \mathbb{R}^{3 m}$ consist of output weights~$a_j$, hidden weights~$w_j$, and hidden biases~$b_j$.
When the skip connection is absent, instead of~$f_{\theta, 0, 0}$ we may write just~$f_\theta$.

\paragraph{Functions.}

We say that a function $g \colon \mathbb{R} \to \mathbb{R}$ is CPA if and only if it is continuous piecewise affine with a finite number of pieces.
We denote the slopes of its pieces by $(s_k)_{k = 0}^q$.  The pieces are separated by its kinks%
\footnote{The ``kinks'' are sometimes in the literature called ``knots'', ``breakpoints'', or ``inflection points''.}
$p_1 < \dots < p_q$.
For a kink~$p_k$, we say that it is convex if its incident slopes satisfy $s_{k - 1} < s_k$, and concave if $s_{k - 1} > s_k$.

CPA functions are exactly the class representable by univariate two-layer ReLU networks, and this characterization is not affected by whether the skip connection is present.  For simplicity of presentation and proximity to practice, in this work we do not consider infinite-width networks.

\paragraph{Datasets.}

We consider univariate binary classification datasets, i.e., $(x_i, y_i)_{i = 1}^n \in (\mathbb{R} \times \{\pm 1\})^n$, and without loss of generality we assume that the inputs are sorted, i.e., $x_1 < \dots < x_n$.

Let $\iota(0) \coloneqq 0$, and for each~$k$ with $\iota(k) < n$ let
$\iota(k + 1) \coloneqq
 \min \bigl(\{i > \iota(k) \mid y_i \neq y_{i + 1}\} \cup \{n\}\bigr)$.
This defines $\iota(0) < \dots < \iota(r + 1) = n$ such that the dataset has $r + 1$ maximal same-label segments and the indices between which the label switches are exactly $\iota(k)$ and $\iota(k) + 1$ for all $k \in [r]$.


\begin{ass}
\label{ass:data}
The inputs in the first (resp., last) two segments of the dataset are negative (resp., positive), i.e., $x_{\iota(2)} < 0 < x_{\iota(r - 1) + 1}$.
\end{ass}

This assumption avoids some special cases that would clutter the statements of our main results.  It combines requiring that the number~$r$ of label switches in the dataset is $\geq 3$ with a coordinate-dependent sign condition used only in the penalized-bias outer-kink arguments; see \cref{s:ass:data} for its precise role.  Also, this is the only assumption in this work; in particular, all our results apply regardless of how the dataset is sampled and for any cardinality~$n$.

\section{Main results}
\label{s:main}

The complete geometric characterizations obtained in our first main theorem reveal the exact effect of penalizing the biases: among switch hugging and convexity correct CPAs (of which there are infinitely many in general), it picks the unique one that is single turning, which is also the sparsest possible.  In addition, the theorem shows that although the skip connection does not modify the sets of minimizers in function space, it changes the optimization landscape dramatically, making every KKT point a (global) minimizer.

\begin{table}[t]
\caption{We consider these four variants of the interpolator norm minimization problem for a given dataset $(x_i, y_i)_{i = 1}^n \in (\mathbb{R} \times \{\pm 1\})^n$.  The skip connection is always ``free'', i.e., its parameters $a_0$~and~$b_0$ are never penalized.  Note also that the network width~$m$ is optimized, being determined by the dimension of the parameters vector $\theta = (a_j, w_j, b_j)_{j = 1}^m \in \mathbb{R}^{3 m}$.}
\label{tab:inter}
\centering
\ifopt\else\vspace{2ex}\fi
\begin{tabular}{m{.12\textwidth}m{.36\textwidth}m{.43\textwidth}}
&
\multicolumn{1}{c}{\bfseries biases not penalized}
&
\multicolumn{1}{c}{\bfseries biases penalized}
\\[-\ifarxiv 1\else 2\fi ex]
{\bfseries\shortstack[r]{skip \\ connection \\ absent}}
&
\begin{equation}
\begin{gathered}
\inf_\theta
\frac{1}{2}
\sum_{j = 1}^m (a_j^2 + w_j^2)
\text{ s.t.} \\
\forall i \in [n], \,
y_i \, f_\theta(x_i) \geq 1
\end{gathered}
\tag{$\mathrm{P}$}\label[prob]{eq:P}
\end{equation}
&
\begin{equation}
\begin{gathered}
\inf_\theta
\frac{1}{2}
\sum_{j = 1}^m (a_j^2 + w_j^2 + b_j^2)
\text{ s.t.} \\
\forall i \in [n], \,
y_i \, f_\theta(x_i) \geq 1
\end{gathered}
\tag{$\mathrm{P^b}$}\label[prob]{eq:Pb}
\end{equation}
\\[-5ex]
{\bfseries\shortstack[r]{skip \\ connection \\ present}}
&
\begin{equation}
\begin{gathered}
\inf_{\theta, a_0, b_0}
\frac{1}{2}
\sum_{j = 1}^m (a_j^2 + w_j^2)
\text{ s.t.} \\
\forall i \in [n], \,
y_i \, f_{\theta, a_0, b_0}(x_i) \geq 1
\end{gathered}
\tag{$\mathrm{P_1}$}\label[prob]{eq:P1}
\end{equation}
&
\begin{equation}
\begin{gathered}
\inf_{\theta, a_0, b_0}
\frac{1}{2}
\sum_{j = 1}^m (a_j^2 + w_j^2 + b_j^2)
\text{ s.t.} \\
\forall i \in [n], \,
y_i \, f_{\theta, a_0, b_0}(x_i) \geq 1
\end{gathered}
\tag{$\mathrm{P^b_1}$}\label[prob]{eq:Pb1}
\end{equation}
\end{tabular}
\end{table}

\begin{restatable}{thm}{thinter}
\label{th:inter}
Under \cref{ass:data}, for the problems in \Cref{tab:inter}, all of the following hold.
\begin{enumerate}[(i)]

\item
\label{th:inter:P.P1}
\Cref{eq:P,eq:P1} have the same minimizers in function space, which are exactly all CPAs $g \colon \mathbb{R} \to \mathbb{R}$ that are:
\begin{description}
\item[(switch hugging)]
at every label switch in the dataset, $g$~passes through the two incident points, i.e., for all $l \in [r]$, we have $g(x_{\iota(l)}) = y_{\iota(l)}$ and $g(x_{\iota(l) + 1}) = y_{\iota(l) + 1}$;
\item[(convexity correct)]
every kink~$p_k$ of~$g$ is contained in an intermediate same-label segment of the dataset, and its convexity is opposite to the segment's label, i.e., there exists $l \in [r - 1]$ with $p_k \in [x_{\iota(l) + 1}, x_{\iota(l + 1)}]$ and $sgn(s_k - s_{k - 1}) = -y_{\iota(l) + 1}$.
\end{description}

\item
\label{th:inter:Pb.Pb1}
\Cref{eq:Pb,eq:Pb1} have the same unique minimizer~$g$ in function space, which is of the form as in part~\ref{th:inter:P.P1} and moreover:
\begin{description}
\item[(single turning)]
every intermediate same-label segment contains exactly one kink of~$g$.
\end{description}
Thus $g$~has $r - 1$ kinks in total, which makes it a sparsest interpolant of the dataset.

\item
\label{th:inter:KKT}
For \cref{eq:P1,eq:Pb1}, every KKT point $(\theta, a_0, b_0) \in \mathbb{R}^{3 m + 2}$ (with any network width~$m$ for which KKT points exist) is a minimizer (globally, including over all~$m$).
\end{enumerate}
\end{restatable}

\newcommand{\fminfun}[1]{
\begin{figure}[#1]
\centering
\begin{tikzpicture}[xscale=\ifopt\else\ifarxiv 1.8\else 1.5\fi\fi,yscale=.75]
\draw         (0.5,  0) node[left]  {$0$} -- (9.3,  0);
\draw[dashed] (0.5,  1) node[left]  {$1$} -- (9.3,  1);
\draw[dashed] (0.5, -1) node[left] {$-1$} -- (9.3, -1);
\draw[yellow!90!black,very thick,dotted] (4.7, 1) -- (4.7, 2.75);
\draw[yellow!90!black,very thick,dotted] (5.2, 1) -- (5.2,    3);
\draw[every node/.append style={text=black},cyan,very thick,yshift=-1mm]
(0.5,     2) --
(2,      -1) node[below,xshift=-1mm]   {$+$} --
(3.2,    -1) node[below,xshift=1mm]    {$+$} --
(4,       1) node[above,xshift=-1.5mm] {$-$} --
(4.9,  2.25) node[above=-1mm]          {$-$} --
(5.5,     2) node[above=-1mm]          {$-$} --
(6.7,     1) node[above,xshift=1mm]    {$-$} --
(9.3, -2.47);
\draw[every node/.append style={text=black},lime!90!black,very thick]
(0.5,     2) --
(2,      -1) --
(2.8,  -1.4) node[below=-.5mm,xshift=-1mm]  {$+$} --
(3.2,    -1) --
(4.7,  2.75) node[above=-1mm,xshift=-1.5mm] {$-$} --
(5.1,  3.13) node[above=-1mm,xshift=1.5mm]  {$-$} --
(9.3, -2.47);
\draw[every node/.append style={text=black},pink,very thick,yshift=1mm]
(0.5,      2) --
(2.67, -2.33) node[below=-.5mm] {$+$} --
(4.94,  3.35) node[above=-1mm]  {$-$} --
(9.3,  -2.47);
\draw[dotted] (1,    1) node[point,color=brown!80!black] {} -- (1,   0) node[below] {$x_{1 = \iota(1)}$};
\draw[dotted] (2,   -1) node[point,color=brown!80!black] {} -- (2,   0) node[above] {$x_2$};
\draw[dotted] (3.2, -1) node[point,color=brown!80!black] {} -- (3.2, 0) node[above] {$x_{3 = \iota(2)}$};
\draw[dotted] (4,    1) node[point,color=brown!80!black] {} -- (4,   0) node[below] {$x_4$};
\draw[dotted] (4.7,  1) node[point]                      {} -- (4.7, 0) node[below] {$x_5$};
\draw[dotted] (5.2,  1) node[point]                      {} -- (5.2, 0) node[below] {$x_6$};
\draw[dotted] (5.8,  1) node[point]                      {} -- (5.8, 0) node[below] {$x_7$};
\draw[dotted] (6.7,  1) node[point,color=brown!80!black] {} -- (6.7, 0) node[below] {$x_{8 = \iota(3)}$};
\draw[dotted] (8.2, -1) node[point,color=brown!80!black] {} -- (8.2, 0) node[above] {$x_9$};
\draw[dotted] (8.8, -1) node[point]                      {} -- (8.8, 0) node[above] {$x_{10 = \iota(4)}$};
\end{tikzpicture}
\vspace{-1ex}
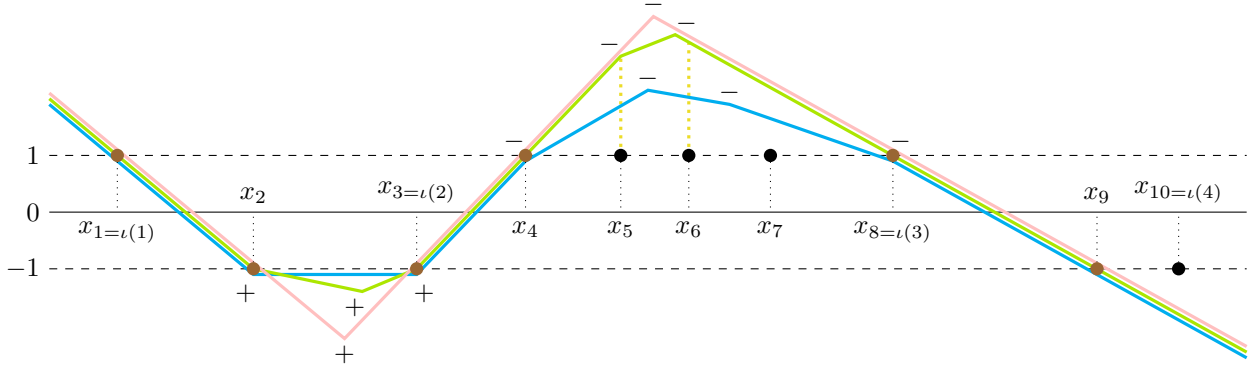
\caption{Here we depict a dataset with $n = 10$ points which are indicated by the bold dots, and three interpolating CPAs $g_\text{cyan}$, $g_\text{lime}$, and~$g_\text{pink}$ which are plotted with those colors respectively.  The dataset has $r = 3$ label switches, and its same-label segments end at indices $1 = \iota(1)$, $3 = \iota(2)$, $8 = \iota(3)$, and $10 = \iota(4)$.  Each convex kink is indicated by~$+$, and each concave kink by~$-$.  All of $g_\text{cyan}$, $g_\text{lime}$, and~$g_\text{pink}$ are switch hugging and convexity correct, however only~$g_\text{pink}$ is single turning (cf.\ \cref{th:inter}).  Also, $g_\text{lime}$~and~$g_\text{pink}$ are interval turning, but $g_\text{cyan}$~is not (cf.\ \cref{c:min.reg.loss.lim}).  For a switch hugging and convexity correct interpolating CPA to be interval turning, all its kinks in the same-label segment $[x_4, x_8]$ need to be in the interval $[x_5, x_6]$ between two consecutive inputs that contains the unique single turning kink (which is indicated by the yellow dotted lines).}
\label{f:min.fun}
\end{figure}
}

\ifopt
\else
\fminfun{t}
\fi

\ifopt
\else
\begin{proof}[Proof sketch]
For parts~\ref{th:inter:P.P1} and~\ref{th:inter:Pb.Pb1}, we first prove necessity of the geometric conditions, for every KKT point with the skip connection. Any violation yields an infinitesimal perturbation that decreases the parameter norm without decreasing any tight training margin, contradicting first-order optimality. The skip connection is crucial: since $\sigma(z) - \sigma(-z) = z$, reversing a neuron changes its contribution by an affine function, which the skip connection absorbs while preserving the relevant directional derivatives, allowing us to orient neurons conveniently when constructing descent directions.

When biases are not penalized (still with the skip connection, so \cref{eq:P1}), these directions exclude kinks of inappropriate convexity, outer kinks, and failures of switch hugging. Conversely, for every switch-hugging and convexity-correct CPA, the prescribed signs of the slope changes make the total variation telescope, so \cref{th:repr} gives the same function-space norm independently of the admissible kink locations and number. Hence these are exactly the minimizers with a skip connection.
When biases are penalized (\cref{eq:Pb1}), first-order rescaling forces balancedness, and the cost of a kink at~$p$ with slope change~$\Delta s$ is $\sqrt{1 + p^2} \, \lvert \Delta s \rvert$. The strict triangle inequality then rules out two distinct kinks of the required convexity in one same-label segment. Each intermediate segment therefore has one kink, uniquely determined by the adjacent switch points, giving the unique and sparsest minimizer.
For any positive-margin CPA with at least three label switches, a convex combination of two no-skip representations attains the same cost, so removing the skip connection (\cref{eq:P,eq:Pb}) does not change the minimizing functions.

For part~\ref{th:inter:KKT}, further descent directions ensure that every KKT parameterization with a skip connection attains the minimal representor norm of its function. Since that function is globally optimal, the parameterization is globally optimal over all widths. The constraint qualification holds throughout the feasible set, so every local minimizer is also covered.
\end{proof}
\fi

\newcommand{\flandnorm}[3]{
\begin{figure}[#1]
\centering
{\tikzexternalenable
\tikzsetnextfilename{cache_norm_a0_pq}
\tikzpicturedependsonfile{plot_dat_tex/obj_norm_a0_pq.dat}
\begin{tikzpicture}
\begin{axis}
[width=\ifopt\else\ifarxiv .475\else .5\fi\fi\textwidth,
 height=\ifopt\else\ifarxiv .475\else .5\fi\fi\textwidth,
 xmin=-3, xmax=3,
 ymin=-3, ymax=3,
 view={0}{90},
 colormap name=warpedviridis,
 point meta min=5,
 point meta max=95,
 unbounded coords=jump,
 title={skip connection absent}]
\addplot3
[surf,
 shader=interp,
 mesh/cols=101]
table
[x=p,
 y=q,
 z=value_capped]
{plot_dat_tex/obj_norm_a0_pq.dat};
\addplot3[red, line width=1.4pt, no marks] coordinates {(-3,-1,11) (-2,-1,11)};
\addplot3[red, line width=1.4pt, no marks] coordinates {( 1, 3,11) ( 1, 2,11)};
\addplot3[magenta, line width=1.4pt, no marks] coordinates {( 1,-3, 7) ( 1,-2, 7)};
\addplot3[magenta, line width=1.4pt, no marks] coordinates {( 3,-1, 7) ( 2,-1, 7)};
\addplot3[pink, only marks, mark=*, mark size=1.4pt] coordinates {( 1,-1, 6)};
\input{plot_dat_tex/boundary_a0_pq}
\input{plot_dat_tex/cont_norm_a0_pq}
\end{axis}
\end{tikzpicture}
\tikzsetnextfilename{cache_norm_opt_a_pq}
\tikzpicturedependsonfile{plot_dat_tex/obj_norm_opt_a_pq.dat}
\begin{tikzpicture}
\begin{axis}
[width=\ifopt\else\ifarxiv .475\else .5\fi\fi\textwidth,
 height=\ifopt\else\ifarxiv .475\else .5\fi\fi\textwidth,
 xmin=-3, xmax=3,
 ymin=-3, ymax=3,
 view={0}{90},
 colormap name=warpedviridis,
 colorbar,
 point meta min=5,
 point meta max=95,
 unbounded coords=jump,
 title={skip connection present}]
\addplot3
[surf,
 shader=interp,
 mesh/cols=101]
table
[x=p,
 y=q,
 z=value_capped]
{plot_dat_tex/obj_norm_opt_a_pq.dat};
\addplot3[pink, line width=1.4pt, no marks] coordinates {(-3,-1, 6) ( 3,-1, 6)};
\addplot3[pink, line width=1.4pt, no marks] coordinates {( 1,-3, 6) ( 1, 3, 6)};
\input{plot_dat_tex/boundary_opt_a_pq}
\input{plot_dat_tex/cont_norm_opt_a_pq}
\end{axis}
\end{tikzpicture}}
\caption{#2}
\label{#3}
\end{figure}
}

\ifopt
\else

Although \cref{th:inter} \ref{th:inter:P.P1}~and~\ref{th:inter:Pb.Pb1} show that adding the free skip connection does not change the optimal functions, its landscape conclusion is specific to the architectures with a skip connection. For the dataset in \Cref{f:land.norm:main}, \cref{eq:P} has plateaus of suboptimal KKT points (the corresponding interpolants have outer kinks and are therefore not convexity correct), which disappear from the plotted profile after adding the skip connection. \cref{th:KKT.not} in the appendix gives the explicit construction and an analogous penalized-bias counterexample.

\begin{figure}[t]
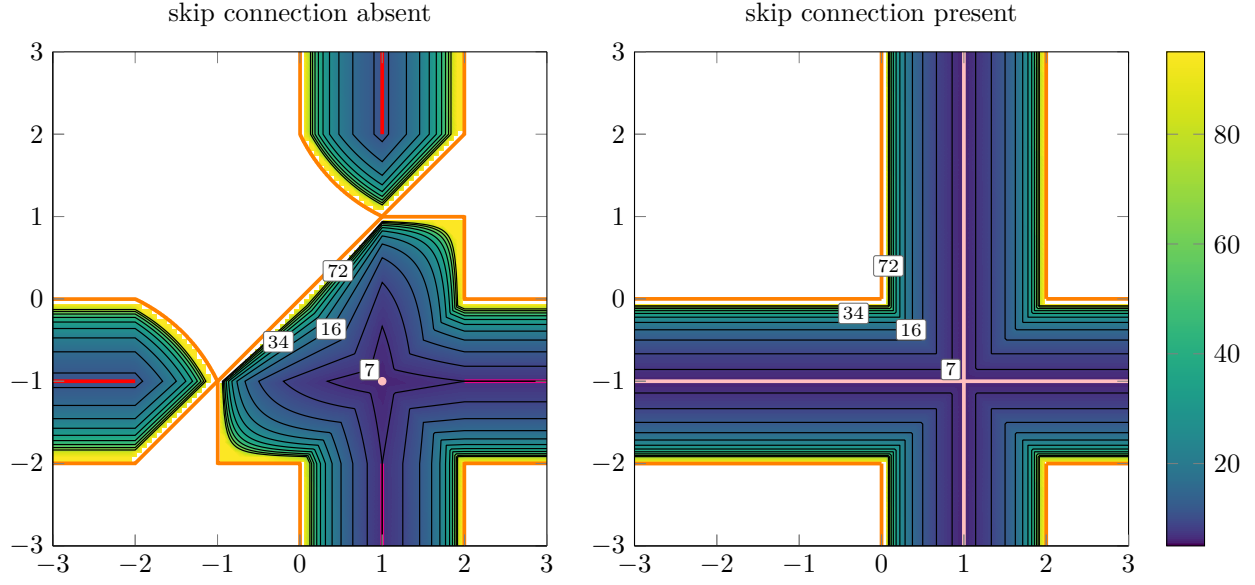

\centering
{\tikzexternalenable
\tikzsetnextfilename{cache_norm_a0_pq}
\tikzpicturedependsonfile{plot_dat_tex/obj_norm_a0_pq.dat}
\begin{tikzpicture}
\begin{axis}
[width=\ifopt .475\else\ifarxiv .475\else .5\fi\fi\textwidth,
 height=\ifopt .475\else\ifarxiv .475\else .5\fi\fi\textwidth,
 xmin=-3, xmax=3,
 ymin=-3, ymax=3,
 view={0}{90},
 colormap name=warpedviridis,
 point meta min=5,
 point meta max=95,
 unbounded coords=jump,
 title={skip connection absent}]
\addplot3
[surf,
 shader=interp,
 mesh/cols=101]
table
[x=p,
 y=q,
 z=value_capped]
{plot_dat_tex/obj_norm_a0_pq.dat};
\addplot3[red, line width=1.4pt, no marks] coordinates {(-3,-1,11) (-2,-1,11)};
\addplot3[red, line width=1.4pt, no marks] coordinates {( 1, 3,11) ( 1, 2,11)};
\addplot3[magenta, line width=1.4pt, no marks] coordinates {( 1,-3, 7) ( 1,-2, 7)};
\addplot3[magenta, line width=1.4pt, no marks] coordinates {( 3,-1, 7) ( 2,-1, 7)};
\addplot3[pink, only marks, mark=*, mark size=1.4pt] coordinates {( 1,-1, 6)};
\input{plot_dat_tex/boundary_a0_pq}
\input{plot_dat_tex/cont_norm_a0_pq}
\end{axis}
\end{tikzpicture}
\tikzsetnextfilename{cache_norm_opt_a_pq}
\tikzpicturedependsonfile{plot_dat_tex/obj_norm_opt_a_pq.dat}
\begin{tikzpicture}
\begin{axis}
[width=\ifopt .475\else\ifarxiv .475\else .5\fi\fi\textwidth,
 height=\ifopt .475\else\ifarxiv .475\else .5\fi\fi\textwidth,
 xmin=-3, xmax=3,
 ymin=-3, ymax=3,
 view={0}{90},
 colormap name=warpedviridis,
 colorbar,
 point meta min=5,
 point meta max=95,
 unbounded coords=jump,
 title={skip connection present}]
\addplot3
[surf,
 shader=interp,
 mesh/cols=101]
table
[x=p,
 y=q,
 z=value_capped]
{plot_dat_tex/obj_norm_opt_a_pq.dat};
\addplot3[pink, line width=1.4pt, no marks] coordinates {(-3,-1, 6) ( 3,-1, 6)};
\addplot3[pink, line width=1.4pt, no marks] coordinates {( 1,-3, 6) ( 1, 3, 6)};
\input{plot_dat_tex/boundary_opt_a_pq}
\input{plot_dat_tex/cont_norm_opt_a_pq}
\end{axis}
\end{tikzpicture}}
\caption{A free skip connection removes suboptimal KKT plateaus from a two-dimensional landscape profile without changing the global optimum. For the dataset $(-2, 1), (-1, -1), (1, 1), (2,-1)$, we plot the norm objective over interpolating symmetric width-four networks with unpenalized biases. The axes give two kink positions $p_\dag, p_\ddag$; the corresponding slope changes are optimized subject to interpolation, and orange marks the feasibility boundary. \textbf{Left:} Without a skip connection, the profile contains suboptimal plateaus corresponding to KKT points at objective values~$11$ (colored red) and~$7$ (colored magenta), whereas the global minimum is~$6$ (colored pink). \textbf{Right:} With a free skip connection, these plateaus disappear from the profile, while the global minimum remains~$6$. See \cref{th:KKT.not} for the KKT construction and \cref{rem:KKT.not} for additional local-minimality observations.}
\label{f:land.norm:main}
\end{figure}

\fi

Next we state a corollary of two theorems fully given in the appendix. It characterizes regularized-loss minimizers in function space as \cref{th:inter} does minimal-norm interpolators, but with far less flexibility when biases are unpenalized: all CPA kinks in each dataset segment must lie between two specific consecutive training inputs.
This discontinuity between regularized and constrained optimization, more pronounced as the dataset grows denser (more samples per same-label segment), identifies a sparsity-driving mechanism beyond bias penalization. The result also shows that a skip connection makes every positive-margin stationary point a (global) minimizer, and that most minimal-\hspace{0em}norm interpolators are not limits of margin-normalized regularized-loss minimizers as $\lambda \to 0$.

\begin{cor}[of \cref{th:min.reg.loss,th:lim}]
\label{c:min.reg.loss.lim}
Under \cref{ass:data}, and for every sufficiently small regularization strength~$\lambda$ (depending only on the dataset), all of the following hold.
\begin{enumerate}[(i)]

\item
\label{c:min.reg.loss.lim:L}
The characterizations in \cref{th:inter}~\ref{th:inter:P.P1} and~\ref{th:inter:Pb.Pb1} of the minimizers of \cref{eq:P,,eq:P1,,eq:Pb,,eq:Pb1} in function space hold also for the minimizers~$g$ of the corresponding $\ell_2$-regularized logistic losses (cf.\ \cpageref{eq:L}), except that the switch hugging property is replaced by
\begin{description}
\item[(switch optimal)]
for all $l \in [r]$, the outputs $g(x_{\iota(l)})$ and~$g(x_{\iota(l) + 1})$ at the inputs incident to the $l$-th label switch are uniquely determined by the dataset and~$\lambda$;
\end{description}
and when the biases are not penalized, the CPA~$g$ additionally satisfies
\begin{description}
\item[(interval turning)]
for all $l \in [r - 1]$, letting~$p$ and~$p'$ be the first and last (respectively) kinks of~$g$ in the $l$-th intermediate same-label segment $[x_{\iota(l) + 1}, x_{\iota(l + 1)}]$, the interval $(p, p')$ contains no inputs from the dataset.
\end{description}

\item
\label{c:min.reg.loss.lim:stat}
If the skip connection is present, then every Clarke stationary point of the $\ell_2$-regularized logistic loss (without or with penalizing the biases) whose margin is positive is a minimizer (globally, including over all~$m$).

\item
\label{c:min.reg.loss.lim:lim}
For all four options of whether the skip connection is present and whether the biases are penalized, the limits of margin-normalized Clarke stationary points of the $\ell_2$-regularized logistic loss as $\lambda$~tends to~$0$ are KKT points of the corresponding interpolator norm minimization problem.  If the skip connection is present and the biases are not penalized, then the latter KKT points are necessarily interval turning in function space.
\end{enumerate}
\end{cor}

\ifopt
\else

\begin{proof}[Proof sketch]
The proof of \cref{th:min.reg.loss} parallels that of \cref{th:inter}, with two changes: every training point contributes to the objective, and the outputs at label switches are selected by a finite-dimensional convex optimization problem rather than fixed at~$\pm 1$. For $\tau, \pi \in \mathbb{R}_{> 0}^r$, let $g_{\tau, \pi}$ be the canonical CPA whose $k$-th crossing piece joins $(x_{\iota(k)}, y_{\iota(k)} \tau_k)$ and $(x_{\iota(k) + 1}, y_{\iota(k) + 1} \pi_k)$. Linear inequalities characterize the compatibility of these pieces across intermediate same-label segments. On this convex feasible set, the empirical logistic loss of~$g_{\tau, \pi}$ plus the corresponding minimal representor norm regularizer is strictly convex because of the terms at the switch inputs. Under the stated upper bound on~$\lambda$ (see \cpageref{th:min.reg.loss}), each objective has a unique minimizer with strictly positive coordinates, defining the switch-optimal values.

The geometric descent arguments from \cref{th:inter} exclude failures of convexity correctness and, when biases are penalized, multiple kinks in a same-label segment, giving single turning. Without bias penalization, if two same-convexity kinks enclose a training input, the graph between them can be moved toward the correct label without increasing the norm and while decreasing the logistic loss. Stationarity therefore requires all kinks in each same-label segment to lie between the same pair of consecutive training inputs, which is exactly interval turning.

Suppose a positive-margin stationary point has the appropriate turning and convexity properties but is not switch optimal. We perturb one or two neurons in each intermediate segment, together with the skip connection, so that the first-order change in the switch margins points from $(\tau, \pi)$ toward the unique convex minimizer $(\tau^\star, \pi^\star)$. The directional derivative of the network regularizer equals that of the corresponding function-space norm, so strict convexity makes the derivative of the full objective negative, contradicting stationarity. This proves the characterizations in \cref{th:min.reg.loss}. As in \cref{th:inter}, balancing and alignment make the parameter norm equal the function's minimal representor norm; hence every positive-margin stationary point with a skip connection is (globally) minimizing.

For \cref{th:lim}, let $\lambda_k \to 0$ and margin-normalize a sequence of approximate stationary points. Stationarity yields approximate KKT conditions for the corresponding unit-margin interpolation problem, with multipliers $\mu_i^{(k)} = 1 / n \lambda_k \bigl(1 + \exp(y_i \, f_k(x_i))\bigr)$. The assumed relation between the stationarity error and~$\lambda_k$ makes the KKT equilibrium error vanish and the unnormalized margins diverge. The multipliers at normalized margins strictly larger than~$1$ therefore vanish, giving asymptotic complementary slackness, so every convergent normalized subsequence limits to a KKT point.

In the unpenalized-bias case with a skip connection, a non-interval-turning limit would yield, uniformly along the sequence, a direction improving a training margin without increasing the quadratic regularizer. The resulting logistic-loss improvement has an exponential scale determined by the margin, and the stronger error condition in \cref{th:lim}~\ref{th:lim:it} ensures that it dominates the stationarity error, a contradiction.
\end{proof}

\Cref{f:land.reg} in the appendix shows that the small-$\lambda$ regularized-loss landscape undergoes the same qualitative transformation as the constrained landscape in \Cref{f:land.norm:main} when the skip connection is added.

\fi

\newcommand{\ExperimentSeedCount}{60}
\newcommand{\LossSymlogThreshold}{0.001}
\newcommand{\SparsitySymlogThreshold}{0.10000000000000001}

\pgfplotsset{
  shared bounds/loss/all/.style={ymin=-0.00038908503539490386,ymax=1.0451640899287564},
  shared bounds/sp/small/.style={ymin=-0.028724660641816252,ymax=15.32665764110862},
  shared bounds/sp/medium/.style={ymin=-0.028724660641816252,ymax=15.32665764110862},
  shared bounds/sp/large/.style={ymin=-0.028724660641816252,ymax=15.32665764110862},
  shared bounds/unsuccessful/all/.style={ymin=0,ymax=0.34999999999999998},
}

\pgfplotsset{
  loss symlog y/.style={
    y coord trafo/.code={\pgfmathparse{sign(##1)*(
      min(abs(##1)/\LossSymlogThreshold,1)
      +ln(max(abs(##1)/\LossSymlogThreshold,1)))}},
    y coord inv trafo/.code={\pgfmathparse{sign(##1)*\LossSymlogThreshold*(
      min(abs(##1),1)+max(exp(abs(##1)-1)-1,0))}},
    ytick={-1,-0.1,-0.01,-0.001,0,0.001,0.01,0.1,1},
    yticklabels={{$-10^0$},{$-10^{-1}$},{$-10^{-2}$},{$-10^{-3}$},{$0$},
                 {$10^{-3}$},{$10^{-2}$},{$10^{-1}$},{$10^0$}},
    minor ytick={-0.9,-0.8,-0.7,-0.6,-0.5,-0.4,-0.3,-0.2,
      -0.09,-0.08,-0.07,-0.06,-0.05,-0.04,-0.03,-0.02,
      -0.009,-0.008,-0.007,-0.006,-0.005,-0.004,-0.003,-0.002,
       0.002,0.003,0.004,0.005,0.006,0.007,0.008,0.009,
       0.02,0.03,0.04,0.05,0.06,0.07,0.08,0.09,
       0.2,0.3,0.4,0.5,0.6,0.7,0.8,0.9},
  },
}

\definecolor{clrNoskip}{HTML}{B2182B}
\definecolor{clrSkip}{HTML}{2166AC}

\pgfplotsset{
  experiment axis/.style={
    unbounded coords=discard, filter discard warning=false,
    width=\ifopt\else\ifarxiv .25\else .26\fi\fi\textwidth,
    height=\ifopt\else .19\fi\textheight,
    xlabel={$r_0$}, xtick={4,6,8,10,12}, xmin=3.4, xmax=12.6,
    grid=major, major grid style={draw=black!12,line width=0.2pt},
    axis line style={black!65}, tick style={black!65},
    xticklabel style={font=\footnotesize}, yticklabel style={font=\footnotesize},
    xlabel style={font=\footnotesize}, ylabel style={font=\footnotesize},
    title style={font=\footnotesize,yshift=-4pt},
    every axis plot/.append style={line width=0.75pt}, clip mode=individual,
  },
  hide row y axis/.style={yticklabels={},ylabel={}},
}


\newcommand{\rangeSeries}[5]{%
  \addplot[draw=none,name path=#5-upper,forget plot]
    table[col sep=comma,x=num_class_breaks,y=max_#2] {#1};%
  \addplot[draw=none,name path=#5-lower,forget plot]
    table[col sep=comma,x=num_class_breaks,y=min_#2] {#1};%
  \addplot[draw=none,fill=#3,fill opacity=0.18,forget plot]
    fill between[of=#5-upper and #5-lower];%
  \addplot[#3,#4]
    table[col sep=comma,x=num_class_breaks,y=mean_#2] {#1};%
}

\newcommand{\skipLegend}{%
  \begingroup\small\hspace{5.4em}
  \tikz[baseline=-0.55ex]{\draw[clrNoskip,solid,line width=0.8pt]
    plot[mark=*,mark size=1.35pt] coordinates {(0,0) (0.65,0)};}
  \ skip connection absent \hspace{3em}
  \tikz[baseline=-0.55ex]{\draw[clrSkip,dashed,line width=0.8pt]
    plot[mark=square*,mark size=1.25pt] coordinates {(0,0) (0.65,0)};}
  \ skip connection present
  \endgroup
}

\DeclareRobustCommand{\twoEquationTitle}[1]{%
  \(\begin{aligned} n&=(r_0/2)^2\\m&=#1\cdot r_0\end{aligned}\)%
}

\newcommand{\lossPanel}[5]{%
  \begin{tikzpicture}
    \begin{axis}[
      experiment axis, loss symlog y, scaled y ticks=false,
      ylabel={\shortstack{regularized loss (relative excess)\\biases #4 regularized}},
      title={\twoEquationTitle{#5}},
      /pgfplots/shared bounds/loss/all/.try,
      /pgfplots/manual bounds/loss/all/.try,
      #3,
    ]
      \rangeSeries{plot_np_data/panels/loss_#1_F#2.csv}{relative_excess_regularized_objective_noskip}{clrNoskip}
        {solid,mark=*,mark size=1.35pt}{loss-#1-#2-noskip}
      \rangeSeries{plot_np_data/panels/loss_#1_F#2.csv}{relative_excess_regularized_objective_skip}{clrSkip}
        {dashed,mark=square*,mark size=1.25pt}{loss-#1-#2-skip}
      \addplot[black!55,densely dotted,mark=none,forget plot]
        coordinates {(4,0) (12,0)};
    \end{axis}
  \end{tikzpicture}%
}

\newcommand{\fregloss}[2]{
\begin{figure}[t]
  \centering\skipLegend\par\medskip
  \setlength{\tabcolsep}{0.2pt}\renewcommand{\arraystretch}{0.90}
  \begin{tabular}{@{}ccccc@{}}
    \lossPanel{adam_weights}{1}{}{not}{1}&
    \lossPanel{adam_weights}{2}{hide row y axis}{not}{2}&
    \lossPanel{adam_weights}{3}{hide row y axis}{not}{3}&
    \lossPanel{adam_weights}{4}{hide row y axis}{not}{4}&
    \lossPanel{adam_weights}{5}{hide row y axis}{not}{5}\\
    \lossPanel{adam_weights_biases}{1}{}{}{1}&
    \lossPanel{adam_weights_biases}{2}{hide row y axis}{}{2}&
    \lossPanel{adam_weights_biases}{3}{hide row y axis}{}{3}&
    \lossPanel{adam_weights_biases}{4}{hide row y axis}{}{4}&
    \lossPanel{adam_weights_biases}{5}{hide row y axis}{}{5}
  \end{tabular}
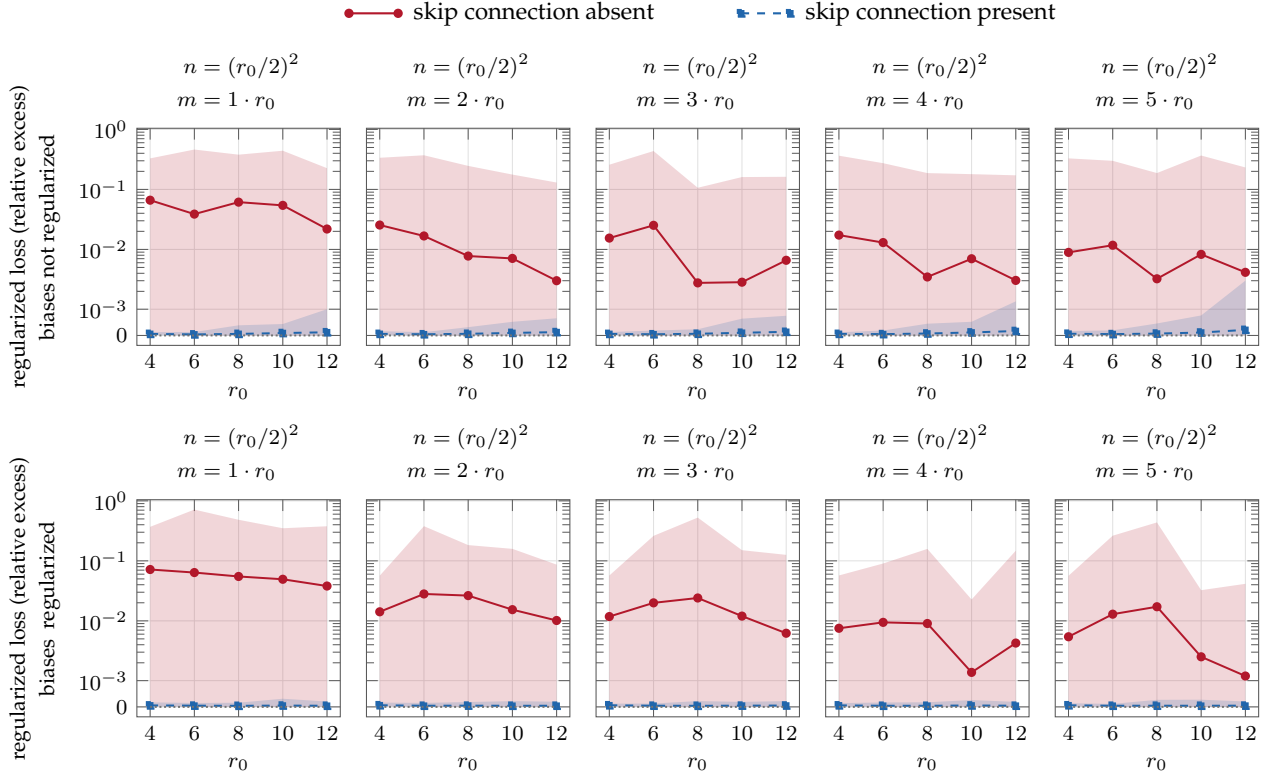
  \caption{{#1}We plot the relative excess of the $\ell_2$-regularized logistic loss~$L$ at the end of training, i.e., $L / L^\star - 1$, where $L^\star$~is the global minimum for the experiment setting, computed via the convex problem characterized in \cref{th:min.reg.loss}.  Rows indicate whether biases are regularized, columns give the overparameterization factor $\omega = m / r_0$, and the two curves indicate whether the skip connection is absent or present.  For each point, we use the runs, among 60 RNG seeds, that reached a positive margin: the curve gives their mean, and the shaded region spans their smallest and largest outcomes (positive-margin failure rates are reported in \Cref{f:inter.fail}).  The shared vertical axis is logarithmic above~$10^{-3}$ and linear below it.  \textbf{Observations:}  With a skip connection, training consistently reaches or closely approaches the global minimum; without one, the best runs reach the minimum but the means and worst outcomes typically remain substantially above it. This is consistent with our theoretical result that a skip connection rules out suboptimal positive-margin stationary points, and with the counterexamples showing that such points can occur without it.}
  \label{#2}
\end{figure}
}

\ifopt
\else

\section{Experiments}

The main outcomes of our numerical experiments are plotted in \Cref{f:reg.loss:main}.  In the interval $(-1, 1)$, we sample $r_0$~class breaks and use them to label a random dataset of $n \coloneqq (r_0 / 2)^2$ inputs, while satisfying \cref{ass:data}.  For an overparameterization hyperparameter~$\omega$, we then randomly initialize a network of width $m \coloneqq \omega \cdot r_0$, and train it by Adam~\citep{KingmaB14} to get close to a Clarke stationary point of the $\ell_2$-regularized logistic loss with small regularization strength~$\lambda$.
\Cref{s:exp:app} gives full details and additional sparsity and positive-margin results.  Since the global minimum is computed independently through the convex characterization in \cref{th:min.reg.loss}, the reported excess loss directly tests whether training reaches the theoretically characterized optimum.

\fregloss{}{f:reg.loss:main}

\fi

\ifopt
\else

\section{Discussion}
\label{s:disc}

A function-space view of norm-controlled shallow ReLU networks begins with characterizations of the minimal parameter norm needed to represent a prescribed function. For univariate inputs, \citet{savarese2019infinite} related this cost to the total variation of the derivative, while \citet{OngieWSS20} developed a multivariate analogue using Radon-transform methods. \citet{parhi2021banach} placed such results in a broader Banach-space representor-theorem framework, under which shallow networks arise as ridge-spline solutions of variational problems. These works characterize the regularizer induced in function space, but do not by themselves determine which function is selected by finite classification constraints. In one dimension, \citet{BoursierF23} showed that including hidden biases in the parameter norm spatially weights the variation penalty and promotes sparse, and in their regression setting generically unique, interpolants, whereas omitting the biases permits nonsparse solutions. Our work builds on these representor-norm formulas, but resolves the corresponding classification problems for both choices of bias regularization and both with and without a free affine skip connection.

For finite-sample interpolation, \citet{Hanin21} showed that minimal-norm univariate ReLU regression exhibits a nearest-neighbor curvature extrapolation rule and established associated generalization guarantees. \citet{DebarreDUF22} studied the related, but distinct, objective of obtaining sparsest one-dimensional piecewise-linear regressors, while \citet{kim2025exploring} connected regularized neural-network landscapes to convex formulations. In complementary multivariate directions, \citet{ArdeshirHS23} investigated the intrinsic dimensionality and generalization properties of the $\mathrm{R}$-norm inductive bias, and \citet{ParkPW23} established an asymptotic connection between regularized empirical-risk minimization and minimal-norm shallow-network interpolation as the sample size and width grow and the regularization vanishes. The latter contrasts with our fixed-dataset characterization: in our classification setting, margin-normalized small-regularization solutions may select only a strict subset of the minimal-norm interpolators.

\Citet{ErgenP21} gave a convex-geometric treatment of norm-regularized finite-width two-layer ReLU networks. For univariate binary classification with hinge loss, they showed that an optimum can be chosen from data-aligned extreme-point ReLU features, reducing training to a finite-dimensional minimal-$\ell_1$-norm SVM. Our results are complementary: we characterize the entire set of globally optimal classifier functions, rather than the existence of a sparse extreme-point representative, and show that in the unpenalized-bias case optimal functions may have kinks away from training inputs. We also treat penalized hidden biases, a free affine skip connection, the stationary-point landscape of the original nonconvex parameterization, regularized logistic loss, and its small-regularization limit.

Among works focused on gradient dynamics and the complexity of the learned classifier, the closest is \citet{SafranVL22}: they studied gradient flow for shallow univariate ReLU classifiers and bounded the number of linear regions of the limiting classifier under distributional and initialization assumptions. \citet{KornowskiYS23} related the piecewise-linear complexity of learned ReLU functions to tempered and benign overfitting. Other works use minimal-norm structure to study statistical or algorithmic consequences: \citet{ZenoOBWS23} derived closed-form shallow denoisers and associated univariate and multivariate geometric properties; \citet{JoshiVS24} showed that minimal-norm univariate ReLU regression can exhibit tempered or catastrophic overfitting depending on the risk criterion; \citet{Smorodinsky26} constructed provable training-data privacy attacks; and \citet{ElHarzliNKCGL26} gave sufficient conditions for the algorithmic stability of minimal-norm interpolating deep ReLU networks. Complementarily, \citet{NakhlehN25} showed that suitable $\ell_p$ quasi-norm objectives have globally sparsest interpolating shallow ReLU networks as their global minimizers. Our deterministic results isolate the positive-margin classification geometry underlying such questions: for arbitrary finite datasets satisfying only a mild assumption, we characterize all globally optimal functions, determine exactly when sparsity and uniqueness occur, and, with a skip connection, upgrade first-order stationarity to global optimality. Thus the univariate setting permits a simultaneous and exact treatment of function-space selection, parameter-space optimization, and the regularized-to-constrained limit that is not currently available in broader architectures.

Our landscape results also connect to theoretical explanations of residual architectures. \citet{HardtM17} showed that identity parameterizations can give deep linear residual networks favorable optimization properties; \citet{OrhanP18} argued that skip connections remove overlap, elimination, and linear-dependence singularities; and \citet{BarzilaiGGB23} and \citet{BelferGGB24} analyzed how residual connections modify convolutional and deep neural tangent kernels. \citet{MacDonaldVSL23} provided a layerwise framework in which skip connections propagate favorable curvature and regularity properties, and identified a mechanism for accelerating gradient optimization. These works explain conditioning, spectral behavior, or trainability in deep models. By contrast, we give an exact finite-network landscape statement: a free affine skip leaves the globally optimal functions unchanged but makes every KKT point of the constrained problem, and every positive-margin stationary point of the regularized problem, globally optimal. This separation of representational and optimizational effects provides a basic setting in which the benefit of a skip connection can be established globally rather than locally or asymptotically.

\fi

\section{Conclusion}

\ifopt
\else

The preceding discussion places our results within broader efforts to understand the function-space bias, statistical consequences, and optimization landscapes induced by norm control. Although we focus on univariate classification by two-layer ReLU networks, this tractable setting allows all three aspects to be resolved together: we completely characterize the optimal classifiers in function space; determine how penalizing biases changes their sparsity and uniqueness; and show that a free affine skip connection eliminates suboptimal KKT points, and hence local minima, without changing the globally optimal functions. These conclusions require only a mild assumption on the dataset.

Natural next steps are to characterize KKT and Clarke stationary points more fully when the skip connection is absent, and to extend the geometric and landscape results to multivariate inputs and deeper networks. The exact univariate theory developed here provides both a benchmark for such extensions and a collection of mechanisms---including label-switch geometry, bias-induced spatial weighting, and skip-enabled descent directions---whose higher-dimensional counterparts remain to be understood.

A complementary empirical direction is to study which of the possible no-skip stationary points are selected by gradient-based dynamics. A natural benchmark is unregularized logistic-loss gradient flow in the homogeneous, bias-penalized architecture, for which implicit-bias results relate long-time normalized trajectories to KKT points of the corresponding margin problem~\citep{LyuL20,JiT20}. Comparing these trajectories with our exact global minimizer could determine whether the dynamics select the global solution or one of the suboptimal KKT points that may occur without a skip connection. Such an experiment is challenging, however, because directional convergence is asymptotic and can be extremely slow.

\fi

\ifopt\else

\ifarxiv
\else
\clearpage
\fi
\phantomsection
\addcontentsline{toc}{section}{AI use statement}
\subsection*{AI use statement}

In this work, we used generative AI tools to: aid or polish writing in around 20\%\ of the paper; find several related works; create and edit the code for running the experiments and plotting their results, and for depicting the optimization landscapes.  We take responsibility for the final content of this work, including the code produced with the aid of generative AI, which we have reviewed and tested.

\phantomsection
\addcontentsline{toc}{section}{Reproducibility statement}
\subsection*{Reproducibility statement}

For the theoretical results in this work, we clarified all nonstandard notation, defined all new notions, stated all required assumptions, and provided all nontrivial proofs (either in the main part of the paper, or in the appendix).
For performing the experiments and plotting their results, and for producing the optimization landscape depictions in \Cref{f:land.norm:main,f:land.reg}, we provided complete code and instructions how to run it
\ificlrfinal
at {\footnotesize \url{https://github.com/englert-m/min-norm-univariate-classification}};
\else
in the supplementary materials;
\fi
moreover, doing so does not require specialized software or hardware, and uses a relatively small amount of compute.


\fi

\ifopt\else
\ificlrfinal
\phantomsection
\addcontentsline{toc}{section}{Acknowledgments}
\subsubsection*{Acknowledgments}

We thank Filip Mazowiecki for helpful discussions.

We acknowledge partial support by the European Research Council (ERC grants BUKA, number 101126229, and INFSYS, number 950398), the Engineering and Physical Sciences Research Council (EPSRC, project reference EP/W524645/1), the Polish National Science Centre (SONATA BIS-12 grant number 2022/46/E/ST6/00230), and the Centre for Discrete Mathematics and its Applications (DIMAP) at the University of Warwick.  Also we acknowledge the use of the Batch Compute System in the Department of Computer Science at the University of Warwick, and associated support services.
\fi
\fi

\phantomsection
\addcontentsline{toc}{section}{References}
\bibliography{main}
\ifopt\else\ifarxiv
\bibliographystyle{plainnat}
\else
\bibliographystyle{iclr2027_conference}
\fi\fi

\clearpage
\appendix

\addcontentsline{toc}{section}{Appendix} 
\part{Appendix} 
\parttoc 

\clearpage
\section{Overview diagram of proof dependencies}

\begin{figure}[h!]
\centering
\hyphenpenalty=10000\exhyphenpenalty=10000
\begin{tikzpicture}[x=1cm,y=1cm,
  grp/.style   = {draw=black,line width=.4pt,rounded corners=2pt,fill=white,
                  inner sep=3pt,align=center,font=\footnotesize,
                  text width=3.45cm,minimum height=1.85cm},
  ext/.style   = {grp,dashed},
  res/.style   = {grp,fill=black!4},
  panel/.style = {draw=black,dotted,line width=.4pt,rounded corners=2pt},
  ptit/.style  = {font=\footnotesize\itshape,align=center,inner sep=1pt},
  a/.style     = {draw=black,line width=.4pt,rounded corners=8pt,
                  -{Stealth[length=1.9mm,width=1.4mm]}},
  A/.style     = {a},
]
\draw[panel] (0,.75) rectangle (4.13,-10.094);
\draw[panel] (4.63,.75) rectangle (8.76,-10.094);
\draw[panel] (9.26,.75) rectangle (13.4,-10.094);
\node[ptit,text width=4cm] at (2.065,.13) {Only for the\\ interpolation problems};
\node[ptit,text width=4cm] at (6.695,.13) {Common machinery\\ (both main theorems)};
\node[ptit,text width=4cm] at (11.33,.13) {Only for the\\ regularized losses};
\node[ext] (mB) at (6.695,-1.735) {\textbf{\cref{th:repr}}\\[2pt] Prior work: \\ minimal representor norms};
\node[grp] (mD) at (6.695,-4.135) {\textbf{\cref{l:rev,,l:up.down,,l:RR1.RbRb1,,l:bal.b,,l:bal,,l:opp,,l:bias}}\\[2pt] Neuron normal forms and norm identities};
\node[grp] (mE) at (6.695,-6.535) {\textbf{\cref{l:cc,l:st,l:all.P1,l:all.Pb1}}\\[2pt] Failed geometry \\ $\Rightarrow$ descent direction};
\node[grp] (mF) at (6.695,-8.935) {\textbf{\cref{l:attain.int,l:unique.st}}\\[2pt] Attainment and uniqueness};
\draw[a] (mB) -- (mD);  \draw[a] (mD) -- (mE);
\node[ext] (lA) at (2.065,-1.735) {\textbf{\cref{th:loc.min.MFCQ.KKT}}\\[2pt] Prior work: \\ local min.\ $+$ MFCQ \\ $\Rightarrow$ KKT};
\node[grp] (lB) at (2.065,-4.135) {\textbf{\cref{pr:dd.not.KKT}}\\[2pt] Descent direction \\ $\Rightarrow$ not KKT};
\node[grp] (lC) at (2.065,-6.535) {\textbf{\cref{l:sh,l:tight.P1,l:tight.Pb1}}\\[2pt] Switch hugging: geometry and descent};
\node[grp] (lD) at (2.065,-8.935) {\textbf{\cref{l:feas.MFCQ}}\\[2pt] MFCQ holds at every feasible point};
\node[grp] (rB) at (11.33,-1.735) {\textbf{\cref{l:it,l:L1}}\\[2pt] Interval turning: geometry and descent};
\node[grp] (rA) at (11.33,-4.135) {\textbf{\cref{pr:dd.not.Clarke}}\\[2pt] Descent direction \\ $\Rightarrow$ not Clarke stationary};
\node[grp] (rD) at (11.33,-6.535) {\textbf{\cref{l:so}}\\[2pt] Not switch optimal \\ $\Rightarrow$ not Clarke stationary};
\node[grp] (rC) at (11.33,-8.935) {\textbf{\cref{l:O1.Ob1,c:pos.margin,l:attain.r.l}}\\[2pt] Uniqueness, attainment, margin positivity};
\draw[a] (rC) -- (rD);  \draw[a] (rA) -- (rD);
\draw[A] (4.63,-5.335) -- (4.13,-5.335);
\draw[A] (8.76,-5.335) -- (9.26,-5.335);
\node[res,text width=3.92cm,minimum height=2.20cm] (T1) at (2.065,-12.11)
  {\textbf{\cref{th:inter}}\\[2pt] Minimal-norm interpolators: exact characterizations; KKT $\Rightarrow$ minimizer};
\node[res,text width=3.92cm,minimum height=2.20cm] (T6) at (6.695,-12.11)
  {\textbf{\cref{th:min.reg.loss}}\\[2pt] Regularized-loss minimizers: exact characterizations; positive margin $+$ Clarke stationary $\Rightarrow$ minimizer};
\node[res,text width=3.92cm,minimum height=2.20cm] (T7) at (11.33,-12.11)
  {\textbf{\cref{th:lim}}\\[2pt] Margin-normalized approximate Clarke \\ stationary points tend to KKT points};
\draw[A] (2.065,-10.094) -- (2.065,-11.01);
\draw[A] (6.695,-10.094) -- (6.695,-10.552) -- (2.065,-10.552) -- (2.065,-11.01);
\draw[A] (6.695,-10.094) -- (6.695,-11.01);
\draw[A] (11.33,-10.094) -- (11.33,-10.552) -- (6.695,-10.552) -- (6.695,-11.01);
\draw[a] (11.33,-10.094) -- (11.33,-11.01);
\draw[a] (2.065,-13.21) -- (2.065,-13.68) -- (11.33,-13.68) -- (11.33,-13.21);
\node[grp,text width=3.92cm,minimum height=1.90cm] (PA) at (2.065,-15.06)
  {\textbf{\cref{c:inter}}\\[2pt] Consequence: \\ privacy attack};
\node[res,text width=3.92cm,minimum height=1.90cm] (LS) at (6.695,-15.06)
  {\textbf{\cref{th:KKT.not}}\\[2pt] Counterexamples: \\ landscapes without the skip connection};
\node[ext,text width=3.92cm,minimum height=1.90cm] (PW) at (11.33,-15.06)
  {\textbf{\cref{th:appr.KKT,pr:margin}}\\[2pt] Prior work and margin normalization};
\draw[a] (2.065,-13.21) -- (2.065,-14.11);
\draw[a] (2.065,-13.21) -- (2.065,-13.68) -- (6.695,-13.68) -- (6.695,-14.11);
\draw[a] (11.33,-14.11) -- (11.33,-13.21);
\end{tikzpicture}
\caption{Each box lists in bold the results that play a single role, followed by a description of that role.  The three dotted panels separate the machinery used for both \cref{th:inter} and \cref{th:min.reg.loss} from what is needed only for one of them.  An arrow means that what is at its head is proved using what is at its tail.  Dashed boxes indicate results that are quoted or derived from prior work, and shaded boxes indicate the main theorems.}
\end{figure}
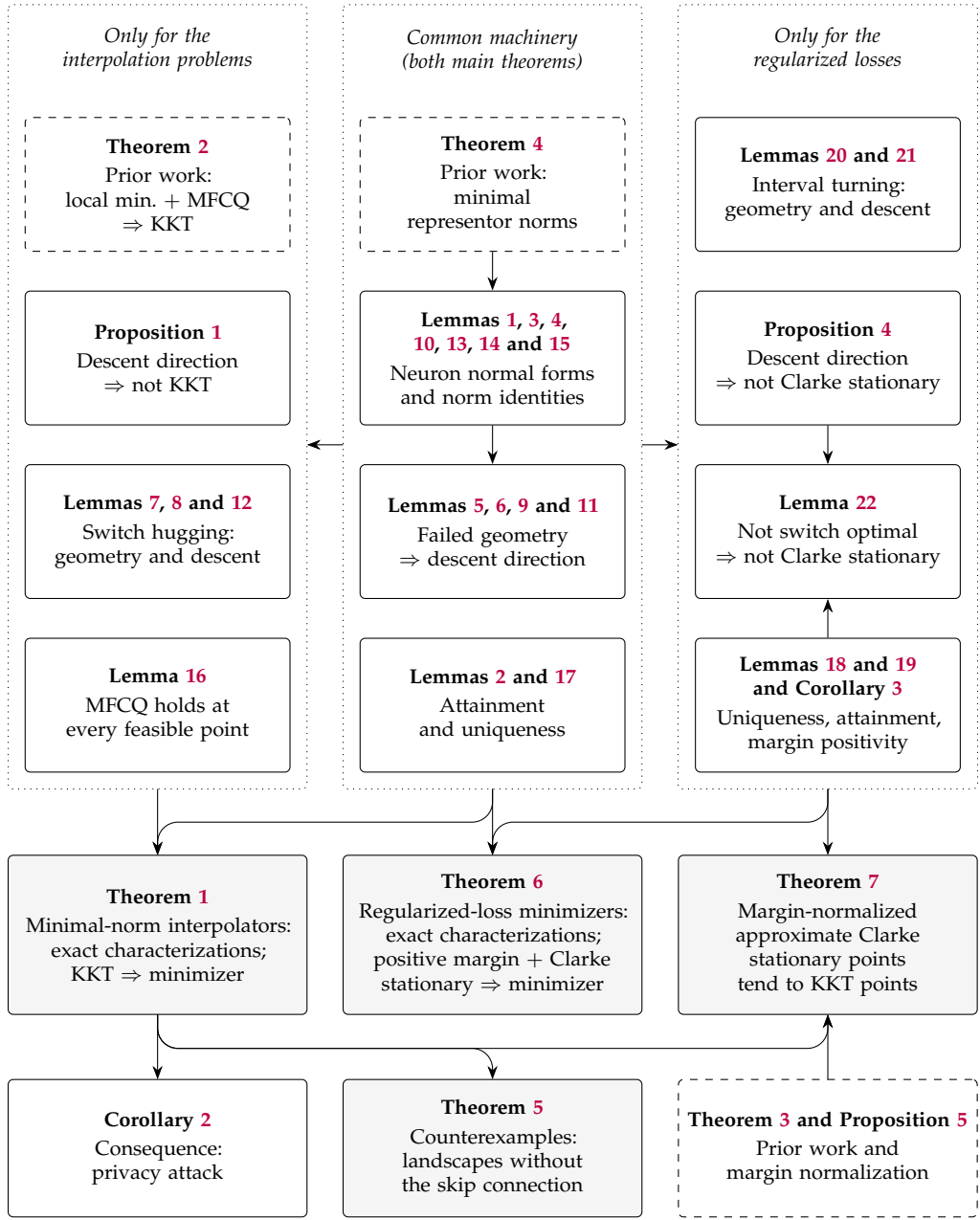

\clearpage
\section{Background on the Clarke subdifferential, KKT points, and MFCQ}
\label{s:Clarke.KKT.MFCQ}

For the reader's convenience, in this section we review a number of standard definitions and results (and for completeness we prove some of them) from nonsmooth analysis and first-order optimality, that we use in this work.  The source of nonsmoothness in our setting is the ReLU nonlinearity, which is not differentiable at~$0$.

\paragraph{Clarke subdifferential and Clarke directional derivative.}

Suppose a function $F \colon \mathbb{R}^N \to \mathbb{R}$ is \emph{locally Lipschitz}, i.e., every point $v \in \mathbb{R}^N$ has a neighborhood~$U$ such that $F$~is Lipschitz continuous on~$U$.  By Rademacher's theorem \citep[see, e.g.,][Theorem~9.1.2]{borweinlewis}, then $F$~is differentiable almost everywhere.  The Clarke subdifferential of~$F$ at a point~$v$ is the convex hull
\[\partial F(v) \coloneqq
  \conv \Bigl\{\lim_{i \to \infty} \nabla F\bigl(u^{(i)}\bigr)
               \Bigm|
               \lim_{i \to \infty} u^{(i)} = v \text{ and }
               \nabla F\bigl(u^{(i)}\bigr) \text{ exists for all } i
        \Bigr\}.\]
Its members are the subgradients, it is nonempty and compact for all~$v$, and equals the singleton $\{\nabla F(v)\}$ if $F$~is continuously differentiable at~$v$ \citep{clarke1975generalized}.  The Clarke directional derivative of~$F$ at~$v$ along~$u$ is then the maximum
\[\mathrm{D}^\circ F(v)[u] \coloneqq
  \max \bigl\{\langle \gamma, u \rangle
              \bigm|
              \gamma \in \partial F(v)
       \bigr\},\]
and it equals the ordinary directional derivative $\mathrm{D} F(v)[u]$ if $F$~is continuously differentiable at~$v$.

For example, the Clarke subdifferential of the ReLU function is
\[\partial \sigma(z) =
  \begin{cases}
  \{0\}  & \text{if } z < 0, \\
  [0, 1] & \text{if } z = 0, \\
  \{1\}  & \text{if } z > 0,
  \end{cases}\]
so its Clarke directional derivatives at~$0$ along~$-1$ and~$1$ are
\begin{align}
\mathrm{D}^\circ \sigma(0)[-1] & = 0 &
\mathrm{D}^\circ \sigma(0)[1]  & = 1.
\end{align}

\paragraph{KKT points.}

Suppose also $G_1, \dots, G_J \colon \mathbb{R}^N \to \mathbb{R}$ are locally Lipschitz.  Following \citet[section~2.2]{DuttaDTA13}, we have that $v \in \mathbb{R}^N$ is a Karush-Kuhn-Tucker (KKT) point of the problem
\begin{equation}
\inf_v F(v)
\quad \text{such that }
\forall j \in [J], \, G_j(v) \leq 0
\tag{$\mathrm{Q}$}\label[prob]{eq:Q}
\end{equation}
if and only if there exist multipliers $\mu_1, \dots, \mu_J \geq 0$ such that:
\begin{description}[labelwidth=\widthof{\textbf{(complementary slackness)}},align=right]
\item[(feasibility)]
$\forall j \in [J], \, G_j(v) \leq 0$;
\item[(equilibrium inclusion)]
$\mathbf{0} \in \partial F(v) + \sum_{j = 1}^J \mu_j \, \partial G_j(v)$;
\item[(complementary slackness)]
$\forall j \in [J], \, \mu_j \, G_j(v) = 0$.
\end{description}

The following basic proposition provides a sufficient condition for impossibility that a feasible point is KKT: existence of a direction along which the derivative of the objective is negative and the derivative of each tight constraint is nonpositive.

\begin{prop}
\label{pr:dd.not.KKT}
Suppose $v \in \mathbb{R}^N$ is feasible for \cref{eq:Q}.  If there exists $u \in \mathbb{R}^N$ such that $\mathrm{D}^\circ F(v)[u] < 0$ and $\mathrm{D}^\circ G_j(v)[u] \leq 0$ for all $j \in [J]$ with $G_j(v) = 0$, then $v$~is not a KKT point of \cref{eq:Q}.
\end{prop}

\begin{proof}
For a contradiction, suppose that $v$~is a KKT point of \cref{eq:Q}.  Then there exist multipliers $\mu_1, \dots, \mu_J \geq 0$ such that $\mu_j \, G_j(v) = 0$ for all $j \in [J]$, and $\mathbf{0} \in \partial F(v) + \sum_{j = 1}^J \mu_j \, \partial G_j(v)$.  Hence $\mathbf{0} \in \partial F(v) + \sum_{G_j(v) = 0} \mu_j \, \partial G_j(v)$, so there exist a subgradient $\varphi \in \partial F(v)$ and subgradients $\gamma_j \in \partial G_j(v)$ for each $j \in [J]$ with $G_j(v) = 0$, such that $\mathbf{0} = \varphi + \sum_{G_j(v) = 0} \mu_j \gamma_j$.  But then 
\begin{align}
0 & =    \Biggl\langle \varphi + \sum_{G_j(v) = 0} \mu_j \gamma_j, u \Biggr\rangle \\
  & =    \langle \varphi, u \rangle
       + \sum_{G_j(v) = 0} \mu_j \langle \gamma_j, u \rangle \\
  & \leq \mathrm{D}^\circ F(v)[u]
       + \sum_{G_j(v) = 0} \mu_j \mathrm{D}^\circ G_j(v)[u] \\
  & <    0.
\ifopt\else\qedhere\fi
\end{align}
\end{proof}

\paragraph{MFCQ.}

Moreover, we say that a feasible point $v \in \mathbb{R}^N$ satisfies Mangasarian-Fromovitz constraint qualification (MFCQ) if and only if there exists $u \in \mathbb{R}^N$ such that $\mathrm{D}^\circ G_j(v)[u] < 0$ for all $j \in [J]$ with $G_j(v) = 0$.

\begin{thm}[corollary of, e.g., {\citealt[Theorem~10.61]{GiorgiJN2023}}]
\label{th:loc.min.MFCQ.KKT}
If a local minimizer of \cref{eq:Q} satisfies MFCQ, then it is a KKT point.
\end{thm}

\paragraph{Approximate KKT points and their limits.}

Following \citet[appendix~C.1]{LyuL20}, supposing that $F, G_1, \dots, G_J \colon \mathbb{R}^N \to \mathbb{R}$ are locally Lipschitz and $\varepsilon, \delta > 0$, we say that $v \in \mathbb{R}^N$ is an $\varepsilon, \delta$-KKT point of \cref{eq:Q} if and only if there exist multipliers $\mu_1, \dots, \mu_J \geq 0$ such that:%
\footnote{We write $\mathbb{B}_2^N$~for the closed Euclidean unit ball in~$\mathbb{R}^N$.}
\begin{description}[labelwidth=\widthof{\textbf{(approximate complementary slackness)}},align=right]
\item[(feasibility)]
$\forall j \in [J], \, G_j(v) \leq 0$;
\item[(approximate equilibrium inclusion)]
$\varepsilon \mathbb{B}_2^N \cap
 \bigl(\partial F(v) + \sum_{j = 1}^J \mu_j \, \partial G_j(v)\bigr)
 \neq \emptyset$;
\item[(approximate complementary slackness)]
$\sum_{j = 1}^J \mu_j \, G_j(v) \geq -\delta$.
\end{description}

\begin{thm}[corollary of {\citealt[Theorem~3.6]{DuttaDTA13}}]
\label{th:appr.KKT}
Suppose that $v_k \in \mathbb{R}^N$ is an $\varepsilon_k, \delta_k$-KKT point of \cref{eq:Q} for all $k \in \mathbb{N}$, that $v_k \to v_\star$, $\varepsilon_k \to 0$, and $\delta_k \to 0$ as $k \to \infty$, and that $v_\star$~satisfies MFCQ.  Then $v_\star$~is a KKT point of \cref{eq:Q}.
\end{thm}

\clearpage
\section{Background on minimal-norm representors}

\begin{table}[t]
\caption{Given a CPA $g \colon \mathbb{R} \to \mathbb{R}$, we consider these four variants of the representor norm minimization problem.  The skip connection is always ``free'', i.e., its parameters $a_0$~and~$b_0$ are never penalized.  Note also that the network width~$m$ is optimized, being determined by the dimension of the parameters vector $\theta = (a_j, w_j, b_j)_{j = 1}^m \in \mathbb{R}^{3 m}$.}
\label{tab:repr}
\centering
\ifopt\else\vspace{2ex}\fi
\begin{tabular}{m{.12\textwidth}m{.36\textwidth}m{.43\textwidth}}
&
\multicolumn{1}{c}{\bfseries biases not penalized}
&
\multicolumn{1}{c}{\bfseries biases penalized}
\\[-1ex]
{\bfseries\shortstack[r]{skip \\ connection \\ absent}}
&
\begin{equation}
\begin{gathered}
\inf_\theta
\frac{1}{2}
\sum_{j = 1}^m (a_j^2 + w_j^2) \\
\text{such that }
f_\theta = g
\end{gathered}
\tag{$\mathrm{R}$}\label[prob]{eq:R}
\end{equation}
&
\begin{equation}
\begin{gathered}
\inf_\theta
\frac{1}{2}
\sum_{j = 1}^m (a_j^2 + w_j^2 + b_j^2) \\
\text{such that }
f_\theta = g
\end{gathered}
\tag{$\mathrm{R^b}$}\label[prob]{eq:Rb}
\end{equation}
\\[-4ex]
{\bfseries\shortstack[r]{skip \\ connection \\ present}}
&
\begin{equation}
\begin{gathered}
\inf_{\theta, a_0, b_0}
\frac{1}{2}
\sum_{j = 1}^m (a_j^2 + w_j^2) \\
\text{such that }
f_{\theta, a_0, b_0} = g
\end{gathered}
\tag{$\mathrm{R_1}$}\label[prob]{eq:R1}
\end{equation}
&
\begin{equation}
\begin{gathered}
\inf_{\theta, a_0, b_0}
\frac{1}{2}
\sum_{j = 1}^m (a_j^2 + w_j^2 + b_j^2) \\
\text{such that }
f_{\theta, a_0, b_0} = g
\end{gathered}
\tag{$\mathrm{R^b_1}$}\label[prob]{eq:Rb1}
\end{equation}
\end{tabular}
\end{table}

After a couple of standard warm-up propositions on the functional expressive power of univariate two-layer ReLU networks, here we collate and state with the same notation the main known results on the representor norm minimization problems in \Cref{tab:repr} (cf.\ question~\ref{q:repr} in \cref{s:intro}), which we build on in this work.

\begin{prop}
\label{pr:skip.expr}
The class of functions representable by a network remains the same without a skip connection, i.e., if $a_0 = 0$ and $b_0 = 0$.
\end{prop}

\begin{proof}
Observe that
$a_0 x + b_0 =
 \sigma( a_0 x + b_0) -
 \sigma(-a_0 x - b_0)$.
\end{proof}

\begin{prop}
\label{pr:repr}
A function from~$\mathbb{R}$ to~$\mathbb{R}$ is representable by a two-layer ReLU network if and only if it is CPA.
\end{prop}

\begin{proof}
For every two-layer ReLU network, the function it represents is CPA since this class contains constant functions and the identity function, and is closed under linear combinations and under applying the ReLU nonlinearity.

Suppose $g \colon \mathbb{R} \to \mathbb{R}$ is CPA, with kinks $(p_k)_{k = 1}^q$ and slopes $(s_k)_{k = 0}^q$.  Then we have, for all $x \in \mathbb{R}$,
\begin{align}
g(x) & =
s_0 (x - p_1) + g(p_1) +
\sum_{k = 1}^q (s_k - s_{k - 1}) \, \sigma(x - p_k).
\ifopt\else\qedhere\fi
\end{align}
\end{proof}

\begin{thm}%
\ifopt
\else%
[corollary of {\citealt[Theorem~3.1]{savarese2019infinite}}, {\citealt[Lemma~1]{OngieWSS20}}, and {\citealt[Theorem~4]{BoursierF23}}]
\fi
\label{th:repr}
The minimal norms of representors of a CPA $g \colon \mathbb{R} \to \mathbb{R}$ with kinks $(p_k)_{k = 1}^q$ and slopes $(s_k)_{k = 0}^q$ are:
\begin{description}
\item[\cref{eq:R}:]
$\max \bigl\{\sum_{k = 1}^q \lvert s_k - s_{k - 1} \rvert,
             \lvert s_0 + s_q \rvert\bigr\}$;
\item[\cref{eq:R1}:]
$\sum_{k = 1}^q \lvert s_k - s_{k - 1} \rvert$;
\item[\cref{eq:Rb}:]
$\begin{aligned}[t]
 & \textstyle \sum_{k = 1}^q \sqrt{1 + p_k^2} \, \lvert s_k - s_{k - 1} \rvert \\
 & \textstyle + \inf_{(\varphi_k)_{k = 1}^q \in [-1, 1]^q}
                h\Bigl(\sum_{k = 1}^q \varphi_k (s_k - s_{k - 1})
                       - (s_0 + s_q), \\
 & \textstyle \qquad\qquad\qquad\qquad\qquad
                       \sum_{k = 1}^q
                       (\lvert p_k \rvert - \varphi_k p_k) (s_k - s_{k - 1})
                       - 2 g(0)\Bigr) \\
& \text{where }
  h(z_1, z_2) \coloneqq
  \begin{cases}
  \sqrt{z_1^2 + z_2^2}
  & \text{if } \lvert z_1 \rvert \geq \sqrt{3} \, \lvert z_2 \rvert, \\
  \frac{\sqrt{3}}{2} \lvert z_1 \rvert + \frac{1}{2} \lvert z_2 \rvert
  & \text{if } \lvert z_1 \rvert <    \sqrt{3} \, \lvert z_2 \rvert;
  \end{cases}
 \end{aligned}$
\item[\cref{eq:Rb1}:]
$\sum_{k = 1}^q \sqrt{1 + p_k^2} \, \lvert s_k - s_{k - 1} \rvert$.
\end{description}
Moreover, for \cref{eq:R1,,eq:Rb,,eq:Rb1}, the minima are attained; and if \cref{eq:R} is modified to
\[\inf_{\theta, b_0}
  \frac{1}{2}
  \sum_{j = 1}^m (a_j^2 + w_j^2)
  \quad \text{such that }
  f_{\theta, 0, b_0} = g,\]
i.e., a free output bias is allowed, then the minimum is unchanged and it is attained.
\end{thm}

\clearpage
\section{Role of the dataset assumption}
\label{s:ass:data}

\cref{ass:data} combines two logically distinct conditions:
\begin{equation}
r \geq 3
\qquad \text{and} \qquad
x_{\iota(2)} < 0 < x_{\iota(r - 1) + 1}.
\label{eq:ass:data}
\end{equation}
The first is a lower bound on the number of label switches. The second places the first two same-label segments strictly to the left of the origin and the last two strictly to its right. In this section, we record separately the roles of these two conditions.

\paragraph{At least three label switches.}

This condition is used in \cref{l:RR1.RbRb1}. A positive-margin classifier then has both a negative and a positive turning interval, which allows its free affine skip connection to be eliminated without increasing either representor norm. Consequently, $\mathrm{R}(g) = \mathrm{R_1}(g)$ and $\mathrm{R^b}(g) = \mathrm{R^b_1}(g)$ for every positive-margin CPA~$g$. This is the step that transfers the function-space characterizations obtained with a skip connection to the corresponding problems without one. The weaker condition $r \geq 2$ suffices for the compactness arguments in \cref{l:attain.int,l:attain.r.l} and for the finite-dimensional problems in \cref{l:O1.Ob1}.

\paragraph{The sign condition.}

The inequalities $x_{\iota(2)} < 0 < x_{\iota(r - 1) + 1}$ are used only when the biases are penalized. Their substantive use is in the outer-kink cases of \cref{l:all.Pb1} and, through the same argument, \cref{l:tight.Pb1}. For example, in the left outer-kink case one has $p_1 < 0$, and one chooses $\widehat{p} \leq p_1$ so that $-w_j^2 (p_1 \widehat{p} + 1) < 0$. The right outer-kink case is symmetric. No corresponding sign condition is needed when the hidden biases are not penalized, because the perturbations in \cref{l:all.P1,l:tight.P1} have regularizer derivative $-a_j^2 < 0$, independently of the kink location. Accordingly, subject only to $r \geq 3$, the conclusions concerning unpenalized biases are unaffected if the sign condition is dropped. These include \cref{th:inter}~\ref{th:inter:P.P1}, the part of \cref{th:inter}~\ref{th:inter:KKT} concerning \cref{eq:P1}, \cref{th:min.reg.loss}~\ref{th:min.reg.loss:L.L1}, the part of \cref{th:min.reg.loss}~\ref{th:min.reg.loss:stat} concerning~\hyperref[eq:L1]{$L\mathrm{^\lambda_1}$}, and the corresponding unpenalized-bias conclusions of \cref{c:min.reg.loss.lim,th:lim}.

By contrast, for penalized biases, the present proofs use the sign condition to exclude certain outer geometric irregularities. Without it, the interior arguments remain valid, including the exclusion of multiple kinks within an intermediate same-label segment, but the present proofs do not establish the full convexity-correct and switch-hugging characterizations. Thus, without additional treatment of the outer cases, the affected statements are \cref{th:inter}~\ref{th:inter:Pb.Pb1}, the part of \cref{th:inter}~\ref{th:inter:KKT} concerning \cref{eq:Pb1}, \cref{th:min.reg.loss}~\ref{th:min.reg.loss:Lb.Lb1}, the part of \cref{th:min.reg.loss}~\ref{th:min.reg.loss:stat} concerning~\hyperref[eq:Lb1]{$L\mathrm{^{\lambda, b}_1}$}, and the penalized-bias conclusions derived from them.

Finally, the sign condition is coordinate dependent when the biases are penalized. Translating~$x$ by~$c$ moves a kink from~$p$ to~$p - c$, and changes its contribution to the penalized representor norm from $\sqrt{1 + p^2} \, \lvert \Delta s \rvert$ to $\sqrt{1 + (p - c)^2} \, \lvert \Delta s \rvert$. Thus, in the penalized-bias setting, an additive translation can enforce the sign condition but is not a without-loss-of-generality reparameterization of the same optimization problem. In the unpenalized-bias setting, the representor cost depends only on the slope changes, and the corresponding conclusions are translation invariant.

\clearpage
\section{Proofs for minimal-norm interpolators}
\label{s:inter}

After restating \cref{th:inter} from \cref{s:main} and showing a privacy attack corollary, we develop and prove a sequence of lemmas culminating in a proof of the theorem.

\thinter*

As a corollary of the exact geometric characterizations and the skip connection's sanitizing effects in \cref{th:inter}, we obtain a powerful privacy attack which is able to extract reliably all training points that are incident to the label switches, given only the classifier in function space (knowledge of the network's parameters is not needed).

\begin{cor}
\label{c:inter}
Under \cref{ass:data}, suppose $\theta, a_0, b_0$ is either a minimizer of \cref{eq:P} or \cref{eq:Pb}, or a KKT point of \cref{eq:P1} or \cref{eq:Pb1}.  Let~$X$ be the set of all $x \in \mathbb{R}$ such that $f_{\theta, a_0, b_0}(x) = \pm 1$ and $f_{\theta, a_0, b_0}$~is not constant on any open interval containing~$x$.  Then all elements of~$X$ are training inputs.  More precisely, $X$~consists exactly of: the last input in the first same-label segment of the training dataset, both first and last inputs in every intermediate segment, and the first input in the last segment.  In particular, the number of switches of $f_{\theta, a_0, b_0}(x)$ as $x$~iterates through a sorted enumeration of~$X$ equals the number~$r$ of label switches in the training dataset.
\end{cor}

\begin{exl}
In \Cref{f:min.fun}, the set~$X$ of inputs extracted as in \cref{c:inter} from any of the three CPAs $g_\text{cyan}$, $g_\text{lime}$, or~$g_\text{pink}$, is shown by the bold dots colored brown.
\end{exl}

\begin{proof}
By \cref{th:inter}, $f_{\theta, a_0, b_0}$~is switch hugging, and so it equals~$\pm 1$ at each $x \in \mathbb{R}$ which is one of: the last input in the first same-label segment of the training dataset, both first and last inputs in every intermediate segment, and the first input in the last segment.  For each such~$x$, by the same property of~$f_{\theta, a_0, b_0}$, it is not constant on any open interval containing~$x$.  Hence it remains to show that the set~$X$ contains no other values, so for a contradiction suppose that $x \in X$ is none of: the last input in the first same-label segment of the training dataset, either first or last input in some intermediate segment, or the first input in the last segment.  Since by \cref{th:inter}, $f_{\theta, a_0, b_0}$~is also convexity correct, necessarily for some intermediate segment with first input~$x'$ and last input~$x''$ we have $x' < x < x''$ and $f_{\theta, a_0, b_0}(x') = f_{\theta, a_0, b_0}(x) = f_{\theta, a_0, b_0}(x'')$.  But then $f_{\theta, a_0, b_0}$~must be constant on the interval $[x', x'']$.
\end{proof}

We now start the sequence of lemmas that leads to the proof of \cref{th:inter}, the first of which shows that, for positive-margin CPAs on datasets with at least three label switches (which is implied by \cref{ass:data}), the skip connection does not lead to lower representational cost.  This fact, which is surprising in view of \cref{th:repr}, can be proved either by arguing directly in terms of the minimal costs given by \cref{th:repr}, or as we opt to do here, by showing how the skip connection can be eliminated through constructions at the network level that do not increase the costs.  We also remark that the constructions we provide are chosen for their clarity, and can be optimized further to reduce the number of additional neurons to a small constant.  Before stating and proving the lemma, we introduce notation for the costs and define the margin precisely.

For a CPA $g \colon \mathbb{R} \to \mathbb{R}$, we write $\mathrm{R}(g)$, $\mathrm{R_1}(g)$, $\mathrm{R^b}(g)$, and~$\mathrm{R^b_1}(g)$ for the minimal norms of its representors that are given in the four cases in \cref{th:repr}, respectively.

Also for a CPA $g \colon \mathbb{R} \to \mathbb{R}$, we write $M(g) \coloneqq \min_{i = 1}^n y_i \, g(x_i)$ for its margin over the dataset $(x_i, y_i)_{i = 1}^n$.  Then for a network $\theta, a_0, b_0$, we may write $M(\theta, a_0, b_0)$ for the margin~$M(f_{\theta, a_0, b_0})$ of its function; and when the skip connection is absent, we may write just~$M(\theta)$ instead of $M(\theta, 0, 0)$.

\begin{lem}
\label{l:RR1.RbRb1}
Suppose the dataset has at least three label switches, i.e., $r \geq 3$, and a CPA $g \colon \mathbb{R} \to \mathbb{R}$ has a positive margin.  Then $\mathrm{R}(g) = \mathrm{R_1}(g)$ and $\mathrm{R^b}(g) = \mathrm{R^b_1}(g)$.  Moreover, the minimum of \cref{eq:R} is attained.
\end{lem}

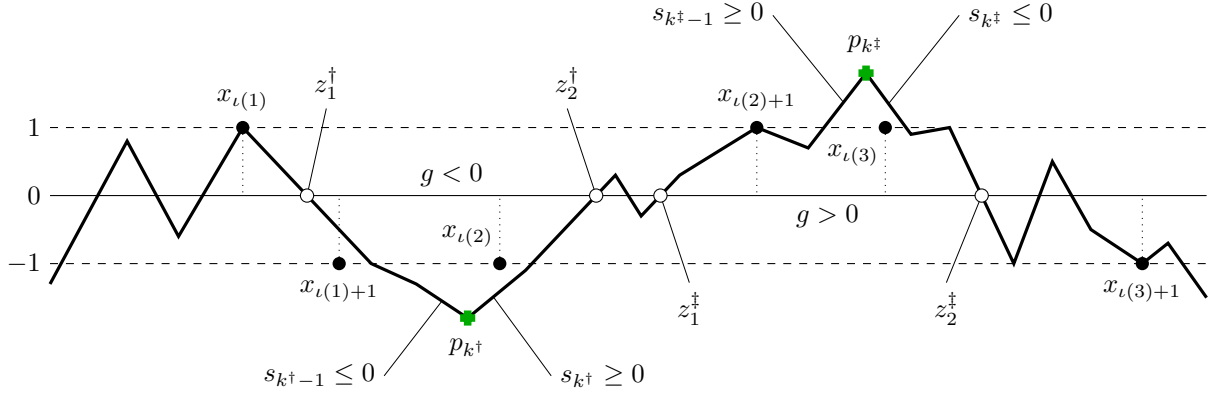
\begin{figure}[t]
\centering
\begin{tikzpicture}[xscale=\ifopt\else\ifarxiv .85\else .7\fi\fi,yscale=.9]
\draw         (0,  0) node[left]  {$0$} -- (18,  0);
\draw[dashed] (0,  1) node[left]  {$1$} -- (18,  1);
\draw[dashed] (0, -1) node[left] {$-1$} -- (18, -1);

\draw[dotted] (3,1) node[point] {} -- (3,0);
\node[above=1mm] at (3,1) {$x_{\iota(1)}$};

\draw[dotted] (4.5,-1) node[point] {} -- (4.5,0);
\node[below=1mm] at (4.5,-1) {$x_{\iota(1)+1}$};

\draw[dotted] (7,-1) node[point] {} -- (7,0);
\node[above=1mm] at (6.5,-1) {$x_{\iota(2)}$};

\draw[dotted] (11,1) node[point] {} -- (11,0);
\node[above=1mm] at (11,1) {$x_{\iota(2)+1}$};

\draw[dotted] (13,1) node[point] {} -- (13,0);
\node[below=1mm] at (12.5,1) {$x_{\iota(3)}$};

\draw[dotted] (17,-1) node[point] {} -- (17,0);
\node[below=1mm] at (17,-1) {$x_{\iota(3)+1}$};

\draw[very thick]
  (0,-1.3)
  -- (1.2,.8)
  -- (2,-.6)
  -- (3,1)
  -- (5,-1)
  -- (5.7,-1.3)
  -- (6.5,-1.8)
  -- (7.4,-1.1)
  -- (8.8,.3)
  -- (9.2,-.3)
  -- (9.8,.3)
  -- (11,1)
  -- (11.8,.7)
  -- (12.7,1.8)
  -- (13.4,.9)
  -- (14,1)
  -- (15,-1)
  -- (15.6,.5)
  -- (16.2,-.5)
  -- (17,-1)
  -- (17.4,-.7)
  -- (18,-1.5);

\draw[very thin]
  (5.2,-2.65)
  node[left] {$s_{k^\dag-1}\leq0$}
  -- (6.1,-1.55);

\draw[very thin]
  (7.8,-2.65)
  node[right] {$s_{k^\dag}\geq0$}
  -- (6.9,-1.49);

\draw[very thin]
  (11.25,2.65)
  node[left] {$s_{k^\ddag-1}\geq0$}
  -- (12.35,1.37);

\draw[very thin]
  (14.15,2.65)
  node[right] {$s_{k^\ddag}\leq0$}
  -- (13.05,1.35);

\draw[very thin]
  (4,0) -- (4.3,1.28)
  node[above] {$z^\dag_1$};
\node[
  circle,
  draw,
  fill=white,
  inner sep=0mm,
  minimum size=1.75mm
] at (4,0) {};

\draw[very thin]
  (8.5,0) -- (8.05,1.28)
  node[above] {$z^\dag_2$};
\node[
  circle,
  draw,
  fill=white,
  inner sep=0mm,
  minimum size=1.75mm
] at (8.5,0) {};

\draw[very thin]
  (9.5,0) -- (9.95,-1.28)
  node[below] {$z^\ddag_1$};
\node[
  circle,
  draw,
  fill=white,
  inner sep=0mm,
  minimum size=1.75mm
] at (9.5,0) {};

\draw[very thin]
  (14.5,0) -- (13.95,-1.28)
  node[below] {$z^\ddag_2$};
\node[
  circle,
  draw,
  fill=white,
  inner sep=0mm,
  minimum size=1.75mm
] at (14.5,0) {};

\pic at (6.5,-1.8) {plus};
\node[below=1.5mm] at (6.5,-1.8) {$p_{k^\dag}$};

\pic at (12.7,1.8) {plus};
\node[above=1.5mm] at (12.7,1.8) {$p_{k^\ddag}$};

\node[above] at (6.25,0) {$g<0$};
\node[below] at (12.1,0) {$g>0$};
\end{tikzpicture}
\caption{Here we illustrate some of the named quantities in the proof of \cref{l:RR1.RbRb1} and their properties.  The bold dots indicate the training points incident to the first three label switches in the dataset, the thick line shows the CPA~$g$, the hollow dots indicate its identified crossings of the $x$-axis, and the green pluses indicate the two chosen kinks.}
\label{f:RR1.RbRb1}
\end{figure}

\begin{proof}
For an illustration of a part of this argument, we refer the reader to \Cref{f:RR1.RbRb1}.

Let $(p_k)_{k = 1}^q$ be the kinks of~$g$ and $(s_k)_{k = 0}^q$ be the slopes of~$g$.

Without loss of generality, $y_1 = 1$.  Since $r \geq 3$, we have at least the following three label switches:
\[ 1 = y_{\iota(1)} \neq y_{\iota(1) + 1} =
  -1 = y_{\iota(2)} \neq y_{\iota(2) + 1} =
   1 = y_{\iota(3)} \neq y_{\iota(3) + 1} =
  -1.\]
Recalling that $M(g) > 0$, there must exist points $z^\dag_1 \in (x_{\iota(1)}, x_{\iota(1) + 1})$ and $z^\dag_2 \in (x_{\iota(1) + 1}, x_{\iota(2) + 1})$ such that
\begin{align}
g(z^\dag_1) = 0 & = g(z^\dag_2) &
g(z) & < 0 \text{ for all } z \in (z^\dag_1, z^\dag_2),
\end{align}
and there must exist points $z^\ddag_1 \in (x_{\iota(2)}, x_{\iota(2) + 1})$ and $z^\ddag_2 \in (x_{\iota(2) + 1}, x_{\iota(3) + 1})$ such that
\begin{align}
g(z^\ddag_1) = 0 & = g(z^\ddag_2) &
g(z) & > 0 \text{ for all } z \in (z^\ddag_1, z^\ddag_2).
\end{align}
Then necessarily $z^\dag_2 \leq z^\ddag_1$.  Also it follows that $g$~must have kinks $p_{k^\dag} \in (z^\dag_1, z^\dag_2)$ and $p_{k^\ddag} \in (z^\ddag_1, z^\ddag_2)$ such that their incident slopes satisfy
\begin{align}
s_{k^\dag - 1}  & \leq 0 &
s_{k^\dag}      & \geq 0 &
s_{k^\ddag - 1} & \geq 0 &
s_{k^\ddag}     & \leq 0.
\end{align}

For each $k \in [q]$, let:
\begin{align}
\mu^\dag_k & \coloneqq
\begin{cases}
s_k - s_{k - 1}  & \text{if } k < k^\dag, \\
- s_{k^\dag - 1} & \text{if } k = k^\dag, \\
0                & \text{if } k > k^\dag;
\end{cases} &
\nu^\dag_k & \coloneqq
\begin{cases}
0               & \text{if } k < k^\dag, \\
s_{k^\dag}      & \text{if } k = k^\dag, \\
s_k - s_{k - 1} & \text{if } k > k^\dag.
\end{cases}
\end{align}
Then let $\theta^\dag$~consist of neuron parameters that give the network function
\[f_{\theta^\dag}(x)
= \sum_{k = 1}^{q}
  \Bigl(
  \sgn(\mu^\dag_k) \sqrt{\lvert \mu^\dag_k \rvert} \,
  \sigma\bigl(-\sqrt{\lvert \mu^\dag_k \rvert} (x - p_k)\bigr)
+ \sgn(\nu^\dag_k) \sqrt{\lvert \nu^\dag_k \rvert} \,
  \sigma\bigl( \sqrt{\lvert \nu^\dag_k \rvert} (x - p_k)\bigr)
  \Bigr).\]
It is straightforward to check that $f_{\theta^\dag}$~has the same kinks and slopes as~$g$, and that $f_{\theta^\dag}(p_{k^\dag}) = 0$, which means that $f_{\theta^\dag}(x) = g(x) - g(p_{k^\dag})$ for all $x \in \mathbb{R}$.

Analogously, for each $k \in [q]$, let:
\begin{align}
\mu^\ddag_k & \coloneqq
\begin{cases}
s_k - s_{k - 1}   & \text{if } k < k^\ddag, \\
- s_{k^\ddag - 1} & \text{if } k = k^\ddag, \\
0                 & \text{if } k > k^\ddag;
\end{cases} &
\nu^\ddag_k & \coloneqq
\begin{cases}
0                & \text{if } k < k^\ddag, \\
s_{k^\ddag}      & \text{if } k = k^\ddag, \\
s_k - s_{k - 1}  & \text{if } k > k^\ddag.
\end{cases}
\end{align}
Then let $\theta^\ddag$~consist of neuron parameters that give the network function
\[f_{\theta^\ddag}(x)
= \sum_{k = 1}^{q}
  \Bigl(
  \sgn(\mu^\ddag_k) \sqrt{\lvert \mu^\ddag_k \rvert} \,
  \sigma\bigl(-\sqrt{\lvert \mu^\ddag_k \rvert} (x - p_k)\bigr)
+ \sgn(\nu^\ddag_k) \sqrt{\lvert \nu^\ddag_k \rvert} \,
  \sigma\bigl( \sqrt{\lvert \nu^\ddag_k \rvert} (x - p_k)\bigr)
  \Bigr).\]
It is straightforward to check that $f_{\theta^\ddag}$~has the same kinks and slopes as~$g$, and that $f_{\theta^\ddag}(p_{k^\ddag}) = 0$, which means that $f_{\theta^\ddag}(x) = g(x) - g(p_{k^\ddag})$ for all $x \in \mathbb{R}$.

Now, recalling that $g(p_{k^\dag}) < 0$ and $g(p_{k^\ddag}) > 0$, let
\[\xi \coloneqq \frac{g(p_{k^\ddag})}{g(p_{k^\ddag}) - g(p_{k^\dag})} \in (0, 1).\]
Then we have $g = \xi f_{\theta^\dag} + (1 - \xi) f_{\theta^\ddag}$, and hence $g = f_{\theta^\star}$, where $\theta^\star \coloneqq \bigl(\sqrt{\xi} \, \theta^\dag, \sqrt{1 - \xi} \, \theta^\ddag\bigr)$, i.e., the convex combination of network functions $f_{\theta^\dag}$ and~$f_{\theta^\ddag}$ is implemented by multiplying the network parameters by~$\sqrt{\xi}$ and $\sqrt{1 - \xi}$ respectively, and concatenating the results.

Therefore
\begin{align}
\mathrm{R}(g)
& \leq \xi
       \sum_{k = 1}^q \bigl(\lvert \mu^\dag_k \rvert + \lvert \nu^\dag_k \rvert\bigr)
     + (1 - \xi)
       \sum_{k = 1}^q \bigl(\lvert \mu^\ddag_k \rvert + \lvert \nu^\ddag_k \rvert\bigr) \\
& =    \xi
       \sum_{k = 1}^q \lvert s_k - s_{k - 1} \rvert
     + (1 - \xi)
       \sum_{k = 1}^q \lvert s_k - s_{k - 1} \rvert \\
& =    \sum_{k = 1}^q \lvert s_k - s_{k - 1} \rvert \\
& =    \mathrm{R_1}(g),
\end{align}
where: the inequality is by the definitions of $\theta^\star$, $\theta^\dag$, and~$\theta^\ddag$; the first equality is by the definitions of $\mu^\dag_k$, $\nu^\dag_k$, $\mu^\ddag_k$, and~$\nu^\ddag_k$ as well as by recalling that $\lvert s_{k^\dag} - s_{k^\dag - 1} \rvert = \lvert s_{k^\dag - 1} \rvert + \lvert s_{k^\dag} \rvert$ and $\lvert s_{k^\ddag} - s_{k^\ddag - 1} \rvert = \lvert s_{k^\ddag - 1} \rvert + \lvert s_{k^\ddag} \rvert$; and the third equality is by \cref{th:repr}.

Hence $\mathrm{R}(g) = \mathrm{R_1}(g)$ also by \cref{th:repr}, so the last inequality above is an equality, and thus the minimum of \cref{eq:R} is attained by the network parameters~$\theta^\star$.

Finally, by considering the following alternative expressions
\begin{align}
f_{\theta^\dag}(x)
& = \sum_{k = 1}^{q}
    \left(
    \sgn(\mu^\dag_k) \sqrt[4]{1 + p_k^2} \sqrt{\lvert \mu^\dag_k \rvert} \,
    \sigma\Biggl(-\frac{\sqrt{\lvert \mu^\dag_k \rvert}}
                       {\sqrt[4]{1 + p_k^2}} (x - p_k)\Biggr)
    \right. \\
& \qquad\qquad\left.
  + \sgn(\nu^\dag_k) \sqrt[4]{1 + p_k^2} \sqrt{\lvert \nu^\dag_k \rvert} \,
    \sigma\Biggl( \frac{\sqrt{\lvert \nu^\dag_k \rvert}}
                       {\sqrt[4]{1 + p_k^2}} (x - p_k)\Biggr)
    \right) \\
f_{\theta^\ddag}(x)
& = \sum_{k = 1}^{q}
    \left(
    \sgn(\mu^\ddag_k) \sqrt[4]{1 + p_k^2} \sqrt{\lvert \mu^\ddag_k \rvert} \,
    \sigma\Biggl(-\frac{\sqrt{\lvert \mu^\ddag_k \rvert}}
                       {\sqrt[4]{1 + p_k^2}} (x - p_k)\Biggr)
    \right. \\
& \qquad\qquad\left.
  + \sgn(\nu^\ddag_k) \sqrt[4]{1 + p_k^2} \sqrt{\lvert \nu^\ddag_k \rvert} \,
    \sigma\Biggl( \frac{\sqrt{\lvert \nu^\ddag_k \rvert}}
                       {\sqrt[4]{1 + p_k^2}} (x - p_k)\Biggr)
    \right),
\end{align}
we obtain
\begin{align}
\mathrm{R^b}(g)
& \leq \xi
       \sum_{k = 1}^q
       \sqrt{1 + p_k^2} \,
       \bigl(\lvert \mu^\dag_k \rvert + \lvert \nu^\dag_k \rvert\bigr)
     + (1 - \xi)
       \sum_{k = 1}^q
       \sqrt{1 + p_k^2} \,
       \bigl(\lvert \mu^\ddag_k \rvert + \lvert \nu^\ddag_k \rvert\bigr) \\
& =    \xi
       \sum_{k = 1}^q \sqrt{1 + p_k^2} \, \lvert s_k - s_{k - 1} \rvert
     + (1 - \xi)
       \sum_{k = 1}^q \sqrt{1 + p_k^2} \, \lvert s_k - s_{k - 1} \rvert \\
& =    \sum_{k = 1}^q \sqrt{1 + p_k^2} \, \lvert s_k - s_{k - 1} \rvert \\
& =    \mathrm{R^b_1}(g),
\end{align}
so by \cref{th:repr} we have $\mathrm{R^b}(g) = \mathrm{R^b_1}(g)$.
\end{proof}

The next lemma shows that the optima of the interpolator norm minimization problems with the skip connection are attained, even for any fixed network width for which interpolation is possible (i.e., $m \geq r - 1$).

\begin{lem}
\label{l:attain.int}
For each of \cref{eq:P1,eq:Pb1} restricted to any network width $m \in \mathbb{N}$ that is feasible, the minimum is attained.
\end{lem}

\begin{proof}
If the dataset has no label switch, then a constant skip connection interpolates it. If it has exactly one label switch, then an affine skip connection interpolates it. In either case, the value zero is attained in both problems. We may therefore assume that $r \geq 2$.

Fix a feasible width~$m$, and consider \cref{eq:P1}. Let $\bigl(\theta^{(k)}, a_0^{(k)}, b_0^{(k)}\bigr)_{k \in \mathbb{N}}$ be a minimizing sequence, then the maximum~$N$ of its objective values is finite. We first replace each term of the sequence by an equivalent parameterization, without changing its values at the training inputs and without increasing its objective, as follows.

For every neuron with $w_j^{(k)} \neq 0$, write $p_j^{(k)} \coloneqq -b_j^{(k)} / w_j^{(k)}$ for its kink. If $p_j^{(k)} \notin [x_1, x_n]$, then $x \mapsto a_j^{(k)} \, \sigma\bigl(w_j^{(k)} x + b_j^{(k)}\bigr)$ is affine on $[x_1, x_n]$, so its contribution on all training inputs can be absorbed into the free affine skip connection, after which the neuron can be replaced by the zero neuron. Similarly, if $w_j^{(k)} = 0$, then the neuron contributes a constant function, which can be absorbed into the skip bias. Thus, we may assume that every nonzero neuron satisfies $p_j^{(k)} \in [x_1, x_n]$.

By positive homogeneity of the ReLU, each nonzero neuron can also be rescaled, without changing the represented function, so that $\lvert a_j^{(k)} \rvert = \lvert w_j^{(k)} \rvert$. This rescaling minimizes $\frac{1}{2} \Bigl(\bigl(a_j^{(k)}\bigr)^2 + \bigl(w_j^{(k)}\bigr)^2\Bigr)$ over the positive-rescaling orbit of the neuron, and hence does not increase the objective. Consequently, $\lvert a_j^{(k)} \rvert, \lvert w_j^{(k)} \rvert \leq \sqrt{2 N}$, and also $\lvert b_j^{(k)} \rvert = \lvert w_j^{(k)} p_j^{(k)} \rvert \leq \sqrt{2 N} \, \max_{i = 1}^n \lvert x_i \rvert$.

It follows that the parameter vectors~$\theta^{(k)}$ lie in a bounded set. In particular, the pre-skip output vectors $h^{(k)} := \bigl(f_{\theta^{(k)}, 0, 0}(x_i)\bigr)_{i = 1}^n$ form a bounded sequence in~$\mathbb{R}^n$. We next show that the skip parameters are bounded.

Suppose, to the contrary, that after passing to a subsequence, $\bigl\lVert a_0^{(k)}, b_0^{(k)} \bigr\rVert_2 \to \infty$. Passing to a further subsequence, we may assume that $\bigl(a_0^{(k)}, b_0^{(k)}\bigr) / \bigl\lVert a_0^{(k)}, b_0^{(k)} \bigr\rVert_2 \to (\alpha_0, \beta_0)$ for some unit vector $(\alpha_0, \beta_0)$. Feasibility gives $y_i \bigl(a_0^{(k)} x_i + b_0^{(k)} + h_i^{(k)}\bigr) \geq 1$ for every $i \in [n]$. Dividing by $\bigl\lVert a_0^{(k)}, b_0^{(k)} \bigr\rVert_2$ and letting $k \to \infty$, using boundedness of $h^{(k)}$, yields $y_i (\alpha_0 x_i + \beta_0) \geq 0$ for every $i \in [n]$. Because $r \geq 2$, there exist three training inputs $x_{i_1} < x_{i_2} < x_{i_3}$ whose labels alternate. Multiplying the affine function $x \mapsto \alpha_0 x + \beta_0$ by the common label of the first and third points if necessary, we obtain
\begin{align}
\alpha_0 x_{i_1} + \beta_0 & \geq 0 &
\alpha_0 x_{i_2} + \beta_0 & \leq 0 &
\alpha_0 x_{i_3} + \beta_0 & \geq 0.
\end{align}
Since $x_{i_2}$~is a strict convex combination of $x_{i_1}$ and~$x_{i_3}$, the value $\alpha_0 x_{i_2} + \beta_0$ is the same strict convex combination of $\alpha_0 x_{i_1} + \beta_0$ and $\alpha_0 x_{i_3} + \beta_0$. It is therefore nonnegative, and hence it must equal zero. Since both outer values are nonnegative, equality of their strict convex combination to zero forces both of them to be zero. Thus the affine function vanishes at two distinct points and is consequently identically zero. Hence $\alpha_0 = \beta_0 = 0$, contradicting that $(\alpha_0, \beta_0)$ is a unit vector. It follows that the skip parameters are bounded.

We have therefore obtained a minimizing sequence contained in a compact subset of~$\mathbb{R}^{3 m + 2}$, so we may pass to a convergent subsequence. The feasible set is closed because the network outputs at the training inputs depend continuously on the parameters. Hence the limit is feasible. Also the objective is continuous, so its value at the limit is the infimum. Thus the minimum of \cref{eq:P1}, restricted to width~$m$, is attained.

The proof for \cref{eq:Pb1} is analogous and slightly simpler. By positive homogeneity, every nonzero neuron may be rescaled so that $\lvert a_j^{(k)} \rvert = \sqrt{\bigl(w_j^{(k)}\bigr)^2 + \bigl(b_j^{(k)}\bigr)^2}$. This rescaling minimizes $\frac{1}{2} \Bigl(\bigl(a_j^{(k)}\bigr)^2 + \bigl(w_j^{(k)}\bigr)^2 + \bigl(b_j^{(k)}\bigr)^2\Bigr)$ over the positive-rescaling orbit of the neuron. The bounded objective of a minimizing sequence therefore directly bounds all coordinates of~$\theta^{(k)}$. The pre-skip outputs on the finite training set are consequently bounded. The same normalization argument using three ordered inputs with alternating labels then shows that the skip parameters are bounded. Thus a minimizing sequence can again be chosen in a compact set, and closedness of the feasible set and continuity of the objective imply that the minimum of \cref{eq:Pb1}, restricted to width~$m$, is attained.
\end{proof}

The orientation of a hidden neuron is not intrinsic when a free affine skip connection is present. Reversing its hidden weight and bias changes its ReLU contribution only by an affine function, which can be absorbed exactly by the skip connection. The next lemma records this reparameterization and, crucially, shows that it also transports arbitrary parameter-space directions while preserving the Clarke directional derivatives of the network outputs, their negatives, and both regularizers considered in this work.%
\footnote{The two regularizers are everywhere continuously differentiable, so their Clarke directional derivatives equal their ordinary directional derivatives.}
Consequently, in the subsequent descent constructions we may choose the orientations of the relevant neurons without loss of generality, including when a training input coincides with one of their kinks.

\begin{lem}
\label{l:rev}
Suppose $J \subseteq [m]$, $\widehat{\theta} \coloneqq \bigl(\widehat{a}_j, \widehat{w}_j, \widehat{b}_j\bigr)_{j = 1}^m$ is obtained from~$\theta$ by reversing the orientations of the hidden neurons~$J$, i.e.,
\[\bigl(\widehat{a}_j, \widehat{w}_j, \widehat{b}_j\bigr) \coloneqq
  \begin{cases}
  (a_j, -w_j, -b_j) & \text{if } j \in    J, \\
  (a_j,  w_j,  b_j) & \text{if } j \notin J,
  \end{cases}\]
and
\begin{align}
\widehat{a}_0 & \coloneqq
a_0 + \sum_{j \in J} a_j w_j &
\widehat{b}_0 & \coloneqq
b_0 + \sum_{j \in J} a_j b_j.
\end{align}
\begin{enumerate}[(i)]
\item
\label{l:rev:f}
We have that $f_{\widehat{\theta}, \widehat{a}_0, \widehat{b}_0} = f_{\theta, a_0, b_0}$.
\item
\label{l:rev:D}
For every direction $\widehat{\theta}', \widehat{a}'_0, \widehat{b}'_0$, letting $\theta' \coloneqq (a'_j, w'_j, b'_j)_{j = 1}^m$ be obtained from~$\widehat{\theta}'$ by reversing the orientations at the hidden neurons~$J$, i.e.,
\[(a'_j, w'_j, b'_j\bigr) \coloneqq
  \begin{cases}
  \bigl(\widehat{a}'_j, -\widehat{w}'_j, -\widehat{b}'_j\bigr)
  & \text{if } j \in    J, \\[.5ex]
  \bigl(\widehat{a}'_j,  \widehat{w}'_j,  \widehat{b}'_j\bigr)
  & \text{if } j \notin J,
  \end{cases}\]
and
\begin{align}
a'_0 & \coloneqq
\widehat{a}'_0 + \sum_{j \in J} \bigl(\widehat{a}'_j \widehat{w}_j
                                    + \widehat{a}_j  \widehat{w}'_j\bigr) &
b'_0 & \coloneqq
\widehat{b}'_0 + \sum_{j \in J} \bigl(\widehat{a}'_j \widehat{b}_j
                                    + \widehat{a}_j  \widehat{b}'_j\bigr),
\end{align}
we have that, for all $x \in \mathbb{R}$,
\begin{align}
\mathrm{D}^\circ
\bigl(f_{\theta, a_0, b_0}(x)\bigr)
[\theta', a'_0, b'_0] & =
\mathrm{D}^\circ
\bigl(f_{\widehat{\theta}, \widehat{a}_0, \widehat{b}_0}(x)\bigr)
\bigl[\widehat{\theta}', \widehat{a}'_0, \widehat{b}'_0\bigr]
\\
\mathrm{D}^\circ
\bigl(-f_{\theta, a_0, b_0}(x)\bigr)
[\theta', a'_0, b'_0] & =
\mathrm{D}^\circ
\bigl(-f_{\widehat{\theta}, \widehat{a}_0, \widehat{b}_0}(x)\bigr)
\bigl[\widehat{\theta}', \widehat{a}'_0, \widehat{b}'_0\bigr]
\\
\mathrm{D}
\biggl(\frac{1}{2} \sum_{j = 1}^m (a_j^2 + w_j^2)\biggr)
[\theta', a'_0, b'_0] & =
\mathrm{D}
\biggl(\frac{1}{2} \sum_{j = 1}^m \bigl(\widehat{a}_j^2
                                      + \widehat{w}_j^2\bigr)\biggr)
\bigl[\widehat{\theta}', \widehat{a}'_0, \widehat{b}'_0\bigr]
\\
\mathrm{D}
\biggl(\frac{1}{2} \sum_{j = 1}^m (a_j^2 + w_j^2 + b_j^2)\biggr)
[\theta', a'_0, b'_0] & =
\mathrm{D}
\biggl(\frac{1}{2} \sum_{j = 1}^m \bigl(\widehat{a}_j^2
                                      + \widehat{w}_j^2
                                      + \widehat{b}_j^2\bigr)\biggr)
\bigl[\widehat{\theta}', \widehat{a}'_0, \widehat{b}'_0\bigr].
\end{align}
\end{enumerate}
\end{lem}

\begin{proof}
Part~\ref{l:rev:f} follows by applying at each $j \in J$ the identity
\[a_j \, \sigma( w_j x + b_j)
- a_j \, \sigma(-w_j x - b_j)
= a_j w_j x + a_j b_j
\quad \text{for all } x \in \mathbb{R}.\]

In part~\ref{l:rev:D}, the equalities of the directional derivatives of the square $\ell_2$-norms (without and with the biases) are straightforward and do not involve the skip connection.

The remaining two equations are symmetric, so we show the former one, at an arbitrary $x \in \mathbb{R}$:
\begin{align}
& \mathrm{D}^\circ
  \bigl(f_{\theta, a_0, b_0}(x)\bigr)
  [\theta', a'_0, b'_0] \\
& = \sum_{j \mid w_j x + b_j = 0}
    \max \bigl\{0, a_j (w'_j x + b'_j)\bigr\} \\
& \quad
  + \sum_{j \mid w_j x + b_j > 0}
    \bigl(a'_j (w_j  x + b_j)
        + a_j  (w'_j x + b'_j)\bigr) \\
& \quad
  + a'_0 x + b'_0 \\
& = \sum_{j \in J \mid \widehat{w}_j x + \widehat{b}_j = 0}
    \Bigl(\max \Bigl\{0, \widehat{a}_j
                         \bigl(-\widehat{w}'_j x - \widehat{b}'_j\bigr)\Bigr\}
                       + \widehat{a}_j
                         \bigl( \widehat{w}'_j x + \widehat{b}'_j\bigr)\Bigr) \\
& \quad
  + \sum_{j \notin J \mid \widehat{w}_j x + \widehat{b}_j = 0}
    \max \Bigl\{0, \widehat{a}_j
                   \bigl(\widehat{w}'_j x + \widehat{b}'_j\bigr)\Bigr\} \\
& \quad
  + \sum_{j \in J \mid \widehat{w}_j x + \widehat{b}_j < 0}
    \Bigl(\widehat{a}'_j \bigl(-\widehat{w}_j  x - \widehat{b}_j\bigr)
        + \widehat{a}_j  \bigl(-\widehat{w}'_j x - \widehat{b}'_j\bigr)
        + \widehat{a}'_j \bigl( \widehat{w}_j  x + \widehat{b}_j\bigr)
        + \widehat{a}_j  \bigl( \widehat{w}'_j x + \widehat{b}'_j\bigr)\Bigr) \\
& \quad
  + \sum_{j \notin J \mid \widehat{w}_j x + \widehat{b}_j > 0}
    \Bigl(\widehat{a}'_j \bigl(\widehat{w}_j  x + \widehat{b}_j\bigr)
        + \widehat{a}_j  \bigl(\widehat{w}'_j x + \widehat{b}'_j\bigr)\Bigr) \\
& \quad
  + \sum_{j \in J \mid \widehat{w}_j x + \widehat{b}_j > 0}
    \Bigl(\widehat{a}'_j \bigl( \widehat{w}_j  x + \widehat{b}_j\bigr)
        + \widehat{a}_j  \bigl( \widehat{w}'_j x + \widehat{b}'_j\bigr)\Bigr) \\
& \quad
  + \widehat{a}'_0 x + \widehat{b}'_0 \\
& = \sum_{j \mid \widehat{w}_j x + \widehat{b}_j = 0}
    \max \Bigl\{0, \widehat{a}_j
                   \bigl(\widehat{w}'_j x + \widehat{b}'_j\bigr)\Bigr\} \\
& \quad
  + \sum_{j \mid \widehat{w}_j x + \widehat{b}_j > 0}
    \Bigl(\widehat{a}'_j \bigl( \widehat{w}_j  x + \widehat{b}_j\bigr)
        + \widehat{a}_j  \bigl( \widehat{w}'_j x + \widehat{b}'_j\bigr)\Bigr) \\
& \quad
  + \widehat{a}'_0 x + \widehat{b}'_0 \\
& = \mathrm{D}^\circ
    \bigl(f_{\widehat{\theta}, \widehat{a}_0, \widehat{b}_0}(x)\bigr)
    \bigl[\widehat{\theta}', \widehat{a}'_0, \widehat{b}'_0\bigr].
\ifopt\else\qedhere\fi
\end{align}
\end{proof}

Several of the perturbations used below move the kink of at most one neuron at any given input. For such directions, the Clarke directional derivative of the network output at an input~$x_0$ can be recovered from the ordinary parameter-directional derivatives evaluated at nearby inputs on the two sides of~$x_0$. The next lemma makes this reduction precise: provided zero hidden neurons are not displaced and no two hidden neurons whose kinks coincide at~$x_0$ are simultaneously moved, the Clarke directional derivative at~$x_0$ is the maximum of the two one-sided limits. This will allow the subsequent proofs to work directly with the affine expressions valid on either side of a kink, rather than repeatedly calculating directly with the Clarke subdifferential at the kink itself.

\begin{lem}
\label{l:up.down}
Suppose $x_0 \in \mathbb{R}$ and $\theta', a'_0, b'_0$ is a direction such that:
\begin{itemize}
\item
for all $j \in [m]$, if $w_j = 0 = b_j$ then $w'_j = 0 = b'_j$;
\item
for all $j, j^\dag \in [m]$ with $j \neq j^\dag$, if $w_j x_0 + b_j = 0 = w_{j^\dag} x_0 + b_{j^\dag}$ then either $w'_j = 0 = b'_j$ or $w'_{j^\dag} = 0 = b'_{j^\dag}$.
\end{itemize}
Then we have:
\begin{align}
& \mathrm{D}^\circ
  \bigl(f_{\theta, a_0, b_0}(x_0)\bigr)
  [\theta', a'_0, b'_0] \\
& = \max \Bigl\{\lim_{x \uparrow x_0}
                \mathrm{D}
                \bigl(f_{\theta, a_0, b_0}(x)\bigr)
                [\theta', a'_0, b'_0],
                \lim_{x \downarrow x_0}
                \mathrm{D}
                \bigl(f_{\theta, a_0, b_0}(x)\bigr)
                [\theta', a'_0, b'_0]\Bigr\} \\
& \mathrm{D}^\circ
  \bigl(-f_{\theta, a_0, b_0}(x_0)\bigr)
  [\theta', a'_0, b'_0] \\
& = \max \Bigl\{\lim_{x \uparrow x_0}
                \mathrm{D}
                \bigl(-f_{\theta, a_0, b_0}(x)\bigr)
                [\theta', a'_0, b'_0],
                \lim_{x \downarrow x_0}
                \mathrm{D}
                \bigl(-f_{\theta, a_0, b_0}(x)\bigr)
                [\theta', a'_0, b'_0]\Bigr\}.
\end{align}
\end{lem}

\begin{proof}
We show the first equation, the second can be proved analogously.  By the second assumption of the lemma, there exists at most one $j \in [m]$ such that $w_j x_0 + b_j = 0$ and either $w'_j \neq 0$ or $b'_j \neq 0$.  We consider the case where such a~$j$ exists, and we denote it~$j^\star$; the other case where such a~$j$ does not exist is simpler.  Without loss of generality, $w_{j^\star} > 0$.  Then we have:
\begin{align}
& \max \Bigl\{\lim_{x \uparrow x_0}
              \mathrm{D}
              \bigl(f_{\theta, a_0, b_0}(x)\bigr)
              [\theta', a'_0, b'_0],
              \lim_{x \downarrow x_0}
              \mathrm{D}
              \bigl(f_{\theta, a_0, b_0}(x)\bigr)
              [\theta', a'_0, b'_0]\Bigr\} \\
& = \max \left\{\begin{gathered}
                \lim_{x \uparrow x_0}
                \sum_{j \mid w_j x + b_j > 0}
                \bigl(a'_j (w_j  x + b_j)
                    + a_j  (w'_j x + b'_j)\bigr), \\
                \lim_{x \downarrow x_0}
                \sum_{j \mid w_j x + b_j > 0}
                \bigl(a'_j (w_j  x + b_j)
                    + a_j  (w'_j x + b'_j)\bigr)
                \end{gathered}\right\} \\
& \quad
  + a'_0 x_0 + b'_0 \\
& = \max \left\{\sum_{j \mid w_j < 0, w_j x_0 + b_j = 0}
                a_j  (w'_j x_0 + b'_j),
                \sum_{j \mid w_j > 0, w_j x_0 + b_j = 0}
                a_j  (w'_j x_0 + b'_j)\right\} \\
& \quad
  + \sum_{j \mid w_j x_0 + b_j > 0}
    \bigl(a'_j (w_j  x_0 + b_j)
        + a_j  (w'_j x_0 + b'_j)\bigr)
  + a'_0 x_0 + b'_0 \\
& = \max \Bigl\{0, a_{j^\star} \bigl(w'_{j^\star} x_0 + b'_{j^\star}\bigr)\Bigr\} \\
& \quad
  + \sum_{j \mid w_j x_0 + b_j > 0}
    \bigl(a'_j (w_j  x_0 + b_j)
        + a_j  (w'_j x_0 + b'_j)\bigr)
  + a'_0 x_0 + b'_0 \\
& = \sum_{j \mid w_j x_0 + b_j = 0}
    \max \bigl\{0, a_j (w'_j x_0 + b'_j)\bigr\} \\
& \quad
  + \sum_{j \mid w_j x_0 + b_j > 0}
    \bigl(a'_j (w_j  x_0 + b_j)
        + a_j  (w'_j x_0 + b'_j)\bigr)
  + a'_0 x_0 + b'_0 \\
& = \mathrm{D}^\circ
    \bigl(f_{\theta, a_0, b_0}(x_0)\bigr)
    [\theta', a'_0, b'_0].
\ifopt\else\qedhere\fi
\end{align}
\end{proof}

We now come to the central part of the work to prove \cref{th:inter}, which consists of a number of lemmas of three kinds: the first kind show that, if a CPA fails some of the properties that make up the geometric characterizations in \cref{th:inter}, then that is due to at least one of typically several possible local geometric irregularities; the second kind show that, for each such irregularity, there exists a direction for infinitesimally perturbing the network's parameters (so that, informally speaking, the irregularity is lessened) such that the signs of the directional derivatives of the network norm and the interpolation constraints mean that the network cannot be a KKT point; and the third kind is similar to the second, except that instead of handling local geometric irregularities, they apply to networks whose parameters have a local internal irregularity (e.g., a failure of balancedness or alignment).  We also remark that the statements of some of these lemmas are tailored not only for contributing to the proof of \cref{th:inter}, but additionally to the proof of \cref{th:min.reg.loss} (which we develop in \cref{s:min.reg.loss}), where all training points contribute to the objective's logistic loss term (in contrast to only those training points at which the interpolation constraints are tight being significant).

First, we identify the local geometric irregularities that capture the failures of convexity correctness for a positive-margin CPA\dots

\begin{lem}
\label{l:cc}
If a CPA $g \colon \mathbb{R} \to \mathbb{R}$ whose kinks are $(p_k)_{k = 1}^q$ has a positive margin, but it is not convexity correct, then at least one of the following holds.
\begin{enumerate}[(i)]
\item
\label{l:cc:opp.conv}
Some kinks $p_k$ and~$p_{k + 1}$ are of opposite convexity and such that either $y_i = -1$ for all $x_i \in [p_k, p_{k + 1}]$ or $y_i = 1$ for all $x_i \in [p_k, p_{k + 1}]$.
\item
\label{l:cc:m.infty}
Either $y_i = -1$ for all $x_i \in (-\infty, p_1]$ or $y_i = 1$ for all $x_i \in (-\infty, p_1]$.
\item
\label{l:cc:p.infty}
Either $y_i = -1$ for all $x_i \in [p_q, \infty)$ or $y_i = 1$ for all $x_i \in [p_q, \infty)$.
\end{enumerate}
\end{lem}

\begin{proof}
Since $g$~is not convexity correct, without loss of generality (i.e., the remaining cases are symmetric):
\begin{itemize}
\item
either some convex kink~$p_k$ and some intermediate segment $\iota(l - 1) + 1, \dots, \iota(l)$ with label~$1$ are such that $p_k \in (x_{\iota(l - 1)}, x_{\iota(l) + 1})$,
\item
or some kink is in $(-\infty, x_{\iota(1) + 1})$.
\end{itemize}

In the former case, since we have $g(x_{\iota(l - 1)}) < 0$ and $g(x_{\iota(l) + 1}) < 0$, but $g(x_i) > 0$ for all $i \in \{\iota(l - 1) + 1, \dots, \iota(l)\}$, there must be a concave kink in $(x_{\iota(l - 1)}, x_{\iota(l) + 1})$, and so there exist kinks $p_{k'}$ and~$p_{k' + 1}$ of opposite convexity in $(x_{\iota(l - 1)}, x_{\iota(l) + 1})$.  Recalling that $y_i = 1$ for all $i \in \{\iota(l - 1) + 1, \dots, \iota(l)\}$, property~\ref{l:cc:opp.conv} holds.

In the latter case, we have $p_1 \in (-\infty, x_{\iota(1) + 1})$, and recalling that labels~$y_i$ are the same for all $i \in \{1, \dots, \iota(1)\}$, property~\ref{l:cc:m.infty} holds.
\end{proof}

\dots and for each of those irregularities, identify a direction along which the bias-free square $\ell_2$-norm of the network parameters has a negative derivative while, for each training point, its negative margin has a nonpositive Clarke derivative.

\begin{lem}
\label{l:all.P1}
If $g = f_{\theta, a_0, b_0}$, whose kinks are $(p_k)_{k = 1}^q$, satisfies at least one of properties \ref{l:cc:opp.conv}, \ref{l:cc:m.infty}, or~\ref{l:cc:p.infty} in \cref{l:cc}, then there exists a direction $\theta', a'_0, b'_0$ such that $\mathrm{D} \bigl(\frac{1}{2} \sum_{j = 1}^m (a_j^2 + w_j^2)\bigr)[\theta', a'_0, b'_0] < 0$ and $\mathrm{D}^\circ \bigl(-y_i \, f_{\theta, a_0, b_0}(x_i)\bigr)[\theta', a'_0, b'_0] \leq 0$ for all $i \in [n]$.
\end{lem}

\begin{proof}
In case property~\ref{l:cc:opp.conv} in \cref{l:cc} holds, without loss of generality, $p_k$~is convex, $p_{k + 1}$~is concave, and $y_i = -1$ for all $x_i \in [p_k, p_{k + 1}]$.  Then $\theta$~must contain some neurons $(a_j, w_j, b_j)$ and $(a_{j'}, w_{j'}, b_{j'})$ such that:
\begin{itemize}
\item
$-b_j / w_j = p_k$ and $a_j > 0$;
\item
$-b_{j'} / w_{j'} = p_{k + 1}$ and $a_{j'} < 0$.
\end{itemize}
By \cref{l:rev}, after reversing the relevant neurons and absorbing the resulting affine terms into the skip connection, we may assume without loss of generality that $w_j > 0$ and $w_{j'} > 0$. Any direction constructed for the resulting parameterization pulls back, by \cref{l:rev}~\ref{l:rev:D}, to a direction with the same relevant Clarke and regularizer derivatives at the original parameterization. Moreover, by \cref{l:up.down}, the claimed nonpositivity of the Clarke directional derivatives of $-y_i \, f_{\theta, a_0, b_0}(x_i)$ is implied by nonpositivity of the corresponding one-sided limits of the ordinary directional derivatives. Therefore, letting $u \coloneqq (a_{j'} w_{j'}, -a_{j'} w_j, w_j b_{j'} - w_{j'} b_j)$, it suffices to show that:
\begin{itemize}
\item
$\mathrm{D}_{a_j, a_{j'}, b_{j'}}
 \bigl(\frac{1}{2}
       \sum_{j'' = 1}^m (a_{j''}^2 + w_{j''}^2)\bigr)
 [u] < 0$;
\item
$\lim_{x \uparrow x_i}
 \mathrm{D}_{a_j, a_{j'}, b_{j'}}
 \bigl(y_i \, f_{\theta, a_0, b_0}(x)\bigr)
 [u] \geq 0$
for all $i \in [n]$;
\item
$\lim_{x \downarrow x_i}
 \mathrm{D}_{a_j, a_{j'}, b_{j'}}
 \bigl(y_i \, f_{\theta, a_0, b_0}(x)\bigr)
 [u] \geq 0$
for all $i \in [n]$.
\end{itemize}
For the first claim, we have
\[\mathrm{D}_{a_j, a_{j'}, b_{j'}}
  \biggl(\frac{1}{2}
         \sum_{j'' = 1}^m (a_{j''}^2 + w_{j''}^2)\biggr)
  [u]
= a_j a_{j'} w_{j'} - a_{j'}^2 w_j
< 0\]
because $a_j > 0$, $a_{j'} < 0$, $w_j > 0$, and $w_{j'} > 0$.
For the second and third claims, since $y_i = -1$ for all $x_i \in [p_k, p_{k + 1}]$, in turn it suffices to observe that:
\begin{itemize}
\item
for all $x \in (-\infty, p_k)$ we have
\begin{align}
& \mathrm{D}_{a_j, a_{j'}, b_{j'}}
  f_{\theta, a_0, b_0}(x)
  [u] \\
& = \mathrm{D}_{a_j, a_{j'}, b_{j'}}
    \bigl(a_j \, \sigma(w_j x + b_j)\bigr)
    [u]
  + \mathrm{D}_{a_j, a_{j'}, b_{j'}}
    \bigl(a_{j'} \, \sigma(w_{j'} x + b_{j'})\bigr)
    [u] \\
& = 0;
\end{align}
\item
for all $x \in (p_k, p_{k + 1})$ we have
\begin{align}
& \mathrm{D}_{a_j, a_{j'}, b_{j'}}
  f_{\theta, a_0, b_0}(x)
  [u] \\
& = \mathrm{D}_{a_j, a_{j'}, b_{j'}}
    \bigl(a_j \, \sigma(w_j x + b_j)\bigr)
    [u]
  + \mathrm{D}_{a_j, a_{j'}, b_{j'}}
    \bigl(a_{j'} \, \sigma(w_{j'} x + b_{j'})\bigr)
    [u] \\
& = (w_j x + b_j) a_{j'} w_{j'} \\
& < 0;
\end{align}
\item
for all $x \in (p_{k + 1}, \infty)$ we have
\begin{align}
& \mathrm{D}_{a_j, a_{j'}, b_{j'}}
  f_{\theta, a_0, b_0}(x)
  [u] \\
& = \mathrm{D}_{a_j, a_{j'}, b_{j'}}
    \bigl(a_j \, \sigma(w_j x + b_j)\bigr)
    [u]
  + \mathrm{D}_{a_j, a_{j'}, b_{j'}}
    \bigl(a_{j'} \, \sigma(w_{j'} x + b_{j'})\bigr)
    [u] \\
& = (w_j x + b_j) a_{j'} w_{j'}
  - (w_{j'} x + b_{j'}) a_{j'} w_j
  + a_{j'} (w_j b_{j'} - w_{j'} b_j) \\
& = 0.
\end{align}
\end{itemize}

The remaining two cases, when property~\ref{l:cc:m.infty} or property~\ref{l:cc:p.infty} in \cref{l:cc} holds, are symmetric so we show how to handle the former, where without loss of generality, $y_i = -1$ for all $x_i \in (-\infty, p_1]$.  Then we have two subcases depending on the convexity of~$p_1$, where we consider the less straightforward one, namely when $p_1$~is concave.  Hence, and by the orientation-reversal reparameterization in \cref{l:rev}, for some neuron $(a_j, w_j, b_j)$ in~$\theta$ we have
\begin{align}
-b_j / w_j & = p_1 &
       a_j & < 0 &
       w_j & < 0.
\end{align}
Now pick $\widehat{p} \leq p_1$ such that $\widehat{p} < x_1$, and let $v \coloneqq (-a_j, w_j \widehat{p} + b_j)$.  By \cref{l:up.down} again for reducing our proof obligations with Clarke directional derivatives to sign properties of one-sided limits of ordinary directional derivatives, it suffices to show that:
\begin{itemize}
\item
$\mathrm{D}_{a_j, b_j}
 \bigl(\frac{1}{2}
       \sum_{j' = 1}^m (a_{j'}^2 + w_{j'}^2)\bigr)
 [v] < 0$;
\item
$\lim_{x \uparrow x_i}
 \mathrm{D}_{a_j, b_j}
 \bigl(y_i \, f_{\theta, a_0, b_0}(x)\bigr)
 [v] \geq 0$
for all $i \in [n]$;
\item
$\lim_{x \downarrow x_i}
 \mathrm{D}_{a_j, b_j}
 \bigl(y_i \, f_{\theta, a_0, b_0}(x)\bigr)
 [v] \geq 0$
for all $i \in [n]$.
\end{itemize}
For the first claim, we have
\[\mathrm{D}_{a_j, b_j}
  \biggl(\frac{1}{2}
         \sum_{j' = 1}^m (a_{j'}^2 + w_{j'}^2)\biggr)
  [v]
= -a_j^2
< 0.\]
For the second and third claims, since $\widehat{p} < x_1$ and $y_i = -1$ for all $x_i \in (-\infty, p_1]$, in turn it suffices to observe that:
\begin{itemize}
\item
for all $x \in (\widehat{p}, p_1)$ we have
\begin{align}
\mathrm{D}_{a_j, b_j}
f_{\theta, a_0, b_0}(x)
[v]
& = \mathrm{D}_{a_j, b_j}
    \bigl(a_j \, \sigma(w_j x + b_j)\bigr)
    [v] \\
& = -a_j (w_j x + b_j) + a_j (w_j \widehat{p} + b_j) \\
& = -a_j w_j (x - \widehat{p}) \\
& < 0;
\end{align}
\item
for all $x \in (p_1, \infty)$ we have
\begin{align}
\mathrm{D}_{a_j, b_j}
f_{\theta, a_0, b_0}(x)
[v]
& = \mathrm{D}_{a_j, b_j}
    \bigl(a_j \, \sigma(w_j x + b_j)\bigr)
    [v] \\
& = 0.
\ifopt\else\qedhere\fi
\end{align}
\end{itemize}
\end{proof}

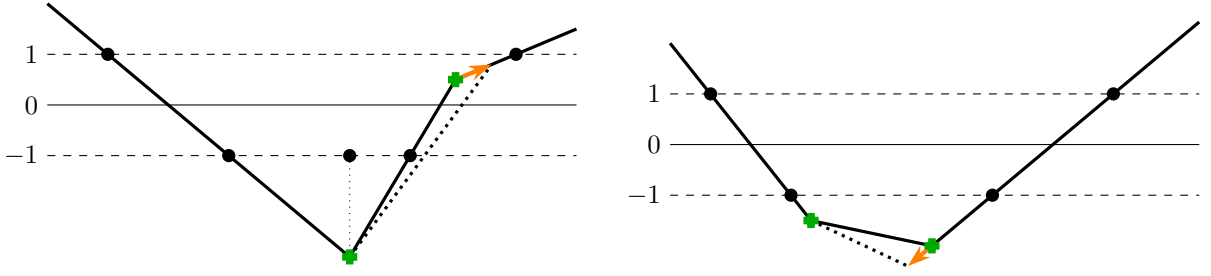
\begin{figure}[t]
\centering
\begin{tikzpicture}[xscale=\ifopt\else\ifarxiv .8\else .67\fi\fi,yscale=.67]
\draw         (-5,  0) node[left]  {$0$} -- (3.75,  0);
\draw[dashed] (-5,  1) node[left]  {$1$} -- (3.75,  1);
\draw[dashed] (-5, -1) node[left] {$-1$} -- (3.75, -1);
\draw[very thick]
(-5   ,  2  ) --
( 0   , -3  ) --
( 1.75,  0.5) --
( 3.75,  1.5);
\node[point] at (-4   ,  1) {};
\node[point] at (-2   , -1) {};
\node[point] at ( 1   , -1) {};
\node[point] at ( 2.75,  1) {};
\draw[thin,dotted] (0, -1) node[point] {} -- (0, -3);
\draw[very thick,dotted] (0, -3) -- (2.35, 0.8);
\draw[->,orange,ultra thick] (1.75, 0.5) -- (2.35, 0.8);
\pic at (0   , -3  ) {plus};
\pic at (1.75,  0.5) {plus};
\end{tikzpicture}
\hspace{1em}
\begin{tikzpicture}[xscale=\ifopt\else\ifarxiv .8\else .67\fi\fi,yscale=.67]
\draw         (-5,  0) node[left]  {$0$} -- (3.75,  0);
\draw[dashed] (-5,  1) node[left]  {$1$} -- (3.75,  1);
\draw[dashed] (-5, -1) node[left] {$-1$} -- (3.75, -1);
\draw[very thick]
(-5   ,  2   ) --
(-2.67, -1.5 ) --
(-0.67, -2   ) --
( 3.75,  2.42);
\node[point] at (-4.33,  1) {};
\node[point] at (-3   , -1) {};
\node[point] at ( 0.33, -1) {};
\node[point] at ( 2.33,  1) {};
\draw[very thick,dotted] (-2.67, -1.5) -- (-1.07, -2.4);
\draw[->,orange,ultra thick] (-0.67, -2) -- (-1.07, -2.4);
\pic at (-2.67, -1.5) {plus};
\pic at (-0.67, -2  ) {plus};
\end{tikzpicture}
\caption{\textbf{Left:}  Here we illustrate the proof of \cref{l:all.P1} for the case of property~\ref{l:cc:opp.conv} in \cref{l:cc}; in this example, one of the kinks of the network function happens to coincide with an input in the dataset.  \textbf{Right:}  Here we illustrate the proof of \cref{l:all.Pb1} for the case of property~\ref{l:st:conv} in \cref{l:st}.  \textbf{All:}  The kinks $p_k$ and~$p_{k + 1}$ in the proof are indicated by the green pluses.  The proof identifies a direction for infinitesimally perturbing the network parameters so that the portion of the network function between $p_k$ and~$p_{k + 1}$ moves as indicated by the orange arrow and the dotted line segment.}
\label{f:cc:opp.conv.st:conv}
\end{figure}

\begin{figure}[t]
\centering
\begin{tikzpicture}[xscale=\ifopt\else\ifarxiv .8\else .67\fi\fi,yscale=.67]
\draw         (-5,  0) node[left]  {$0$} -- (3.75,  0);
\draw[dashed] (-5,  1) node[left]  {$1$} -- (3.75,  1);
\draw[dashed] (-5, -1) node[left] {$-1$} -- (3.75, -1);
\draw[very thick]
(-5   , -3   ) --
(-1.5 ,  0.5 ) --
( 3.75,  1.55);
\node[point] at (-4  , -1) {};
\node[point] at (-3  , -1) {};
\node[point] at ( 1  ,  1) {};
\node[point] at ( 2.5,  1) {};
\node[point] at ( 3  ,  1) {};
\draw[thin,dotted] (-4.25, -1) -- (-4.25, -2.25);
\draw[very thick,dotted] (-5, -2.84) -- (-0.5, 0.7);
\draw[->,orange,ultra thick] (-1.5, 0.5) -- (-0.5, 0.7);
\pic at (-4.25, -2.25) {plus};
\pic at (-1.5 ,  0.5 ) {plus};
\end{tikzpicture}
\hspace{1em}
\begin{tikzpicture}[xscale=\ifopt\else\ifarxiv .8\else .67\fi\fi,yscale=.6]
\draw         (-5,  0) node[left]  {$0$} -- (3.75,  0);
\draw[dashed] (-5,  1) node[left]  {$1$} -- (3.75,  1);
\draw[dashed] (-5, -1) node[left] {$-1$} -- (3.75, -1);
\draw[very thick]
(-5   , -2.6) --
( 1.5 ,  2.6) --
( 3.75,  1.7);
\node[point]            at (-4  , -1) {};
\node[point]            at (-3  , -1) {};
\node[point,color=gray] at ( 1  ,  1) {};
\node[point]            at ( 2.5,  1) {};
\node[point]            at ( 3  ,  1) {};
\draw[thin,dotted] (-4.25, -1) -- (-4.25, -2);
\draw[very thick,dotted] (-5, -2.5) -- (2.25, 2.3);
\draw[->,orange,ultra thick] (1.5, 2.6) -- (2.25, 2.3);
\pic at (-4.25, -2  ) {plus};
\pic at ( 1.5 ,  2.6) {plus};
\end{tikzpicture}
\caption{\textbf{Left:}  Here we illustrate the proofs of \cref{l:all.P1,l:all.Pb1} for the case of property~\ref{l:cc:m.infty} in \cref{l:cc}.  \textbf{Right:}  Here we illustrate the proofs of \cref{l:tight.P1,l:tight.Pb1} for the case of property~\ref{l:sh:m.infty} in \cref{l:sh}; this example also shows, by the dataset point colored gray, how the directional derivative of $-y_i \, f_{\theta, a_0, b_0}(x_i)$ may be positive when the interpolation at~$x_i$ is not tight (i.e., $f_{\theta, a_0, b_0}(x_i) \neq y_i$).  \textbf{All:}  The first kink~$p_1$ and the center of rotation~$\widehat{p}$ in the proofs are indicated by the green pluses.  The proof identifies a direction for infinitesimally perturbing the network parameters so that the portion of the network function before~$p_1$ moves as indicated by the orange arrow and the dotted half line.}
\label{f:cc:m.infty.sh:m.infty}
\end{figure}

Next, we identify the local geometric irregularities that capture the failures of switch hugging for a convexity-correct interpolating CPA when there are at least two label switches\dots

\begin{lem}
\label{l:sh}
Suppose the dataset has at least two label switches, i.e., $r \geq 2$.  If a CPA $g \colon \mathbb{R} \to \mathbb{R}$ whose kinks are $(p_k)_{k = 1}^q$ interpolates the dataset and is convexity correct, but it is not switch hugging, then at least one of the following holds.
\begin{enumerate}[(i)]
\item
\label{l:sh:opp.conv}
Some kinks $p_k$ and~$p_{k + 1}$ are of opposite convexity and such that either $g(x_i) \neq 1$ for all $x_i \in [p_k, p_{k + 1}]$ or $g(x_i) \neq -1$ for all $x_i \in [p_k, p_{k + 1}]$.
\item
\label{l:sh:m.infty}
We have $p_1 \in (-\infty, x_{\iota(2)}]$, and either $g(x_i) \neq 1$ for all $x_i \in (-\infty, p_1]$ or $g(x_i) \neq -1$ for all $x_i \in (-\infty, p_1]$.
\item
\label{l:sh:p.infty}
We have $p_q \in [x_{\iota(r - 1) + 1}, \infty)$, and either $g(x_i) \neq 1$ for all $x_i \in [p_q, \infty)$ or $g(x_i) \neq -1$ for all $x_i \in [p_q, \infty)$.
\end{enumerate}
\end{lem}

\begin{proof}
Without loss of generality, there exists $l \in [r]$ such that $y_{\iota(l)} = 1$ and $g(x_{\iota(l)}) > 1$.

If $l = 1$, then since $r \geq 2$, and $g$~has a positive margin and is convexity correct, we have that $p_1 \in [x_{\iota(1) + 1}, x_{\iota(2)}]$, and so $g(x_i) \neq 1$ for all $x_i \in (-\infty, p_1]$, which means that property~\ref{l:sh:m.infty} holds.

If $l = r$, then since $r \geq 2$, and $g$~has a positive margin and is convexity correct, we have that $p_q \in [x_{\iota(r - 1) + 1}, x_{\iota(r)}]$, and so $g(x_i) \neq 1$ for all $x_i \in [p_q, \infty)$, which means that property~\ref{l:sh:p.infty} holds.

Otherwise $1 < l < r$, so $g$~has a last kink $p_k \leq x_{\iota(l)}$ which is necessarily concave, and the next kink~$p_{k + 1}$ exists and is necessarily convex.  Moreover $g(x_i) \neq 1$ for all $x_i \in [p_k, p_{k + 1}]$, so property~\ref{l:sh:opp.conv} holds.
\end{proof}

\dots and for each of those irregularities, identify a direction along which the bias-free square $\ell_2$-norm of the network parameters has a negative derivative while, for each training point at which interpolation is tight, its negative margin has a nonpositive Clarke derivative.  This lemma and the previous one, in contrast to the preceding two lemmas, are used only in the proof of \cref{th:inter} (i.e., not also in the proof of \cref{th:min.reg.loss}).

\begin{lem}
\label{l:tight.P1}
If $g = f_{\theta, a_0, b_0}$, whose kinks are $(p_k)_{k = 1}^q$, interpolates the dataset and satisfies at least one of properties \ref{l:sh:opp.conv}, \ref{l:sh:m.infty}, or~\ref{l:sh:p.infty} in \cref{l:sh}, then there exists a direction $\theta', a'_0, b'_0$ such that $\mathrm{D} \bigl(\frac{1}{2} \sum_{j = 1}^m (a_j^2 + w_j^2)\bigr)[\theta', a'_0, b'_0] < 0$ and $\mathrm{D}^\circ \bigl(-y_i \, f_{\theta, a_0, b_0}(x_i)\bigr)[\theta', a'_0, b'_0] \leq 0$ for all $i \in [n]$ with $g(x_i) = y_i$.
\end{lem}

\begin{proof}
Properties \ref{l:sh:opp.conv}, \ref{l:sh:m.infty}, and~\ref{l:sh:p.infty} in \cref{l:sh} are identical to, respectively, properties \ref{l:cc:opp.conv}, \ref{l:cc:m.infty}, or~\ref{l:cc:p.infty} in \cref{l:cc}, except that each condition that all training inputs from an interval have a fixed label is weakened to apply only to points at which the corresponding interpolation constraint is tight.  The proof of \cref{l:tight.P1} can therefore be obtained from the proof of \cref{l:all.P1}, by restricting the reasoning to the latter training points.
\end{proof}

Next, we consider convexity-correct CPAs that fail to be single turning.

\begin{lem}
\label{l:st}
If a CPA $g \colon \mathbb{R} \to \mathbb{R}$ whose kinks are $(p_k)_{k = 1}^q$ is convexity correct, but it is not single turning, then at least one of the following holds.
\begin{enumerate}[(i)]
\item
\label{l:st:conv}
Some kinks $p_k$ and~$p_{k + 1}$ are convex, and $y_i = -1$ for all $x_i \in [p_k, p_{k + 1}]$.
\item
\label{l:st:conc}
Some kinks $p_k$ and~$p_{k + 1}$ are concave, and $y_i = 1$ for all $x_i \in [p_k, p_{k + 1}]$.
\end{enumerate}
\end{lem}

\begin{proof}
This is immediate from the definitions of convexity correctness and single turning.
\end{proof}

Before proceeding to handle the two local geometric irregularities identified in \cref{l:st}, we show that, if the network has any neuron which is imbalanced in the sense that the $\ell_2$-norm of its output weight differs from that of its hidden weight and bias, then there exists a direction along which the square $\ell_2$-norm of the network parameters has a negative derivative while, for each training point, its negative margin has a nonpositive Clarke derivative.

\begin{lem}
\label{l:bal.b}
If $\theta$~contains a neuron $(a_j, w_j, b_j)$ such that $a_j^2 \neq w_j^2 + b_j^2$, then there exists a direction $\theta', a'_0, b'_0$ such that $\mathrm{D} \bigl(\frac{1}{2} \sum_{j = 1}^m (a_j^2 + w_j^2 + b_j^2)\bigr)[\theta', a'_0, b'_0] < 0$ and $\mathrm{D}^\circ \bigl(-y_i \, f_{\theta, a_0, b_0}(x_i)\bigr)[\theta', a'_0, b'_0] \leq 0$ for all $i \in [n]$.
\end{lem}

\begin{proof}
We consider the case $a_j^2 > w_j^2 + b_j^2$, and either $w_j \neq 0$ or $b_j \neq 0$; the other cases are analogous or simpler.  Letting $v \coloneqq (-a_j, w_j, b_j)$, we are done by observing that
\[\mathrm{D}_{a_j, w_j, b_j}
  \biggl(\frac{1}{2}
         \sum_{j' = 1}^m (a_{j'}^2 + w_{j'}^2 + b_{j'}^2)\biggr)
  [v]
= -a_j^2 + w_j^2 + b_j^2
< 0,\]
that $w_j x + b_j > 0$ implies
\begin{align}
\mathrm{D}_{a_j, w_j, b_j}
f_{\theta, a_0, b_0}(x)
[v]
& = \mathrm{D}_{a_j, w_j, b_j}
    \bigl(a_j \, \sigma(w_j x + b_j)\bigr)
    [v] \\
& = -a_j (w_j x + b_j) + a_j w_j x + a_j b_j \\
& = 0,
\end{align}
and that $w_j x + b_j < 0$ implies
\begin{align}
\mathrm{D}_{a_j, w_j, b_j}
f_{\theta, a_0, b_0}(x)
[v]
& = \mathrm{D}_{a_j, w_j, b_j}
    \bigl(a_j \, \sigma(w_j x + b_j)\bigr)
    [v] \\
& = -a_j 0 + w_j 0 + b_j 0 \\
& = 0.
\ifopt\else\qedhere\fi
\end{align}
\end{proof}

Now we handle the two irregularities identified in \cref{l:st} which arise from failures of single turning, and also the three irregularities identified in \cref{l:cc} which capture failures of convexity correctness.  In doing so, thanks to the previous lemma, we can assume that all neurons in the network are balanced.

\begin{lem}
\label{l:all.Pb1}
Under \cref{ass:data}, if $g = f_{\theta, a_0, b_0}$, whose kinks are $(p_k)_{k = 1}^q$, satisfies at least one of properties \ref{l:cc:opp.conv}, \ref{l:cc:m.infty}, or~\ref{l:cc:p.infty} in \cref{l:cc}, or \ref{l:st:conv} or~\ref{l:st:conc} in \cref{l:st}, then there exists a direction $\theta', a'_0, b'_0$ such that $\mathrm{D} \bigl(\frac{1}{2} \sum_{j = 1}^m (a_j^2 + w_j^2 + b_j^2)\bigr)[\theta', a'_0, b'_0] < 0$ and $\mathrm{D}^\circ \bigl(-y_i \, f_{\theta, a_0, b_0}(x_i)\bigr)[\theta', a'_0, b'_0] \leq 0$ for all $i \in [n]$.
\end{lem}

\begin{proof}
First we observe that the disjunction of properties \ref{l:cc:opp.conv} in \cref{l:cc}, and \ref{l:st:conv} and~\ref{l:st:conc} in \cref{l:st}, is equivalent to the following.
\begin{enumerate}[(I)]
\item
\label{any.conv}
Some kinks $p_k$ and~$p_{k + 1}$ are such that either at least one of $p_k, p_{k + 1}$ is convex and $y_i = -1$ for all $x_i \in [p_k, p_{k + 1}]$, or at least one of $p_k, p_{k + 1}$ is concave and $y_i = 1$ for all $x_i \in [p_k, p_{k + 1}]$.
\end{enumerate}

Now the reasoning follows the same pattern as the proof of \cref{l:all.P1}, so we focus on the modifications and additions, which are to handle the replacement of property~\ref{l:cc:opp.conv} in \cref{l:cc} by the weaker property~\ref{any.conv} above (i.e., now only one convexity combination is excluded depending on the labels) and the enlargement of the cost by including the biases.

In case property~\ref{any.conv} above holds, without loss of generality, $p_k$~is convex and $y_i = -1$ for all $x_i \in [p_k, p_{k + 1}]$.  Then $\theta$~must contain some neurons $(a_j, w_j, b_j)$ and $(a_{j'}, w_{j'}, b_{j'})$ such that:
\begin{itemize}
\item
$-b_j / w_j = p_k$ and $a_j > 0$;
\item
$-b_{j'} / w_{j'} = p_{k + 1}$ and $a_{j'} \neq 0$.
\end{itemize}
By \cref{l:bal.b} we may restrict attention to balanced neurons (in the sense that includes the biases), and by \cref{l:rev} we may orient the relevant neurons so that their hidden weights are positive, without changing the represented function or any of the directional-derivative inequalities to be established. Thus for the following construction we may assume that:
\begin{itemize}
\item
$w_j > 0$ and $a_j^2 = w_j^2 + b_j^2$;
\item
$w_{j'} > 0$ and $a_{j'}^2 = w_{j'}^2 + b_{j'}^2$.
\end{itemize}
Let $u \coloneqq \bigl(-w_{j'}, w_j, (w_{j'} b_j - w_j b_{j'}) / a_{j'}\bigr)$.  Here vector~$u$ has the same direction as in the proof of \cref{l:all.P1}.  Its different scaling helps to handle the unknown sign of~$a_{j'}$ in the following derivation of the negativity of the directional derivative of the cost along~$u$:
\begin{align}
& \mathrm{D}_{a_j, a_{j'}, b_{j'}}
  \biggl(\frac{1}{2}
         \sum_{j'' = 1}^m (a_{j''}^2 + w_{j''}^2 + b_{j''}^2)\biggr)
  [u] \\
& =    -a_j w_{j'} + a_{j'} w_j + \frac{w_{j'} b_j b_{j'} - w_j b_{j'}^2}{a_{j'}} \\
& =    -a_j w_{j'} + \frac{b_j b_{j'} w_{j'} + w_j w_{j'}^2}{a_{j'}} \\
& =    -\frac{w_{j'}}{a_{j'}} (a_j a_{j'} - w_j w_{j'} - b_j b_{j'}) \\
& =    -\frac{w_{j'}}{\lvert a_{j'} \rvert}
       \bigl(a_j \lvert a_{j'} \rvert - \sgn(a_{j'}) (w_j w_{j'} + b_j b_{j'})\bigr) \\
& =    -\frac{w_{j'}}{\lvert a_{j'} \rvert}
       \Bigl(\sqrt{w_j^2 + b_j^2} \sqrt{w_{j'}^2 + b_{j'}^2}
           - \sgn(a_{j'}) (w_j w_{j'} + b_j b_{j'})\Bigr) \\
& < 0,
\end{align}
where the Cauchy-Bunyakovsky-Schwarz inequality is strict thanks to $p_k \neq p_{k + 1}$.

The remaining two cases, when property~\ref{l:cc:m.infty} or property~\ref{l:cc:p.infty} in \cref{l:cc} holds, are symmetric so we consider the former, where without loss of generality, $y_i = -1$ for all $x_i \in (-\infty, p_1]$.  Then we have two subcases depending on the convexity of~$p_1$, where again we consider the less straightforward one, namely when $p_1$~is concave.  Hence, and by \cref{l:rev,l:bal.b}, for some neuron $(a_j, w_j, b_j)$ in~$\theta$ we have
\begin{align}
-b_j / w_j & = p_1 &
       a_j & < 0 &
       w_j & < 0 &
     a_j^2 & = w_j^2 + b_j^2.
\end{align}
As before, pick $\widehat{p} \leq p_1$ with $\widehat{p} < x_1$, and let $v \coloneqq (-a_j, w_j \widehat{p} + b_j)$.  It remains to observe that
\[\mathrm{D}_{a_j, b_j}
  \biggl(\frac{1}{2}
         \sum_{j' = 1}^m (a_{j'}^2 + w_{j'}^2 + b_{j'}^2)\biggr)
  [v]
= -a_j^2 + w_j b_j \widehat{p} + b_j^2
= -w_j^2 (p_1 \widehat{p} + 1)
< 0,\]
where the inequality follows from $\widehat{p} \leq p_1$, since $p_1 < 0$ by \cref{ass:data}.
\end{proof}

Like \cref{l:tight.P1}, the next lemma handles the three irregularities identified in \cref{l:sh} which arise from failures of switch hugging, the difference being that here the biases are included in the square $\ell_2$-norm of the network parameters.  Also like \cref{l:tight.P1}, this result is used only in the proof of \cref{th:inter}.

\begin{lem}
\label{l:tight.Pb1}
Under \cref{ass:data}, if $g = f_{\theta, a_0, b_0}$, whose kinks are $(p_k)_{k = 1}^q$, interpolates the dataset and satisfies at least one of properties \ref{l:sh:opp.conv}, \ref{l:sh:m.infty}, or~\ref{l:sh:p.infty} in \cref{l:sh}, then there exists a direction $\theta', a'_0, b'_0$ such that $\mathrm{D} \bigl(\frac{1}{2} \sum_{j = 1}^m (a_j^2 + w_j^2 + b_j^2)\bigr)[\theta', a'_0, b'_0] < 0$ and $\mathrm{D}^\circ \bigl(-y_i \, f_{\theta, a_0, b_0}(x_i)\bigr)[\theta', a'_0, b'_0] \leq 0$ for all $i \in [n]$ with $g(x_i) = y_i$.
\end{lem}

\begin{proof}
As with the proofs of \cref{l:all.P1,l:tight.P1}, the proof of \cref{l:tight.Pb1} proceeds analogously to the proof of \cref{l:all.Pb1}, with a further simplification due to property~\ref{any.conv} in the latter proof not requiring the two kinks to have opposite convexity whereas property~\ref{l:sh:opp.conv} in \cref{l:sh} requires it.
\end{proof}

The following three lemmas handle further local internal irregularities in the network: first, where some neuron is imbalanced in the sense that the absolute values of its output and hidden weights differ (this mode of balancedness, which is in contrast to that in \cref{l:bal.b}, is appropriate when the biases are not penalized)\dots

\begin{lem}
\label{l:bal}
If $\theta$~contains a neuron $(a_j, w_j, b_j)$ such that $\lvert a_j \rvert \neq \lvert w_j \rvert$, then there exists a direction $\theta', a'_0, b'_0$ such that $\mathrm{D} \bigl(\frac{1}{2} \sum_{j = 1}^m (a_j^2 + w_j^2)\bigr)[\theta', a'_0, b'_0] < 0$ and $\mathrm{D}^\circ \bigl(-y_i \, f_{\theta, a_0, b_0}(x_i)\bigr)[\theta', a'_0, b'_0] \leq 0$ for all $i \in [n]$.
\end{lem}

\begin{proof}
This is similar to and simpler than the proof of \cref{l:bal.b}.  If $\lvert a_j \rvert > \lvert w_j \rvert$ (resp., $\lvert a_j \rvert < \lvert w_j \rvert$), it suffices to consider directional derivatives at $(a_j, w_j, b_j)$ along $(-a_j, w_j, b_j)$ (resp., $(a_j, -w_j, -b_j)$).
\end{proof}

\dots second, where some two neurons are aligned in the sense that their kinks are the same, but the signs of their output weights are conflicting\dots

\begin{lem}
\label{l:opp}
If $\theta$~contains neurons $(a_j, w_j, b_j)$ and $(a_{j'}, w_{j'}, b_{j'})$ such that all of $a_j$, $w_j$, $a_{j'}$, and~$w_{j'}$ are nonzero, $-b_j / w_j = -b_{j'} / w_{j'}$, and $\sgn a_j \neq \sgn a_{j'}$, then there exists a direction $\theta', a'_0, b'_0$ such that $\mathrm{D} \bigl(\frac{1}{2} \sum_{j = 1}^m (a_j^2 + w_j^2)\bigr)[\theta', a'_0, b'_0] < 0$, $\mathrm{D} \bigl(\frac{1}{2} \sum_{j = 1}^m (a_j^2 + w_j^2 + b_j^2)\bigr)[\theta', a'_0, b'_0] < 0$, and $\mathrm{D}^\circ \bigl(-y_i \, f_{\theta, a_0, b_0}(x_i)\bigr)[\theta', a'_0, b'_0] \leq 0$ for all $i \in [n]$.
\end{lem}

\begin{proof}
Without loss of generality, we have $a_j > 0$ and $a_{j'} < 0$.  Moreover, by the orientation-reversal reparameterization in \cref{l:rev}, we may assume $w_j > 0$ and $w_{j'} > 0$.  Letting $v \coloneqq (-w_{j'}, w_j)$ and verifying that \cref{l:up.down} applies because we are not perturbing any hidden-layer parameters, we are done by observing that
\begin{align}
& \mathrm{D}_{a_j, a_{j'}}
  \biggl(\frac{1}{2}
         \sum_{j'' = 1}^m (a_{j''}^2 + w_{j''}^2 + b_{j''}^2)\biggr)
  [v] \\
& = \mathrm{D}_{a_j, a_{j'}}
    \biggl(\frac{1}{2}
           \sum_{j'' = 1}^m (a_{j''}^2 + w_{j''}^2)\biggr)
    [v] \\
& = -a_j w_{j'} + a_{j'} w_j \\
& < 0,
\end{align}
that for all $x < -b_j / w_j$ we have
\begin{align}
\mathrm{D}_{a_j, a_{j'}}
f_{\theta, a_0, b_0}(x)
[v]
& = \mathrm{D}_{a_j, a_{j'}}
    \bigl(a_j    \, \sigma(w_j    x + b_j)
        + a_{j'} \, \sigma(w_{j'} x + b_{j'})\bigr)
    [v] \\
& = -w_{j'} \, \sigma(w_j x + b_j) + w_j \, \sigma(w_{j'} x + b_{j'}) \\
& = -w_{j'} 0 + w_j 0 \\
& = 0,
\end{align}
and that for all $x > -b_j / w_j$ we have
\begin{align}
\mathrm{D}_{a_j, a_{j'}}
f_{\theta, a_0, b_0}(x)
[v]
& = \mathrm{D}_{a_j, a_{j'}}
    \bigl(a_j    \, \sigma(w_j    x + b_j)
        + a_{j'} \, \sigma(w_{j'} x + b_{j'})\bigr)
    [v] \\
& = -w_{j'} \, \sigma(w_j x + b_j) + w_j \, \sigma(w_{j'} x + b_{j'}) \\
& = -w_{j'} (w_j x + b_j) + w_j (w_{j'} x + b_{j'}) \\
& = -w_{j'} b_j + w_j b_{j'} \\
& = 0.
\ifopt\else\qedhere\fi
\end{align}
\end{proof}

\dots and third, where the hidden weight of some neuron is zero (this is relevant when the biases are penalized, in which case by \cref{l:bal.b}, we can assume that the neuron's output weight and hidden bias have equal absolute values).

\begin{lem}
\label{l:bias}
If $\theta$~contains $(a_j, w_j, b_j)$ such that $\lvert a_j \rvert = \lvert b_j \rvert > 0$ and $w_j = 0$, then there exists $\theta', a'_0, b'_0$ such that $\mathrm{D} \bigl(\frac{1}{2} \sum_{j = 1}^m (a_j^2 + w_j^2 + b_j^2)\bigr)[\theta', a'_0, b'_0] < 0$ and $\mathrm{D}^\circ \bigl(-y_i \, f_{\theta, a_0, b_0}(x_i)\bigr)[\theta', a'_0, b'_0] \leq 0$ for all $i \in [n]$.
\end{lem}

\begin{proof}
Again this is similar to the proof of \cref{l:bal.b}.  It suffices to consider directional derivatives at $(a_j, b_j, b_0)$ along $(-\sgn a_j, -\sgn b_j, 2 \sgn(a_j) \sigma(b_j))$.
\end{proof}

The final lemma in the sequence shows that MFCQ is satisfied everywhere in the feasible sets of the four interpolator norm minimization problems we consider.  The proof is based on the familiar phenomenon that increasing the magnitudes of an interpolating network's parameters increases the margin, however some care is required here due to the failure of homogeneity when the skip connection is present.

\begin{lem}
\label{l:feas.MFCQ}
For each of \cref{eq:P,,eq:P1,,eq:Pb,,eq:Pb1}, every feasible point satisfies MFCQ.
\end{lem}

\begin{proof}
We show the claim for \cref{eq:Pb1}, the cases of \cref{eq:P,,eq:P1,,eq:Pb} can be handled analogously.

Suppose $\theta, a_0, b_0$ is feasible.  For all $i \in [n]$ (we do not need to consider only the margin) and all $\gamma \in \partial_\theta f_{\theta, a_0, b_0}(x_i)$, we have
\begin{align}
& \Bigl\langle \Bigl(\frac{1}{2} \theta, a_0, b_0 \Bigr),
               y_i \bigl(\gamma,
                         \nabla_{a_0} f_{\theta, a_0, b_0}(x_i),
                         \nabla_{b_0} f_{\theta, a_0, b_0}(x_i)\bigr) \Bigr\rangle \\
& = \Bigl\langle \Bigl(\frac{1}{2} \theta, a_0, b_0 \Bigr),
                 y_i (\gamma, x_i, 1) \Bigr\rangle \\
& = y_i \Bigl(\frac{1}{2} \langle \theta, \gamma \rangle
            + a_0 x_i + b_0 \Bigr) \\
& = y_i \bigl(f_\theta(x_i)
            + a_0 x_i + b_0 \bigr) \\
& = y_i \, f_{\theta, a_0, b_0}(x_i) \\
& \geq 1
\end{align}
as required, where the second last equality is by Euler's homogeneous function theorem.
\end{proof}

\begin{lem}
\label{l:unique.st}
For each $k \in [r]$, let $h_k \colon \mathbb{R} \to \mathbb{R}$ be the affine function whose graph passes through $\bigl(x_{\iota(k)}, y_{\iota(k)}\bigr)$ and $\bigl(x_{\iota(k) + 1}, y_{\iota(k) + 1}\bigr)$. There exists a unique CPA $g \colon \mathbb{R} \to \mathbb{R}$ that is switch hugging, convexity correct, and single turning. This~$g$ has exactly $r - 1$ kinks: for every $k \in [r - 1]$, its unique kink in the $k$-th intermediate same-label segment is the unique intersection of $h_k$ and~$h_{k + 1}$.
\end{lem}

\begin{proof}
For every $k \in [r - 1]$, since labels $y_{\iota(k)}$ and~$y_{\iota(k) + 1}$ incident to the $k$-th label switch are opposite, labels $y_{\iota(k) + 1}$ and~$y_{\iota(k + 1)}$ of the $k$-th intermediate same-label segment are the same, and labels $y_{\iota(k + 1)}$ and~$y_{\iota(k + 1) + 1}$ incident to the $(k + 1)$-th label switch are opposite, we have that the affine functions $h_k$ and~$h_{k + 1}$ have a unique intersection; we denote its $x$-coordinate as~$p_k$.

Letting
\[g(x) \coloneqq
  \begin{cases}
  h_1(x) & \text{if } x \leq p_1, \\
  h_k(x) & \text{if } p_{k - 1} \leq x \leq p_k, \\
  h_r(x) & \text{if } p_{r - 1} \leq x,
  \end{cases}\]
it is straightforward to verify that $g$~is switch hugging, convexity correct, and single turning.

Conversely, switch hugging and convexity correctness fix the crossing pieces of any such CPA to be $h_1, \dots, h_r$, and continuity forces its unique kink in the $k$-th intermediate same-label segment to be the intersection~$p_k$. Therefore $g$~is unique.
\end{proof}

\begin{proof}%
\ifopt
\else
[Proof of \cref{th:inter}]
\fi
For part~\ref{th:inter:P.P1}, suppose $g = f_{\theta, a_0, b_0}$ is a minimizer of \cref{eq:P1} in function space.  By \cref{l:feas.MFCQ,th:loc.min.MFCQ.KKT}, $\theta, a_0, b_0$ is a KKT point, so by \cref{pr:dd.not.KKT} and \cref{l:cc,,l:all.P1,,l:sh,,l:tight.P1}, $g$~is switch hugging and convexity correct.

Conversely, suppose $g$~is switch hugging and convexity correct, but not a minimizer of \cref{eq:P1} in function space.  By \cref{l:attain.int}, for some $m \in \mathbb{N}$ and some $\theta', a'_0, b'_0$ which is a minimizer of \cref{eq:P1} restricted to network width~$m$, we have that $\mathrm{R_1}(f_{\theta', a'_0, b'_0}) < \mathrm{R_1}(g)$.  Reasoning like above, $f_{\theta', a'_0, b'_0}$~is switch hugging and convexity correct.  But then by \cref{th:repr} we have
\[\mathrm{R_1}(f_{\theta', a'_0, b'_0})
= \sum_{k = 1}^{r - 1}
  \biggl(\frac{2}{x_{\iota(k)     + 1} - x_{\iota(k)}}
       + \frac{2}{x_{\iota(k + 1) + 1} - x_{\iota(k + 1)}}\biggr)
= \mathrm{R_1}(g),\]
which is a contradiction.

The remainder of part~\ref{th:inter:P.P1}, which is for \cref{eq:P}, now follows by \cref{l:RR1.RbRb1}.

For part~\ref{th:inter:Pb.Pb1}, the reasoning follows the same pattern, with \cref{l:st} added for the single turning property and with \cref{l:all.P1,l:tight.P1} replaced by \cref{l:all.Pb1,l:tight.Pb1}. \cref{l:unique.st} gives a unique switch-hugging, convexity-correct, and single-turning interpolant, with exactly one kink in each intermediate same-label segment. The preceding necessity argument shows that every minimizer has these properties, while the same representor-norm comparison as above shows that this canonical interpolant is a minimizer. It is therefore the unique minimizer in function space and has exactly $r - 1$ kinks.

For part~\ref{th:inter:KKT}, suppose $(\theta, a_0, b_0) \in \mathbb{R}^{3 m + 2}$ is a KKT point of \cref{eq:P1}.  Then by \cref{pr:dd.not.KKT} and \cref{l:cc,,l:all.P1,,l:sh,,l:tight.P1} we have that $f_{\theta, a_0, b_0}$~is switch hugging and convexity correct, and moreover by \cref{th:repr,l:bal,l:opp} we have that $\frac{1}{2} \sum_{j = 1}^m (a_j^2 + w_j^2) = \mathrm{R_1}(f_{\theta, a_0, b_0})$, so by part~\ref{th:inter:P.P1} we conclude that $\theta, a_0, b_0$ is a minimizer of \cref{eq:P1}.

Part~\ref{th:inter:KKT} for \cref{eq:Pb1} follows analogously, with \cref{l:st} added for the single turning property, with \cref{l:all.P1,l:tight.P1} replaced by \cref{l:all.Pb1,l:tight.Pb1}, and with \cref{l:bal} replaced by \cref{l:bal.b,l:bias}.
\end{proof}

\clearpage
\section{Counterexamples without the skip connection}

In this section, we show that the skip connection is essential for the benign properties of the optimization landscapes in \cref{th:inter}~\ref{th:inter:KKT} to hold.  Namely, in the next theorem, we demonstrate that, if the skip connection is absent, then there exist KKT points that are not (global) minimizers, regardless of whether the biases are penalized.

In addition, for the unpenalized-bias example, \Cref{f:land.norm} displays a two-dimensional landscape profile with multiple valleys and several suboptimal KKT plateaus, together with the corresponding profile after adding a skip connection. When biases are penalized, the explicit KKT examples are isolated rather than forming the same plateaus; \cref{rem:KKT.not} records the additional observation that the identified points are strict local minimizers.

\begin{thm}
\label{th:KKT.not}
\begin{enumerate}[(i)]
\item
\label{th:KKT.not:P}
For the dataset
\[(-2, 1), (-1, -1),
  ( 1, 1), ( 2, -1),\]
there exists a two-layer ReLU network~$\theta$ of width~$4$ and without a skip connection, which is a KKT point but not a minimizer of \cref{eq:P}, and for which the interpolant~$f_\theta$ is not convexity correct.
\item
\label{th:KKT.not:Pb}
For the dataset
\[(-10, -1), (-6,  1), (-5, -1), (-4,  1),
  (  4,  1), ( 5, -1), ( 6,  1), (10, -1),\]
there exists a two-layer ReLU network~$\theta$ of width~$6$ and without a skip connection, which is a KKT point but not a minimizer of \cref{eq:Pb}, and for which the interpolant~$f_\theta$ is not single turning and not sparsest.
\end{enumerate}
\end{thm}

\begin{figure}[tp]
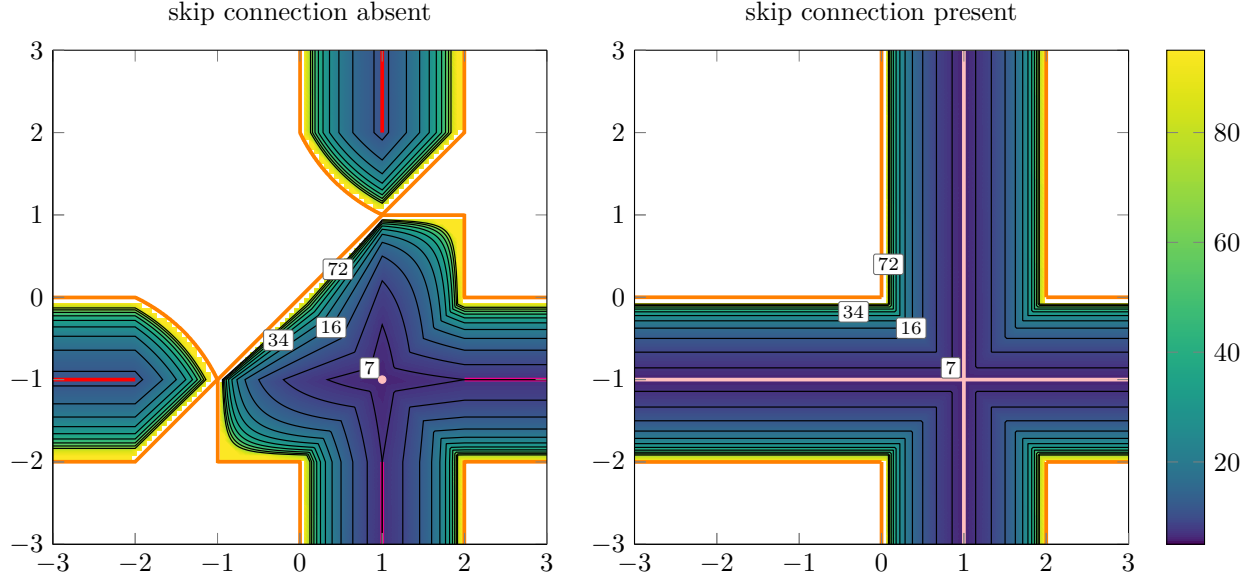

\centering
{\tikzexternalenable
\tikzsetnextfilename{cache_norm_a0_pq}
\tikzpicturedependsonfile{plot_dat_tex/obj_norm_a0_pq.dat}
\begin{tikzpicture}
\begin{axis}
[width=\ifopt\else\ifarxiv .475\else .5\fi\fi\textwidth,
 height=\ifopt\else\ifarxiv .475\else .5\fi\fi\textwidth,
 xmin=-3, xmax=3,
 ymin=-3, ymax=3,
 view={0}{90},
 colormap name=warpedviridis,
 point meta min=5,
 point meta max=95,
 unbounded coords=jump,
 title={skip connection absent}]
\addplot3
[surf,
 shader=interp,
 mesh/cols=101]
table
[x=p,
 y=q,
 z=value_capped]
{plot_dat_tex/obj_norm_a0_pq.dat};
\addplot3[red, line width=1.4pt, no marks] coordinates {(-3,-1,11) (-2,-1,11)};
\addplot3[red, line width=1.4pt, no marks] coordinates {( 1, 3,11) ( 1, 2,11)};
\addplot3[magenta, line width=1.4pt, no marks] coordinates {( 1,-3, 7) ( 1,-2, 7)};
\addplot3[magenta, line width=1.4pt, no marks] coordinates {( 3,-1, 7) ( 2,-1, 7)};
\addplot3[pink, only marks, mark=*, mark size=1.4pt] coordinates {( 1,-1, 6)};
\input{plot_dat_tex/boundary_a0_pq}
\input{plot_dat_tex/cont_norm_a0_pq}
\end{axis}
\end{tikzpicture}
\tikzsetnextfilename{cache_norm_opt_a_pq}
\tikzpicturedependsonfile{plot_dat_tex/obj_norm_opt_a_pq.dat}
\begin{tikzpicture}
\begin{axis}
[width=\ifopt\else\ifarxiv .475\else .5\fi\fi\textwidth,
 height=\ifopt\else\ifarxiv .475\else .5\fi\fi\textwidth,
 xmin=-3, xmax=3,
 ymin=-3, ymax=3,
 view={0}{90},
 colormap name=warpedviridis,
 colorbar,
 point meta min=5,
 point meta max=95,
 unbounded coords=jump,
 title={skip connection present}]
\addplot3
[surf,
 shader=interp,
 mesh/cols=101]
table
[x=p,
 y=q,
 z=value_capped]
{plot_dat_tex/obj_norm_opt_a_pq.dat};
\addplot3[pink, line width=1.4pt, no marks] coordinates {(-3,-1, 6) ( 3,-1, 6)};
\addplot3[pink, line width=1.4pt, no marks] coordinates {( 1,-3, 6) ( 1, 3, 6)};
\input{plot_dat_tex/boundary_opt_a_pq}
\input{plot_dat_tex/cont_norm_opt_a_pq}
\end{axis}
\end{tikzpicture}}
\caption{\ifopt\else (This is \Cref{f:land.norm:main}, reproduced here for the reader's convenience, and with a more detailed caption.)  \fi\textbf{Left:}  For the dataset in \cref{th:KKT.not}~\ref{th:KKT.not:P}, here we depict the norm landscape of interpolating two-layer ReLU networks of width~$4$ and without a skip connection, i.e., the objective landscape of \cref{eq:P} with $m = 4$.  The symmetry of the dataset allows us to obtain an informative two-dimensional plot as follows.  The network consists of: a pair of balanced neurons $(a_\dag, w_\dag, b_\dag), (a'_\dag, w'_\dag, b'_\dag)$ of opposite signs ($a_\dag < 0$ and $a'_\dag > 0$), same slopes ($a_\dag w_\dag = t_\dag = a'_\dag w'_\dag$), and symmetric kinks ($-b_\dag / w_\dag = p_\dag$ and $-b'_\dag / w'_\dag = -p_\dag$); and another pair of balanced neurons $(a_\ddag, w_\ddag, b_\ddag), (a'_\ddag, w'_\ddag, b'_\ddag)$ of opposite signs ($a_\ddag > 0$ and $a'_\ddag < 0$), same slopes ($a_\ddag w_\ddag = t_\ddag = a'_\ddag w'_\ddag$), and symmetric kinks ($-b_\ddag / w_\ddag = p_\ddag$ and $-b'_\ddag / w'_\ddag = -p_\ddag$).  The horizontal and vertical axes of the plot are parameterized by the kink positions $p_\dag$~and~$p_\ddag$ respectively, and the slope values $t_\dag$~and~$t_\ddag$ are optimized.  The boundary of the feasible set is depicted in orange.  Objective values, which diverge to infinity on approaches to the feasibility boundary, are capped at~$95$ in this plot.  For greater visibility, the contour lines are placed at values $6 + k (k + 1) / 2$ for $k \in [12]$, and the color map for the objective is also quadratic.  Colored red are two plateaus with value~$11$ (at $p_\dag \leq -2, p_\ddag = -1$ and at $p_\dag = 1, p_\ddag \geq 2$), colored magenta are two plateaus with value~$7$ (at $p_\dag = 1, p_\ddag \leq -2$ and at $p_\dag \geq 2, p_\ddag = -1$), and colored pink is the minimizer whose value is~$6$ (at $p_\dag = 1, p_\ddag = -1$).  The red plateaus depicted here are symmetric one-dimensional slices of two-dimensional plateaus as in \Cref{f:KKT.not:P}~top, and the magenta plateaus depicted here are symmetric one-dimensional slices of two-dimensional plateaus as in \Cref{f:KKT.not:P}~middle.  \textbf{Right:}  This is the corresponding symmetric profile when a skip connection is added, i.e., here we depict the objective landscape of \cref{eq:P1} with $m = 4$.  The symmetry of the dataset allows us to consider only a skip connection $a_0 x$ with zero bias, and its weight~$a_0$ is optimized.  The feasible set becomes connected, and the suboptimal plateaus disappear.  Colored pink is the plateau of minimizers whose value is~$6$ (at $p_\dag = 1$ and at $p_\ddag = -1$).}
\label{f:land.norm}
\end{figure}

\begin{figure}[t]
\centering
\begin{tikzpicture}[xscale=\ifopt\else\ifarxiv 1.85\else 1.56\fi\fi,yscale=.5]
\draw         (-4,  0) node[left]  {$0$} -- (4,  0);
\draw[dashed] (-4,  1) node[left]  {$1$} -- (4,  1);
\draw[dashed] (-4, -1) node[left] {$-1$} -- (4, -1);
\draw[red,very thick]
(-4   ,  1.89) --
(-2.67,  2.33);
\draw[very thick]
(-2.67,  2.33) --
(-1   , -1   ) --
( 1   ,  1   ) --
( 2.33, -1.67);
\draw[red,very thick]
( 2.33, -1.67) --
( 4   , -0.56);
\draw[dotted] (-2,  1) node[point] {} -- (-2, 0) node[below] {$x_1$};
\draw[dotted] (-1, -1) node[point] {} -- (-1, 0) node[above] {$x_2$};
\draw[dotted] ( 1,  1) node[point] {} -- ( 1, 0) node[below] {$x_3$};
\draw[dotted] ( 2, -1) node[point] {} -- ( 2, 0) node[above] {$x_4$};
\end{tikzpicture}
\\[1ex]
\begin{tikzpicture}[xscale=\ifopt\else\ifarxiv 1.85\else 1.56\fi\fi,yscale=.5]
\draw         (-4,  0) node[left]  {$0$} -- (4,  0);
\draw[dashed] (-4,  1) node[left]  {$1$} -- (4,  1);
\draw[dashed] (-4, -1) node[left] {$-1$} -- (4, -1);
\draw[magenta,very thick]
(-4  ,  5.88) --
(-2.5,  2   );
\draw[very thick]
(-2.5,  2   ) --
(-1  , -1   ) --
( 1  ,  1   ) --
( 3.5, -4   );
\draw[magenta,very thick]
( 3.5, -4   ) --
( 4  , -5.21);
\draw[dotted] (-2,  1) node[point] {} -- (-2, 0) node[below] {$x_1$};
\draw[dotted] (-1, -1) node[point] {} -- (-1, 0) node[above] {$x_2$};
\draw[dotted] ( 1,  1) node[point] {} -- ( 1, 0) node[below] {$x_3$};
\draw[dotted] ( 2, -1) node[point] {} -- ( 2, 0) node[above] {$x_4$};
\end{tikzpicture}
\\[-10ex]
\begin{tikzpicture}[xscale=\ifopt\else\ifarxiv 1.85\else 1.56\fi\fi,yscale=.5]
\draw         (-4,  0) node[left]  {$0$} -- (4,  0);
\draw[dashed] (-4,  1) node[left]  {$1$} -- (4,  1);
\draw[dashed] (-4, -1) node[left] {$-1$} -- (4, -1);
\draw[very thick]
(-4,  5) --
(-1, -1) --
( 1,  1) --
( 4, -5);
\draw[dotted] (-2,  1) node[point] {} -- (-2, 0) node[below] {$x_1$};
\draw[dotted] (-1, -1) node[point] {} -- (-1, 0) node[above] {$x_2$};
\draw[dotted] ( 1,  1) node[point] {} -- ( 1, 0) node[below] {$x_3$};
\draw[dotted] ( 2, -1) node[point] {} -- ( 2, 0) node[above] {$x_4$};
\end{tikzpicture}
\caption{\textbf{Top:}  Here we plot in function space a KKT point of \cref{eq:P} constructed in the proof of \cref{th:KKT.not}~\ref{th:KKT.not:P}, which is not a minimizer.  The red coloring indicates where this interpolant of the depicted dataset fails to be convexity correct, which by \cref{th:inter}~\ref{th:inter:P.P1} means that this example is suboptimal already in function space.  This is one from a two-dimensional plateau of KKT points of cost~$11$, which can be obtained by placing the first kink arbitrarily in the interval $(-\infty, x_1]$ and placing the last kink arbitrarily in the interval $[x_4, \infty)$ (those choices then determine the initial and final slopes of the KKT point in function space).  \textbf{Middle:}  Here we plot another KKT point of \cref{eq:P} which is not a minimizer, and where the failures of convexity correctness of this interpolant are indicated in magenta.  This KKT point is from a different two-dimensional plateau in which the cost is~$7$.  \textbf{Bottom:}  Here we plot the unique in function space minimizer of \cref{eq:P} for the depicted dataset, whose cost is~$6$.}
\label{f:KKT.not:P}
\end{figure}
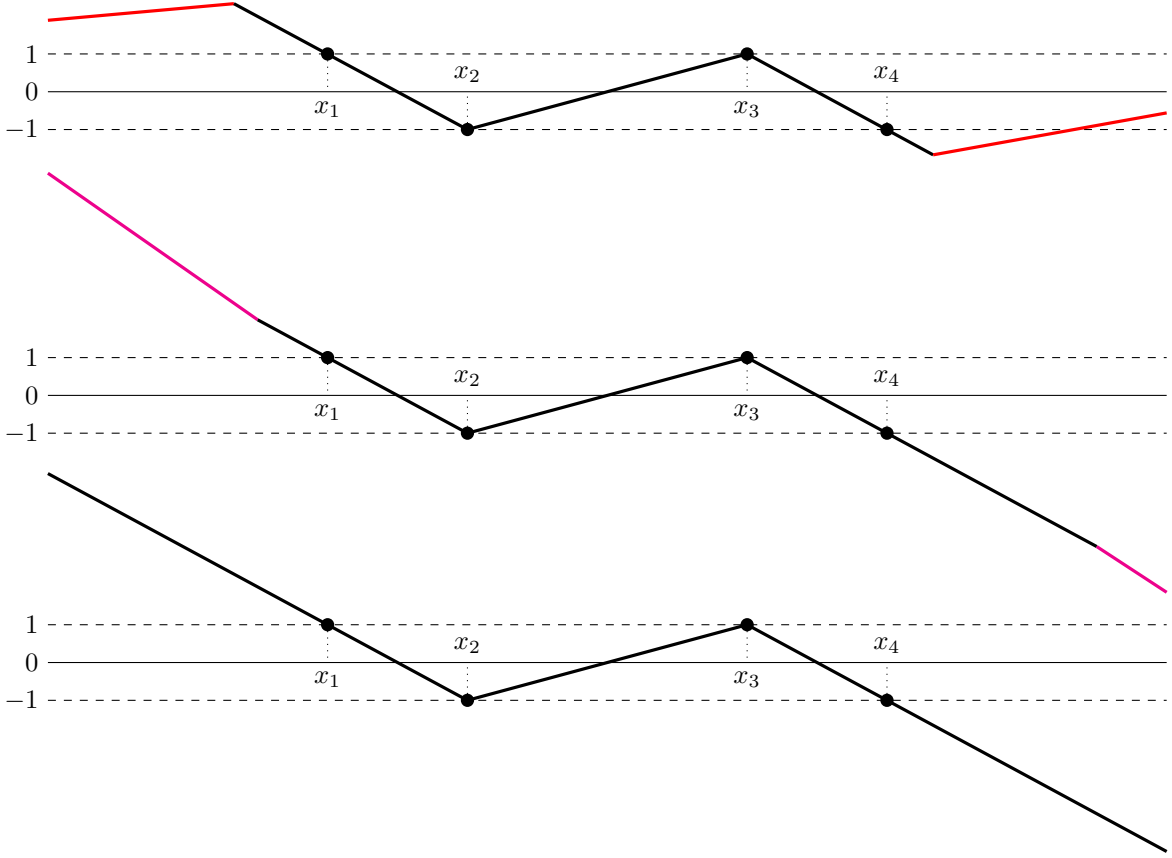

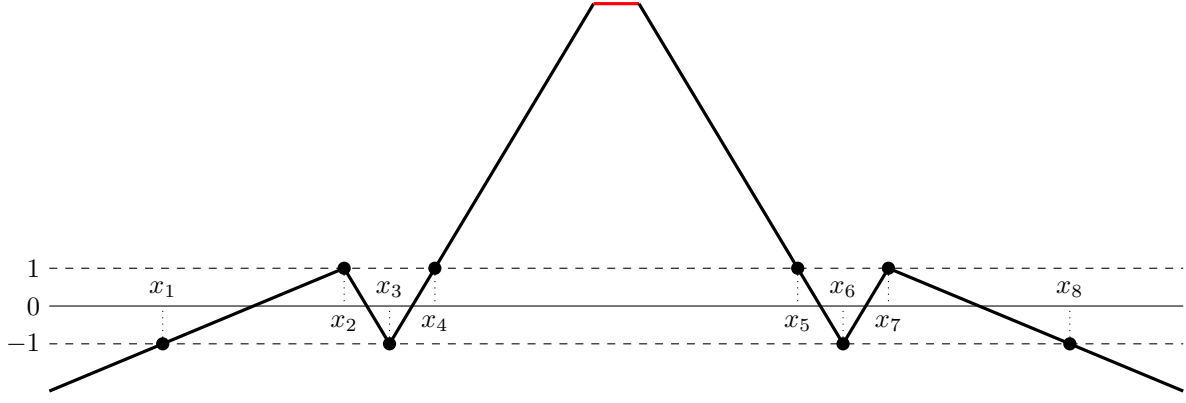
\begin{figure}[t]
\centering
\begin{tikzpicture}[xscale=\ifopt\else\ifarxiv .6\else .5\fi\fi,yscale=.5]
\draw         (-12.5,  0) node[left]  {$0$} -- (12.5,  0);
\draw[dashed] (-12.5,  1) node[left]  {$1$} -- (12.5,  1);
\draw[dashed] (-12.5, -1) node[left] {$-1$} -- (12.5, -1);
\draw[very thick]
(-12.5, -2.25) --
( -6  ,  1   ) --
( -5  , -1   ) --
( -0.5,  8   );
\draw[red,very thick]
( -0.5,  8   ) --
(  0.5,  8   );
\draw[very thick]
( 12.5, -2.25) --
(  6  ,  1   ) --
(  5  , -1   ) --
(  0.5,  8   );
\draw[dotted] (-10, -1) node[point] {} -- (-10, 0) node[above] {$x_1$};
\draw[dotted] ( -6,  1) node[point] {} -- ( -6, 0) node[below] {$x_2$};
\draw[dotted] ( -5, -1) node[point] {} -- ( -5, 0) node[above] {$x_3$};
\draw[dotted] ( -4,  1) node[point] {} -- ( -4, 0) node[below] {$x_4$};
\draw[dotted] (  4,  1) node[point] {} -- (  4, 0) node[below] {$x_5$};
\draw[dotted] (  5, -1) node[point] {} -- (  5, 0) node[above] {$x_6$};
\draw[dotted] (  6,  1) node[point] {} -- (  6, 0) node[below] {$x_7$};
\draw[dotted] ( 10, -1) node[point] {} -- ( 10, 0) node[above] {$x_8$};
\end{tikzpicture}
\caption{This shows in function space the KKT point of \cref{eq:Pb} constructed in the proof of \cref{th:KKT.not}~\ref{th:KKT.not:Pb}, which is not a minimizer.  The red coloring indicates where this interpolant of the depicted dataset fails to be single turning, which by \cref{th:inter}~\ref{th:inter:Pb.Pb1} means that this example is suboptimal already in function space (and not sparsest).}
\label{f:KKT.not:Pb}
\end{figure}

\begin{proof}
For part~\ref{th:KKT.not:P}, let us write $(x_i, y_i)_{i = 1}^4$ for the dataset and $\theta = (a_j, w_j, b_j)_{j = 1}^4$ for the network.

Given any $p_1 < x_1$ and $p_4 > x_4$, we will define~$\theta$ so that
\begin{itemize}
\item
$f_\theta$~is switch hugging, i.e., $f_\theta(x_i) = y_i$ for all $i \in [4]$,
\item
the kinks of the four neurons are at
\begin{align}
p_1 & &
p_2 & \coloneqq x_2 &
p_3 & \coloneqq x_3 &
p_4 &
\end{align}
respectively, i.e., $-b_j / w_j = p_j$ for all $j \in [4]$, and
\item
the first (resp., last) two neurons are active in the positive (resp., negative) direction, i.e., $w_1, w_2 > 0$ and $w_3, w_4 < 0$.
\end{itemize}
With the latter two properties, the switch hugging constraints are a system of four linear equations in the slopes $t_j \coloneqq a_j w_j$ of the four neurons, which has the following unique solution:
\begin{align}
t_1 & = -\frac{5 p_4}{p_4 - p_1} &
t_2 & =  3 &
t_3 & =  3 &
t_4 & =  \frac{5 p_1}{p_4 - p_1}.
\end{align}
To obtain a KKT point, we need to have $a_j^2 = w_j^2$ for all four neurons, so
\begin{align}
a_1 & = -\sqrt{\frac{ 5 p_4}{p_4 - p_1}} &
a_2 & =  \sqrt{3} &
a_3 & = -\sqrt{3} &
a_4 & =  \sqrt{\frac{-5 p_1}{p_4 - p_1}} \\
w_1 & =  \sqrt{\frac{ 5 p_4}{p_4 - p_1}} &
w_2 & =  \sqrt{3} &
w_3 & = -\sqrt{3} &
w_4 & = -\sqrt{\frac{-5 p_1}{p_4 - p_1}},
\end{align}
and the biases are determined by the positions $-b_j / w_j = p_j$ of the kinks, so
\begin{align}
b_1 & = -p_1 \sqrt{\frac{ 5 p_4}{p_4 - p_1}} &
b_2 & = \sqrt{3} &
b_3 & = \sqrt{3} &
b_4 & =  p_4 \sqrt{\frac{-5 p_1}{p_4 - p_1}}.
\end{align}

We already have that $\theta$~is a feasible point of \cref{eq:P}, and all four interpolation constraints are tight, so to show that $\theta$~is a KKT point, it suffices to exhibit multipliers $\mu_i \geq 0$ such that
\[  \partial \biggl(\frac{1}{2} \sum_{j = 1}^4 (a_j^2 + w_j^2)\biggr)
=   (a_j, w_j, 0)_{j = 1}^4
\in \sum_{i = 1}^4 \mu_i \, \partial (y_i \, f_\theta(x_i)).\]
Thanks to the balancedness $a_j^2 = w_j^2$, in turn it suffices to restrict our attention to the hidden weights and biases, so after a little simplification, the equilibrium inclusion becomes the following system of eight equations in $(\mu_i)_{i = 1}^4$ and $\rho_2, \rho_3$, where the latter variables determine the subgradients in the Clarke subdifferentials of the two ReLUs at $x_2, x_3$ respectively:
\begin{align}
\begin{bmatrix}
\mu_1 \\ \mu_2 \\ \mu_3 \\ \mu_4
\end{bmatrix}^\top
\begin{bmatrix*}[r]
-y_1 x_1 &              0 &        y_1 x_1 & -y_1 x_1 \\
-y_2 x_2 & \rho_2 y_2 x_2 &        y_2 x_2 & -y_2 x_2 \\
-y_3 x_3 &        y_3 x_3 & \rho_3 y_3 x_3 & -y_3 x_3 \\
-y_4 x_4 &        y_4 x_4 &              0 & -y_4 x_4
\end{bmatrix*}
& =
\begin{bmatrix}
1 \\ 1 \\ 1 \\ 1
\end{bmatrix}^\top
& 
\begin{bmatrix}
\mu_1 \\ \mu_2 \\ \mu_3 \\ \mu_4
\end{bmatrix}^\top
\begin{bmatrix*}[r]
-y_1 &          0 &        y_1 & -y_1 \\
-y_2 & \rho_2 y_2 &        y_2 & -y_2 \\
-y_3 &        y_3 & \rho_3 y_3 & -y_3 \\
-y_4 &        y_4 &          0 & -y_4
\end{bmatrix*}
& =
\begin{bmatrix}
0 \\ 0 \\ 0 \\ 0
\end{bmatrix}^\top.
\end{align}
This system has a unique solution
\begin{align}
\mu_1 & = 2 &
\mu_2 & = \frac{7}{2} &
\mu_3 & = \frac{7}{2} &
\mu_4 & = 2 \\
       & &
\rho_2 & = \frac{3}{7} &
\rho_3 & = \frac{3}{7}, &
       &
\end{align}
in which all~$\mu_i$ are nonnegative and both of $\rho_2, \rho_3$ are in the closed interval $[0, 1]$, as required.

Using \cref{th:repr} and \cref{th:inter}~\ref{th:inter:P.P1}, it is straightforward to compute that the minimum of \cref{eq:P} for our dataset is $2 (1 + 2) = 6$, however the norm of~$\theta$ (without the biases) is
\[\sum_{j = 1}^4 \lvert t_j \rvert
= \frac{5 p_4}{p_4 - p_1} + 3 + 3 - \frac{5 p_1}{p_4 - p_1}
= 11,\]
so it is not a minimizer.  The gap of~$5$ is exactly the cost of the two kinks of~$f_\theta$ at~$p_1$ and~$p_4$, and is constant on the two-dimensional plateau obtained by varying~$p_1$ and~$p_4$.

That $f_\theta$~is not convexity correct is because of the two kinks at~$p_1$ and~$p_4$, which are not in any intermediate same-label segment of the dataset.

For part~\ref{th:KKT.not:Pb}, let us write $(x_i, y_i)_{i = 1}^8$ for the dataset and $\theta = (a_j, w_j, b_j)_{j = 1}^6$ for the network.  We will follow the same pattern as for part~\ref{th:KKT.not:P}, however the details are more involved.

We will define~$\theta$ so that
\begin{itemize}
\item
$f_\theta$~is switch hugging, i.e., $f_\theta(x_i) = y_i$ for all $i \in [8]$,
\item
the kinks of the six neurons are at
\begin{align}
p_1 & \coloneqq x_2 &
p_2 & \coloneqq x_3 &
p_3 &               &
p_4 &               &
p_5 & \coloneqq x_6 &
p_6 & \coloneqq x_7
\end{align}
respectively, where $p_3, p_4 \in (x_4, x_5)$ are to be determined, and
\item
the first (resp., last) three neurons are active in the positive (resp., negative) direction, i.e., $w_1, w_2, w_3 > 0$ and $w_4, w_5, w_6 < 0$.
\end{itemize}
With the latter two properties, the switch hugging constraints are a system of eight equations in the slopes $t_j \coloneqq a_j w_j$ of the six neurons and in the two kinks $p_3$~and~$p_4$, which has the following unique solution:
\begin{align}
t_1 & = -\frac{5}{2} &
t_2 & =  4 &
t_3 & = -2 &
t_4 & =  2 &
t_5 & = -4 &
t_6 & =  \frac{5}{2} \\
& & & &
p_3 & = -\frac{1}{2} &
p_4 & =  \frac{1}{2}.
& & & &
\end{align}
This satisfies $p_3, p_4 \in (x_4, x_5) = (-4, 4)$, and we remark that the six signs of $a_j = t_j / w_j$ are such that $f_\theta$~is convexity correct.

To obtain a KKT point, we need to have $a_j^2 = w_j^2 + b_j^2$ for all $j \in [6]$, so from the values of~$p_j$ and~$t_j$ we define
\begin{align}
\lvert a_j \rvert & \coloneqq
\sqrt[4]{1 + p_j^2}
\sqrt{\lvert t_j \rvert} &
\lvert w_j \rvert & \coloneqq
\frac{\sqrt{\lvert t_j \rvert}}
     {\sqrt[4]{1 + p_j^2}},
\end{align}
then set the sign of~$w_j$ as we specified above, set the sign of~$a_j$ to satisfy $t_j = a_j w_j$, and define $b_j \coloneqq -w_j p_j$.

We already have that $\theta$~is a feasible point of \cref{eq:Pb}, and all eight interpolation constraints are tight, so to show that $\theta$~is a KKT point, it suffices to exhibit multipliers $\mu_i \geq 0$ such that
\[  \partial \biggl(\frac{1}{2} \sum_{j = 1}^6 (a_j^2 + w_j^2 + b_j^2)\biggr)
=   (a_j, w_j, b_j)_{j = 1}^6
\in \sum_{i = 1}^8 \mu_i \, \partial (y_i \, f_\theta(x_i)).\]
Thanks to the balancedness $a_j^2 = w_j^2 + b_j^2$, in turn it suffices to restrict our attention to the hidden weights and biases, so after a little simplification, the equilibrium inclusion becomes the following system of 12~equations in $(\mu_i)_{i = 1}^8$ and $\rho_1, \rho_2, \rho_5, \rho_6$, where the latter variables determine the subgradients in the Clarke subdifferentials of the first two and the last two ReLUs at $x_2, x_3, x_6, x_7$ respectively:
\begin{align}
\begin{bmatrix}
\mu_1 \\ \mu_2 \\ \mu_3 \\ \mu_4 \\ \mu_5 \\ \mu_6 \\ \mu_7 \\ \mu_8
\end{bmatrix}^\top
\begin{bmatrix*}[r]
              0 &              0 &        0 & y_1 x_1 &        -y_1 x_1 &        y_1 x_1 \\
-\rho_1 y_2 x_2 &              0 &        0 & y_2 x_2 &        -y_2 x_2 &        y_2 x_2 \\
       -y_3 x_3 & \rho_2 y_3 x_3 &        0 & y_3 x_3 &        -y_3 x_3 &        y_3 x_3 \\
       -y_4 x_4 &        y_4 x_4 &        0 & y_4 x_4 &        -y_4 x_4 &        y_4 x_4 \\
       -y_5 x_5 &        y_5 x_5 & -y_5 x_5 &       0 &        -y_5 x_5 &        y_5 x_5 \\
       -y_6 x_6 &        y_6 x_6 & -y_6 x_6 &       0 & -\rho_5 y_6 x_6 &        y_6 x_6 \\
       -y_7 x_7 &        y_7 x_7 & -y_7 x_7 &       0 &               0 & \rho_6 y_7 x_7 \\
       -y_8 x_8 &        y_8 x_8 & -y_8 x_8 &       0 &               0 &              0
\end{bmatrix*}
& =
\begin{bmatrix}
\frac{1}{\sqrt{1 + p_1^2}} \\
\frac{1}{\sqrt{1 + p_2^2}} \\
\frac{1}{\sqrt{1 + p_3^2}} \\
\frac{1}{\sqrt{1 + p_4^2}} \\
\frac{1}{\sqrt{1 + p_5^2}} \\
\frac{1}{\sqrt{1 + p_6^2}}
\end{bmatrix}^\top
\\
\begin{bmatrix}
\mu_1 \\ \mu_2 \\ \mu_3 \\ \mu_4 \\ \mu_5 \\ \mu_6 \\ \mu_7 \\ \mu_8
\end{bmatrix}^\top
\begin{bmatrix*}[r]
          0 &          0 &    0 & y_1 &        -y_1 &        y_1 \\
-\rho_1 y_2 &          0 &    0 & y_2 &        -y_2 &        y_2 \\
       -y_3 & \rho_2 y_3 &    0 & y_3 &        -y_3 &        y_3 \\
       -y_4 &        y_4 &    0 & y_4 &        -y_4 &        y_4 \\
       -y_5 &        y_5 & -y_5 &   0 &        -y_5 &        y_5 \\
       -y_6 &        y_6 & -y_6 &   0 & -\rho_5 y_6 &        y_6 \\
       -y_7 &        y_7 & -y_7 &   0 &           0 & \rho_6 y_7 \\
       -y_8 &        y_8 & -y_8 &   0 &           0 &          0
\end{bmatrix*}
& =
\begin{bmatrix}
-\frac{p_1}{\sqrt{1 + p_1^2}} \\
-\frac{p_2}{\sqrt{1 + p_2^2}} \\
-\frac{p_3}{\sqrt{1 + p_3^2}} \\
-\frac{p_4}{\sqrt{1 + p_4^2}} \\
-\frac{p_5}{\sqrt{1 + p_5^2}} \\
-\frac{p_6}{\sqrt{1 + p_6^2}}
\end{bmatrix}^\top.
\end{align}
This system has a unique solution
\begin{align}
\mu_1 = \mu_8 & =
\frac{5 \sqrt{37}                - 12 \sqrt{5}}{20} &
\mu_2 = \mu_7 & =
\frac{5 \sqrt{37} +  4 \sqrt{26} -  4 \sqrt{5}}{4} \\
\mu_3 = \mu_6 & =
\frac{5 \sqrt{37} + 10 \sqrt{26} +  6 \sqrt{5}}{5} &
\mu_4 = \mu_5 & =
\frac{               5 \sqrt{26} +  7 \sqrt{5}}{5} \\
\rho_1 = \rho_6 & =
\frac{ 4 \bigl(31 \sqrt{37} + 37 \sqrt{26}             \bigr)}
     {37 \bigl( 5 \sqrt{37} +  4 \sqrt{26} - 4 \sqrt{5}\bigr)} &
\rho_2 = \rho_5 & =
\frac{ 3 \bigl(               35 \sqrt{26} + 52 \sqrt{5}\bigr)}
     {26 \bigl( 5 \sqrt{37} + 10 \sqrt{26} +  6 \sqrt{5}\bigr)},
\end{align}
in which all~$\mu_i$ are nonnegative and each of $\rho_1, \rho_2, \rho_5, \rho_6$ is in the closed interval $[0, 1]$, as required.

That $\theta$~is not a minimizer of \cref{eq:Pb} is due to the suboptimality of the two kinks at~$p_3$ and~$p_4$, whose combined cost of $2 \cdot 2 \sqrt{1 + (1 / 2)^2} = 2 \sqrt{5}$ is greater than the cost of~$4$ for the optimal one kink at~$0$ with slope change of~$4$.

Finally, $f_\theta$~is not single turning because of the two kinks at~$p_3$ and~$p_4$, and not sparsest because it has $6$~kinks rather than the minimal~$5$.
\end{proof}

\begin{rem}
\label{rem:KKT.not}
\cref{th:KKT.not} is robust to sufficiently small independent perturbations of the input locations in both parts; the symmetric datasets are used only to simplify the constructions. We also record the following additional properties, whose direct but lengthy verifications we omit because they are not used elsewhere: the KKT plateau constructed in part~\ref{th:KKT.not:P} consists of local minimizers of \cref{eq:P}; within that family, only the point whose outer kinks are at~$x_1$ and~$x_4$ is KKT for \cref{eq:Pb}, and it is suboptimal; and this point and the KKT point constructed in part~\ref{th:KKT.not:Pb} are strict local minimizers of \cref{eq:Pb}.
\end{rem}

\clearpage
\section{Proofs for minimizing \texorpdfstring{$\ell_2$}{l2}-regularized logistic loss}
\label{s:min.reg.loss}

In this section, we develop several auxiliary results that, together with some of the lemmas from \cref{s:inter}, contribute to the proof of \cref{th:min.reg.loss}, which is our next main result and the counterpart of \cref{th:inter}, where here the focus is on minimizing $\ell_2$-regularized logistic loss instead of minimizing interpolator norm.

Before that effort, we provide in \Cref{f:land.reg} optimization landscape plots that correspond to those in \Cref{f:land.norm} and therefore provide numerical evidence that an analogous qualitative landscape phenomenon occurs for the small-$\lambda$ regularized loss.

\begin{figure}[t]
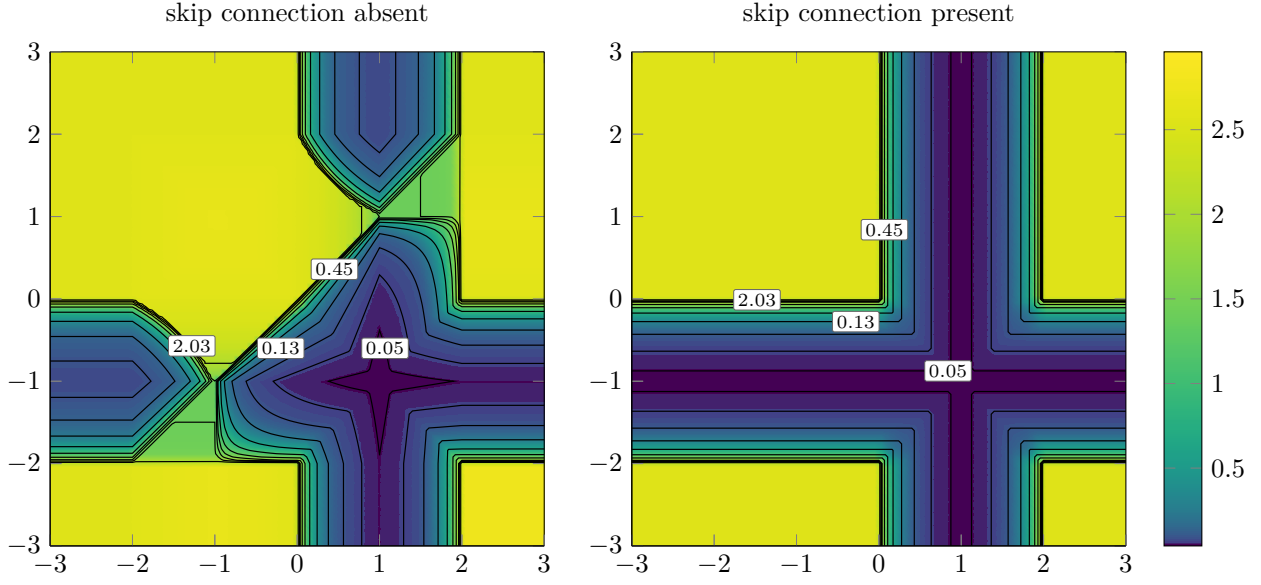

\centering
{\tikzexternalenable
\tikzsetnextfilename{cache_reg_a0_pq}
\tikzpicturedependsonfile{plot_dat_tex/obj_reg_a0_pq.dat}
\begin{tikzpicture}
\begin{axis}
[width=\ifopt\else\ifarxiv .475\else .5\fi\fi\textwidth,
 height=\ifopt\else\ifarxiv .475\else .5\fi\fi\textwidth,
 xmin=-3, xmax=3,
 ymin=-3, ymax=3,
 view={0}{90},
 colormap name=warped3viridis,
 point meta min=0.04,
 point meta max=2.96,
 point meta=explicit,
 unbounded coords=jump,
 title={skip connection absent}]
\addplot3
[surf,
 shader=interp,
 mesh/cols=76]
table
[x=p,
 y=q,
 z=value,
 meta expr={min(\thisrow{value},2.96)}]
{plot_dat_tex/obj_reg_a0_pq.dat};
\input{plot_dat_tex/cont_reg_a0_pq}
\end{axis}
\end{tikzpicture}
\tikzsetnextfilename{cache_reg_opt_a_pq}
\tikzpicturedependsonfile{plot_dat_tex/obj_reg_opt_a_pq.dat}
\begin{tikzpicture}
\begin{axis}
[width=\ifopt\else\ifarxiv .475\else .5\fi\fi\textwidth,
 height=\ifopt\else\ifarxiv .475\else .5\fi\fi\textwidth,
 xmin=-3, xmax=3,
 ymin=-3, ymax=3,
 view={0}{90},
 colormap name=warped3viridis,
 colorbar,
 point meta min=0.04,
 point meta max=2.96,
 point meta=explicit,
 unbounded coords=jump,
 title={skip connection present}]
\addplot3
[surf,
 shader=interp,
 mesh/cols=76]
table
[x=p,
 y=q,
 z=value,
 meta expr={min(\thisrow{value},2.96)}]
{plot_dat_tex/obj_reg_opt_a_pq.dat};
\input{plot_dat_tex/cont_reg_opt_a_pq}
\end{axis}
\end{tikzpicture}}
\caption{For the dataset in \cref{th:KKT.not}~\ref{th:KKT.not:P} and two-layer ReLU networks of width~$4$, here we depict the landscape of the empirical logistic loss with $\ell_2$~regularization of strength $\lambda = 0.001$ of network weights (biases are not regularized).  The remainder of the setup is analogous to that for \Cref{f:land.norm}: the network is symmetric thanks to the symmetry of the dataset, with neuron kinks at $\pm p_\dag$~and~$\pm p_\ddag$, which parameterize the axes in the plots.  The neuron slopes and the skip connection weight (if present) are optimized.  For greater visibility, the contour lines and the color map follow superlinear progressions.  \textbf{Left:}  The skip connection is absent, so the landscape is of~\hyperref[eq:L]{$L\mathrm{^{0.001}}$} with $m = 4$.  \textbf{Right:}  The skip connection is present, so the landscape is of~\hyperref[eq:L1]{$L\mathrm{^{0.001}_1}$} with $m = 4$.  Within this two-dimensional profile, adding the skip connection merges the three visible valleys and removes the suboptimal flat regions, paralleling the qualitative change in \Cref{f:land.norm}.}
\label{f:land.reg}
\end{figure}

\paragraph{Losses and regularizations.}

Let 
$\ell(z) \coloneqq
 \log(1 + \e^{-z})$
denote the logistic loss function, and
\[L(\theta, a_0, b_0) \coloneqq
  \frac{1}{n}
  \sum_{i = 1}^n
  \ell\bigl(y_i \, f_{\theta, a_0, b_0}(x_i)\bigr)\]
denote the logistic loss over the dataset $(x_i, y_i)_{i = 1}^n$.  When the skip connection is absent, instead of $L(\theta, 0, 0)$ we may write just $L(\theta)$.

We consider the following four variants of the empirical logistic loss with $\ell_2$~regularization of strength $\lambda > 0$, depending on whether biases are not [\cref{eq:L,eq:L1}] or are [\cref{eq:Lb,eq:Lb1}] regularized, and on whether a skip connection is absent [\cref{eq:L,eq:Lb}] or present [\cref{eq:L1,eq:Lb1}].  The skip connection is always ``free'', i.e., its parameters $a_0$~and~$b_0$ are never regularized.
\begin{align}
L\mathrm{^\lambda}(\theta) & \coloneqq
L(\theta) +
\frac{\lambda}{2}
\sum_{j = 1}^m (a_j^2 + w_j^2)
\tag{$\mathrm{L^\lambda}$}\label{eq:L} \\
L\mathrm{^\lambda_1}(\theta, a_0, b_0) & \coloneqq
L(\theta, a_0, b_0) +
\frac{\lambda}{2}
\sum_{j = 1}^m (a_j^2 + w_j^2)
\tag{$\mathrm{L^\lambda_1}$}\label{eq:L1} \\
L\mathrm{^{\lambda, b}}(\theta) & \coloneqq
L(\theta) +
\frac{\lambda}{2}
\sum_{j = 1}^m (a_j^2 + w_j^2 + b_j^2)
\tag{$\mathrm{L^{\lambda, b}}$}\label{eq:Lb} \\
L\mathrm{^{\lambda, b}_1}(\theta, a_0, b_0) & \coloneqq
L(\theta, a_0, b_0) +
\frac{\lambda}{2}
\sum_{j = 1}^m (a_j^2 + w_j^2 + b_j^2)
\tag{$\mathrm{L^{\lambda, b}_1}$}\label{eq:Lb1}
\end{align}

Our first auxiliary result is a basic proposition that spells out a sufficient condition for impossibility that a network with the skip connection is a Clarke stationary point of the regularized loss (the biases can optionally be penalized).  In contrast to showing impossibility of KKT points for interpolator norm minimization, where the directional derivative of the square $\ell_2$-norm needed to be negative and Clarke directional derivatives of the negative margins needed to be nonpositive only at training points where interpolation is tight (cf.\ \cref{pr:dd.not.KKT}), here all training points are involved and it suffices that any one of the derivatives is negative (i.e., not necessarily that of the $\ell_2$-norm).

\begin{prop}
\label{pr:dd.not.Clarke}
For any regularization strength $\lambda > 0$, if there exists a direction $\theta', a'_0, b'_0$ such that $\mathrm{D} \bigl(\frac{1}{2} \sum_{j = 1}^m (a_j^2 + w_j^2)\bigr)[\theta', a'_0, b'_0] \leq 0$ (resp., $\mathrm{D} \bigl(\frac{1}{2} \sum_{j = 1}^m (a_j^2 + w_j^2 + b_j^2)\bigr)[\theta', a'_0, b'_0] \leq 0$), $\mathrm{D}^\circ \bigl(-y_i \, f_{\theta, a_0, b_0}(x_i)\bigr)[\theta', a'_0, b'_0] \leq 0$ for all $i \in [n]$, and at least one of those $1 + n$ inequalities is strict, then $\theta, a_0, b_0$ is not a Clarke stationary point of~\hyperref[eq:L1]{$L\mathrm{^\lambda_1}$} (resp.,~\hyperref[eq:Lb1]{$L\mathrm{^{\lambda, b}_1}$}).
\end{prop}

\begin{proof}
We consider~\hyperref[eq:Lb1]{$L\mathrm{^{\lambda, b}_1}$}, the proof for~\hyperref[eq:L1]{$L\mathrm{^\lambda_1}$} is analogous and simpler.  For a contradiction, suppose that $\mathbf{0} \in \partial L\mathrm{^{\lambda, b}_1}(\theta, a_0, b_0)$.  Then there exist subgradients $\bigl(\theta^{(i)}, a^{(i)}_0, b^{(i)}_0\bigr) \in \partial \bigl(-y_i \, f_{\theta, a_0, b_0}(x_i)\bigr)$ for each $i \in [n]$ such that
\[\mathbf{0}
= \frac{1}{n}
  \sum_{i = 1}^n
  \Bigl(-\ell'\bigl(y_i \, f_{\theta, a_0, b_0}(x_i)\bigr)\Bigr)
  \bigl(\theta^{(i)}, a^{(i)}_0, b^{(i)}_0\bigr)
+ \lambda
  \nabla \biggl(\frac{1}{2}
                \sum_{j = 1}^m (a_j^2 + w_j^2 + b_j^2)\biggr),\]
so recalling that the derivative of the logistic loss function is negative, we have
\begin{align}
0 & =    \Biggl\langle
         \begin{aligned}[t]
  &      \frac{1}{n}
         \sum_{i = 1}^n
         \Bigl(-\ell'\bigl(y_i \, f_{\theta, a_0, b_0}(x_i)\bigr)\Bigr)
         \bigl(\theta^{(i)}, a^{(i)}_0, b^{(i)}_0\bigr)
       + \lambda
         \nabla \biggl(\frac{1}{2}
                       \sum_{j = 1}^m (a_j^2 + w_j^2 + b_j^2)\biggr), \\
  &      (\theta', a'_0, b'_0)
         \Biggr\rangle
         \end{aligned} \\
  & =    \begin{aligned}[t]
  &      \frac{1}{n}
         \sum_{i = 1}^n
         \Bigl(-\ell'\bigl(y_i \, f_{\theta, a_0, b_0}(x_i)\bigr)\Bigr)
         \Bigl\langle
         \bigl(\theta^{(i)}, a^{(i)}_0, b^{(i)}_0\bigr),
         (\theta', a'_0, b'_0)
         \Bigr\rangle \\
  &    + \lambda \,
         \Biggl\langle
         \nabla \biggl(\frac{1}{2}
                       \sum_{j = 1}^m (a_j^2 + w_j^2 + b_j^2)\biggr),
         (\theta', a'_0, b'_0)
         \Biggr\rangle
         \end{aligned} \\
  & \leq \begin{aligned}[t]
  &      \frac{1}{n}
         \sum_{i = 1}^n
         \Bigl(-\ell'\bigl(y_i \, f_{\theta, a_0, b_0}(x_i)\bigr)\Bigr)
         \mathrm{D}^\circ \bigl(-y_i \, f_{\theta, a_0, b_0}(x_i)\bigr)
         [\theta', a'_0, b'_0] \\
  &    + \lambda \,
         \mathrm{D} \biggl(\frac{1}{2}
                           \sum_{j = 1}^m (a_j^2 + w_j^2 + b_j^2)\biggr)
         [\theta', a'_0, b'_0]
         \end{aligned} \\
  & <    0.
\ifopt\else\qedhere\fi
\end{align}
\end{proof}

\paragraph{Minimizers of the regularized losses.}

We start with a key lemma that, provided the regularization strength~$\lambda$ is sufficiently small depending on the dataset (and on whether the biases are penalized), shows that the optimization problems that determine the pairs of margins incident to all label switches (cf.\ the switch optimal property in \cref{th:min.reg.loss}, defined on \cpageref{sw.opt}) have strictly convex objectives, convex feasibility sets, and unique minimizers.  Before stating and proving the lemma, we define the optimization problems.

Given $\tau_k > 0$ and $\pi_k > 0$ for all $k \in [r]$, let $g_{\tau, \pi} \colon \mathbb{R} \to \mathbb{R}$ denote the CPA that has exactly~$r$ slopes and, for all $k \in [r]$, passes through points $(x_{\iota(k)}, y_{\iota(k)} \tau_k)$ and $(x_{\iota(k) + 1}, y_{\iota(k) + 1} \pi_k)$; this is well defined if and only if vectors $\tau$~and~$\pi$ satisfy
\begin{align}
\Phi(\tau, \pi) \coloneqq
\forall k \in [r - 1], \,
&     -\frac{\tau_k + \pi_k}
            {x_{\iota(k) + 1} - x_{\iota(k)}}
       (x_{\iota(k + 1)} - x_{\iota(k) + 1}) \\
& \leq \pi_k - \tau_{k + 1} \\
& \leq \frac{\tau_{k + 1} + \pi_{k + 1}}
            {x_{\iota(k + 1) + 1} - x_{\iota(k + 1)}}
       (x_{\iota(k + 1)} - x_{\iota(k) + 1}).
\end{align}

Recalling from \cref{s:inter} the CPA norms $\mathrm{R_1}$~and~$\mathrm{R^b_1}$, we consider the following two variants of the problem of minimizing the empirical logistic loss with norm regularization of strength $\lambda > 0$.
\begin{align}
\inf_{\tau, \pi \in \mathbb{R}_{> 0}^r}
\Biggl(
\frac{1}{n}
\sum_{i = 1}^n
\ell\bigl(y_i \, g_{\tau, \pi}(x_i)\bigr) +
\lambda \,
\mathrm{R_1}(g_{\tau, \pi})
\Biggr)
& \quad \text{such that }
\Phi(\tau, \pi)
\tag{$\mathrm{O_1}$}\label[prob]{eq:O1} \\
\inf_{\tau, \pi \in \mathbb{R}_{> 0}^r}
\Biggl(
\frac{1}{n}
\sum_{i = 1}^n
\ell\bigl(y_i \, g_{\tau, \pi}(x_i)\bigr) +
\lambda \,
\mathrm{R^b_1}(g_{\tau, \pi})
\Biggr)
& \quad \text{such that }
\Phi(\tau, \pi)
\tag{$\mathrm{O^b_1}$}\label[prob]{eq:Ob1}
\end{align}

\begin{lem}
\label{l:O1.Ob1}
Suppose the dataset has at least two label switches, i.e., $r \geq 2$, and let:
\begin{align}
d_{\min} & \coloneqq \min_{k \in [r]} (x_{\iota(k) + 1} - x_{\iota(k)}) &
x_{\max} & \coloneqq \max_{i \in [n]} \lvert x_i \rvert.
\end{align}
We have that each of \cref{eq:O1,eq:Ob1} has a strictly convex objective and a convex feasible set, and moreover that:
\begin{enumerate}[(i)]
\item
\label{l:O1.Ob1:O1}
provided $\lambda \leq d_{\min} / 24 (r - 1) n (\log n)$, \cref{eq:O1} has a unique minimizer;
\item
\label{l:O1.Ob1:Ob1}
provided $\lambda \leq d_{\min} / 24 (r - 1) n (\log n) \sqrt{1 + x_{\max}^2}$, \cref{eq:Ob1} has a unique minimizer.
\end{enumerate}
\end{lem}

\begin{proof}
We first consider part~\ref{l:O1.Ob1:O1}, where by \cref{th:repr} we have that
\[\mathrm{R_1}(g_{\tau, \pi}) =
  \frac{\tau_1 + \pi_1}
       {x_{\iota(1) + 1} - x_{\iota(1)}}
+ 2 \sum_{k = 2}^{r - 1}
  \frac{\tau_k + \pi_k}
       {x_{\iota(k) + 1} - x_{\iota(k)}}
+ \frac{\tau_r + \pi_r}
       {x_{\iota(r) + 1} - x_{\iota(r)}},\]
so the regularizing term is linear in vectors $\tau$~and~$\pi$, and thus convex.

For every input~$x_i$ that is in an intermediate same-label segment of the dataset, i.e., such that $i \in \{\iota(k) + 1, \dots, \iota(k + 1)\}$ for some $k \in [r - 1]$, we have that
\begin{align}
\ell\bigl(y_i \, g_{\tau, \pi}(x_i)\bigr) =
\max \Biggl\{& \ell\biggl(
               \pi_k + 
               \frac{\tau_k + \pi_k}
                    {x_{\iota(k) + 1} - x_{\iota(k)}}
               (x_i - x_{\iota(k) + 1})\biggr), \\
             & \ell\biggl(
               \tau_{k + 1} +
               \frac{\tau_{k + 1} + \pi_{k + 1}}
                    {x_{\iota(k + 1) + 1} - x_{\iota(k + 1)}}
               (x_{\iota(k + 1)} - x_i)\biggr)\Biggr\},
\end{align}
which is the maximum of two compositions of the logistic loss function~$\ell$ with linear functions in $\tau$~and~$\pi$, and thus also convex.  For inputs~$x_i$ that are in the first or last same-label segments of the dataset, the argument is simpler as no maximum and only one composition is involved.

Hence the objective of \cref{eq:O1} is convex, and moreover strictly convex since, for every variable~$\tau_k$, it contains the term $\frac{1}{n} \ell(\tau_k)$ (at~$x_{\iota(k)}$), and for every variable~$\pi_k$, it contains the term $\frac{1}{n} \ell(\pi_k)$ (at~$x_{\iota(k) + 1}$).

Since the feasibility condition~$\Phi(\tau, \pi)$ is a conjunction of linear inequalities, the feasible set of \cref{eq:O1} is convex.

To complete the proof for part~\ref{l:O1.Ob1:O1}, it suffices to show that the feasible set has a compact convex subset such that the infimum of the objective outside that subset is greater than the value of the objective at a point in the subset.  We will obtain such a compact convex subset by bounding all variables~$\tau_k$ and~$\pi_k$ away from zero and away from infinity; that it is compact and convex will follow from its boundedness together with nonstrictness and affinity of all its defining inequalities.

Let $z \coloneqq \log(2 n / \ell(0))$, and let $\tau^\dag_k \coloneqq z$ and $\pi^\dag_k \coloneqq z$ for all $k \in [r]$.  Then the objective of \cref{eq:O1} at $\tau^\dag, \pi^\dag$ is less than
\begin{multline}
  \ell(z)             + \frac{4 \lambda z (r - 1)}{d_{\min}}
< \e^{-z}             + \frac{z}{6 n (\log n)}
= \frac{\ell(0)}{2 n} + \frac{\log(n) + \log(2 / \ell(0))}{6 n (\log n)} \\
< \frac{\ell(0)}{2 n} + \frac{2 \log(n)}{6 n (\log n)}
= \frac{\ell(0) + 2 / 3}{2 n}
< \frac{\ell(0)}{n}.
\end{multline}
Since the objective of \cref{eq:O1} at any feasible point $\tau, \pi$ is greater than $\ell\bigl(\min \{\tau_k, \pi_k \mid k \in [r]\}\bigr) / n$, it follows that there exists $\varepsilon > 0$ such that the infimum over all feasible points $\tau, \pi$ with $\tau_k < \varepsilon$ or $\pi_k < \varepsilon$ for some $k \in [r]$ is greater than the value at $\tau^\dag, \pi^\dag$.

It remains to observe that $R_1(g_{\tau, \pi})$~is linear in every variable~$\tau_k$ and every variable~$\pi_k$ with a positive coefficient, and so there exists $N > 0$ such that the infimum of the objective of \cref{eq:O1} over all feasible points $\tau, \pi$ with $\tau_k > N$ or $\pi_k > N$ for some $k \in [r]$ is greater than the value at $\tau^\dag, \pi^\dag$.

For part~\ref{l:O1.Ob1:Ob1}, the reasoning follows the same pattern.  From \cref{th:repr}, it is straightforward to derive that
\[\mathrm{R^b_1}(g_{\tau, \pi}) =
  \sum_{k = 1}^{r - 1}
  \sqrt{\begin{gathered}
        \Biggl(\frac{\tau_k + \pi_k}
                    {x_{\iota(k) + 1} - x_{\iota(k)}}
             + \frac{\tau_{k + 1} + \pi_{k + 1}}
                    {x_{\iota(k + 1) + 1} - x_{\iota(k + 1)}}\Biggr)^2 + \\
        \Biggl(\frac{x_{\iota(k) + 1} \tau_k + x_{\iota(k)} \pi_k}
                    {x_{\iota(k) + 1} - x_{\iota(k)}}
             + \frac{x_{\iota(k + 1) + 1} \tau_{k + 1} + x_{\iota(k + 1)} \pi_{k + 1}}
                    {x_{\iota(k + 1) + 1} - x_{\iota(k + 1)}}\Biggr)^2
        \end{gathered}}\]
and so the unscaled regularizer is a sum of $r - 1$ Euclidean lengths of vectors that are linear in $\tau$~and~$\pi$, thus convex.

The only other significant difference is in bounding the variables away from infinity; for that, it suffices to recall that $\mathrm{R^b_1}(g_{\tau, \pi}) \geq \mathrm{R_1}(g_{\tau, \pi})$.
\end{proof}

\begin{cor}
\label{c:pos.margin}
Suppose $r \geq 2$ and $\lambda$~satisfies the upper bound in part~\ref{l:O1.Ob1:O1} (resp.,~\ref{l:O1.Ob1:Ob1}) of \cref{l:O1.Ob1}. Then the unique minimizer of \cref{eq:O1} (resp., \cref{eq:Ob1}) has value less than $\ell(0) / n$ and lies in $\mathbb{R}_{> 0}^{2 r}$.  Also networks whose \hyperref[eq:L1]{$L\mathrm{^\lambda_1}$} (resp.,~\hyperref[eq:Lb1]{$L\mathrm{^{\lambda, b}_1}$}) values are less than $\ell(0) / n$ exist, and every such network has a positive margin.
\end{cor}

\begin{proof}
Let $z \coloneqq \log(2 n / \ell(0))$, and let $\tau^\dag_k \coloneqq z$ and $\pi^\dag_k \coloneqq z$ for all $k \in [r]$. Then $\Phi(\tau^\dag, \pi^\dag)$ holds, and the calculation in the proof of \cref{l:O1.Ob1} gives $\frac{1}{n} \sum_{i = 1}^n \ell\bigl(y_i \, g_{\tau^\dag, \pi^\dag}(x_i)\bigr) + \lambda \mathrm{R_1}(g_{\tau^\dag, \pi^\dag}) < \ell(0) / n$. Under the bound in \cref{l:O1.Ob1}~\ref{l:O1.Ob1:Ob1}, the same inequality holds with~$\mathrm{R^b_1}$, because $\mathrm{R^b_1}(g_{\tau^\dag, \pi^\dag}) \leq \sqrt{1 + x_{\max}^2} \, \mathrm{R_1}(g_{\tau^\dag, \pi^\dag})$ as every kink of~$g_{\tau^\dag, \pi^\dag}$ belongs to $[x_1, x_n]$. By \cref{th:repr}, $g_{\tau^\dag, \pi^\dag}$~has a network representor attaining the relevant regularizer.

On the other hand, if a network function~$f_{\theta, a_0, b_0}$ has a nonpositive margin, then $y_i \, f_{\theta, a_0, b_0}(x_i) \leq 0$ for some $i \in [n]$, and therefore
\[L\mathrm{^{\lambda, b}_1}(\theta, a_0, b_0)
  \geq L\mathrm{^\lambda_1}(\theta, a_0, b_0)
  \geq \frac{1}{n}
       \sum_{i' = 1}^n
       \ell\bigl(y_{i'} \, f_{\theta, a_0, b_0}(x_{i'})\bigr)
  \geq \frac{\ell(0)}{n}.\]

Finally, if one allows the closed domain $\tau, \pi \in \mathbb{R}_{\geq 0}^r$ in the finite-dimensional problems, any boundary point has objective at least $\ell(0) / n$, because some switch-incident margin is zero. The strict comparison above therefore places the unique minimizer in~$\mathbb{R}_{> 0}^{2 r}$.
\end{proof}

We recall from \cref{s:inter} two properties of a CPA $g \colon \mathbb{R} \to \mathbb{R}$ with respect to the dataset: convexity correct and single turning; and we define the following two properties:
\begin{description}
\item[(switch optimal)]
\label{sw.opt}
depending on whether biases are not or are regularized, for all $k \in [r]$, we have $g(x_{\iota(k)}) = y_{\iota(k)} \tau^\star_k$ and $g(x_{\iota(k) + 1}) = y_{\iota(k) + 1} \pi^\star_k$, where $\tau^\star, \pi^\star$ is the unique minimizer of \cref{eq:O1} or \cref{eq:Ob1} (respectively);
\item[(interval turning)]
\label{int.turn}
for all $k \in [r - 1]$, letting~$p$ and~$p'$ be the first and last (respectively) kinks of~$g$ \ifarxiv within\else in\fi\ $[x_{\iota(k) + 1}, x_{\iota(k + 1)}]$, the interval $(p, p')$ contains no inputs from the dataset.
\end{description}

\begin{rem}
If $g$~is single turning, then it is also interval turning: since within every intermediate segment of the dataset there is exactly one kink, the first kink~$p$ and the last kink~$p'$ coincide, so the interval $(p, p')$ is empty.
\end{rem}

\begin{thm}
\label{th:min.reg.loss}
Under \cref{ass:data}, and provided $\lambda \leq d_{\min} / 24 (r - 1) n (\log n)$ if biases are not regularized (i.e., for losses~\hyperref[eq:L]{$L\mathrm{^\lambda}$} and~\hyperref[eq:L1]{$L\mathrm{^\lambda_1}$}) or $\lambda \leq d_{\min} / 24 (r - 1) n (\log n) \sqrt{1 + x_{\max}^2}$ if biases are regularized (i.e., for losses~\hyperref[eq:Lb]{$L\mathrm{^{\lambda, b}}$} and~\hyperref[eq:Lb1]{$L\mathrm{^{\lambda, b}_1}$}), all of the following hold.
\begin{enumerate}[(i)]

\item
\label{th:min.reg.loss:L.L1}
Losses~\hyperref[eq:L]{$L\mathrm{^\lambda}$} and~\hyperref[eq:L1]{$L\mathrm{^\lambda_1}$} have the same minimizers in function space (over all~$m$), which are exactly all CPAs $g \colon \mathbb{R} \to \mathbb{R}$ that are switch optimal, convexity correct, and interval turning.

\item
\label{th:min.reg.loss:Lb.Lb1}
Losses~\hyperref[eq:Lb]{$L\mathrm{^{\lambda, b}}$} and~\hyperref[eq:Lb1]{$L\mathrm{^{\lambda, b}_1}$} have the same unique minimizer in function space (over all~$m$), which is switch optimal, convexity correct, and single turning; and thus with $r - 1$ kinks in total.

\item
\label{th:min.reg.loss:stat}
For losses~\hyperref[eq:L1]{$L\mathrm{^\lambda_1}$} and~\hyperref[eq:Lb1]{$L\mathrm{^{\lambda, b}_1}$}, we have that every Clarke stationary point $(\theta, a_0, b_0) \in \mathbb{R}^{3 m + 2}$ whose margin is positive (with any network width~$m$ for which such points exist) is a minimizer (globally, including over all~$m$).
\end{enumerate}
\end{thm}

Having now stated our main result in this section, we continue the development toward its proof.  The following lemma shows that the optima of the regularized loss minimization problems with the skip connection are attained, even for any fixed network width.  There need to be at least two label switches in the dataset (which is implied by \cref{ass:data}), however there is no constraint on~$\lambda$.

\begin{lem}
\label{l:attain.r.l}
Suppose the dataset has at least two label switches, i.e., $r \geq 2$.  For each of~\hyperref[eq:L1]{$L\mathrm{^\lambda_1}$} and~\hyperref[eq:Lb1]{$L\mathrm{^{\lambda, b}_1}$} restricted to any network width $m \in \mathbb{N}$, and with any fixed regularization strength $\lambda > 0$, the minimum is attained.
\end{lem}

\begin{proof}
Fix $m \in \mathbb{N}$ and $\lambda > 0$. We first consider~\hyperref[eq:L1]{$L\mathrm{^\lambda_1}$}. Let $\bigl(\theta^{(k)}, a_0^{(k)}, b_0^{(k)}\bigr)_{k \in \mathbb{N}}$ be a minimizing sequence, then the maximum~$N$ of its \hyperref[eq:L1]{$L\mathrm{^\lambda_1}$}~values is finite. Since both the empirical logistic loss and the regularization term are nonnegative, it follows that
\begin{equation}
\frac{1}{2}
\sum_{j = 1}^m
\Bigl(\bigl(a_j^{(k)}\bigr)^2 + \bigl(w_j^{(k)}\bigr)^2\Bigr)
\leq \frac{N}{\lambda}.
\label{eq:attain.r.l}
\end{equation}

As in the proof of \cref{l:attain.int}, we may alter each parameterization without changing its outputs at the training inputs and without increasing the regularization term. If the kink of a neuron lies outside $[x_1, x_n]$, then the contribution of that neuron is affine on all training inputs. Its contribution there may therefore be absorbed into the free affine skip connection, after which the neuron may be replaced by the zero neuron. A neuron with zero hidden weight contributes a constant and can be treated in the same way. We may thus assume that every nonzero neuron has its kink in $[x_1, x_n]$. By positive homogeneity of the ReLU, each remaining nonzero neuron can be balanced so that $\lvert a_j^{(k)} \rvert = \lvert w_j^{(k)} \rvert$ without changing the represented function or increasing the regularization term. Inequality~(\ref{eq:attain.r.l}) then bounds $a_j^{(k)}$ and~$w_j^{(k)}$. Also $\lvert b_j^{(k)} \rvert \leq \lvert w_j^{(k)} \rvert \max_{i = 1}^n \lvert x_i \rvert$.

Thus the parameter vectors~$\theta^{(k)}$ lie in a bounded set. In particular, the pre-skip output vectors $h^{(k)} \coloneqq \bigl(f_{\theta^{(k)}, 0, 0}(x_i)\bigr)_{i = 1}^n$ form a bounded sequence. We now show that the skip parameters are bounded.

Suppose, to the contrary, that after passing to a subsequence, $\bigl\lVert a_0^{(k)}, b_0^{(k)}\bigr\lVert_2 \to \infty$. Passing to a further subsequence gives $(a_0^{(k)}, b_0^{(k)}) / \bigl\lVert a_0^{(k)}, b_0^{(k)}\bigr\lVert_2 \to (\alpha_0, \beta_0)$ for some unit vector $(\alpha_0, \beta_0)$. For every $i \in [n]$, nonnegativity of the other summands implies $\ell\Bigl(y_i \bigl(a_0^{(k)} x_i + b_0^{(k)} + h_i^{(k)}\bigr)\Bigr) \leq n N$. Since $\ell(z) = \log(1 + \e^{-z})$ is strictly decreasing and tends to infinity as $z \to -\infty$, there exists a finite~$Q$, depending only on $n$ and~$N$, such that $y_i \bigl(a_0^{(k)} x_i + b_0^{(k)} +h_i^{(k)}\bigr) \geq -Q$ for all $i \in [n]$ and all $k \in \mathbb{N}$. Dividing by $\bigl\lVert a_0^{(k)}, b_0^{(k)}\bigr\lVert_2$ and letting $k \to \infty$, using boundedness of~$h^{(k)}$, yields $y_i (\alpha_0 x_i + \beta_0) \geq 0$ for every $i \in [n]$. Because the dataset has at least two label switches, there exist three training inputs $x_{i_1} < x_{i_2} < x_{i_3}$ whose labels alternate. After multiplying $x \mapsto \alpha_0 x + \beta_0$ by the common label of the first and third points if necessary, we obtain
\begin{align}
\alpha_0 x_{i_1} + \beta_0 & \geq 0 &
\alpha_0 x_{i_2} + \beta_0 & \leq 0 &
\alpha_0 x_{i_3} + \beta_0 & \geq 0.
\end{align}
Since $x_{i_2}$~is a strict convex combination of $x_{i_1}$ and~$x_{i_3}$, the affine value $\alpha_0 x_{i_2} + \beta_0$ is the same strict convex combination of the two outer values. It is therefore nonnegative, and hence it must be zero. The two nonnegative outer values must then both be zero as well. Consequently, the affine function vanishes at two distinct points and is identically zero. This gives $\alpha_0 = \beta_0 =0$, contradicting that $(\alpha_0, \beta_0)$ is a unit vector. The skip parameters are therefore bounded.

Hence the minimizing sequence is contained in a fixed compact subset of~$\mathbb{R}^{3 m + 2}$. Passing to a convergent subsequence and using continuity of the ReLU, the logistic loss, and the regularization term, we conclude that its limit attains the infimum of~\hyperref[eq:L1]{$L\mathrm{^\lambda_1}$}.

For~\hyperref[eq:Lb1]{$L\mathrm{^{\lambda, b}_1}$}, the argument is the same, except that the boundedness of the parameters follows immediately from the bounded regularization term. Indeed, we now have
\begin{equation}
\frac{1}{2}
\sum_{j = 1}^m
\Bigl(\bigl(a_j^{(k)}\bigr)^2 + \bigl(w_j^{(k)}\bigr)^2 + \bigl(b_j^{(k)}\bigr)^2\Bigr)
\leq \frac{N}{\lambda}.
\label{eq:attain.r.l.b}
\end{equation}
Thus all coordinates of~$\theta^{(k)}$ are bounded, and hence so are the pre-skip outputs on the finite training set. The preceding normalization argument, which uses the existence of three ordered inputs with alternating labels, then shows that the skip parameters are bounded as well. Compactness and continuity therefore imply that the infimum of~\hyperref[eq:Lb1]{$L\mathrm{^{\lambda, b}_1}$} is attained.
\end{proof}

Next, we identify two local geometric irregularities that capture the failures of interval turning for a convexity-correct CPA\dots

\begin{lem}
\label{l:it}
If a CPA $g \colon \mathbb{R} \to \mathbb{R}$ whose kinks are $(p_k)_{k = 1}^q$ is convexity correct, but it is not interval turning, then at least one of the following holds.
\begin{enumerate}[(i)]
\item
\label{l:it:conv}
Some kinks $p_k$ and~$p_{k'}$ with $k < k'$ are convex and such that $y_i = -1$ for all $x_i \in [p_k, p_{k'}]$, and there exists $x_i \in (p_k, p_{k'})$.
\item
\label{l:it:conc}
Some kinks $p_k$ and~$p_{k'}$ with $k < k'$ are concave and such that $y_i = 1$ for all $x_i \in [p_k, p_{k'}]$, and there exists $x_i \in (p_k, p_{k'})$.
\end{enumerate}
\end{lem}

\begin{proof}
This is immediate from the definitions of convexity correctness and interval turning.
\end{proof}

\dots and for either of those irregularities, identify a direction along which the derivative of the bias-free square $\ell_2$-norm of the network parameters is nonpositive, the Clarke derivative of the negative margin of each training point is nonpositive, and at least one of the latter is negative.  For an illustration of a part of the proof, we refer the reader to \Cref{f:it:conv.so} left.

\begin{lem}
\label{l:L1}
If $g = f_{\theta, a_0, b_0}$, whose kinks are $(p_k)_{k = 1}^q$, satisfies at least one of properties \ref{l:it:conv} or~\ref{l:it:conc} in \cref{l:it}, then there exists a direction $\theta', a'_0, b'_0$ such that $\mathrm{D} \bigl(\frac{1}{2} \sum_{j = 1}^m (a_j^2 + w_j^2)\bigr)[\theta', a'_0, b'_0] \leq 0$, $\mathrm{D}^\circ \bigl(-y_i \, f_{\theta, a_0, b_0}(x_i)\bigr)[\theta', a'_0, b'_0] \leq 0$ for all $i \in [n]$, and at least one of the latter $n$~inequalities is strict.
\end{lem}

\begin{proof}
We consider property~\ref{l:it:conv} in \cref{l:it}; since property~\ref{l:it:conc} is symmetric, this is without loss of generality.  Then $\theta$~must contain some neurons $(a_j, w_j, b_j)$ and $(a_{j'}, w_{j'}, b_{j'})$ such that:
\begin{itemize}
\item
$-b_j / w_j = p_k$ and $a_j > 0$;
\item
$-b_{j'} / w_{j'} = p_{k'}$ and $a_{j'} > 0$.
\end{itemize}
By the orientation-reversal reparameterization in \cref{l:rev}, for the following construction we may assume that $w_j > 0$ and $w_{j'} > 0$; and by symmetry, also that $a_j / w_j \geq a_{j'} / w_{j'}$. Moreover, by \cref{l:up.down}, the claimed nonpositivity of the Clarke directional derivatives of $-y_i \, f_{\theta, a_0, b_0}(x_i)$ is implied by nonpositivity of the corresponding one-sided limits of the ordinary directional derivatives, where also the claimed strictness is implied by strictness on both sides. Therefore, letting $u \coloneqq (-a_{j'} w_{j'}, a_{j'} w_j, w_{j'} b_j - w_j b_{j'})$, it suffices to show that:
\begin{itemize}
\item
$\mathrm{D}_{a_j, a_{j'}, b_{j'}}
 \bigl(\frac{1}{2}
       \sum_{j'' = 1}^m (a_{j''}^2 + w_{j''}^2)\bigr)
 [u] \leq 0$;
\item
for all $i \in [n]$ we have
\begin{align}
\lim_{x \uparrow x_i}
\mathrm{D}_{a_j, a_{j'}, b_{j'}}
\bigl(y_i \, f_{\theta, a_0, b_0}(x)\bigr)
[u] & \geq 0 \\
\lim_{x \downarrow x_i}
\mathrm{D}_{a_j, a_{j'}, b_{j'}}
\bigl(y_i \, f_{\theta, a_0, b_0}(x)\bigr)
[u] & \geq 0,
\end{align}
and there exists $i \in [n]$ for which both latter inequalities are strict.
\end{itemize}
For the first claim, we have
\[\mathrm{D}_{a_j, a_{j'}, b_{j'}}
  \biggl(\frac{1}{2}
         \sum_{j'' = 1}^m (a_{j''}^2 + w_{j''}^2)\biggr)
  [u]
=    a_{j'} (-a_j w_{j'} + a_{j'} w_j)
\leq 0.\]
For the second claim, since $y_i = -1$ for all $x_i \in [p_k, p_{k'}]$, and there exists $x_i \in (p_k, p_{k'})$, in turn it suffices to observe that:
\begin{itemize}
\item
for all $x \in (-\infty, p_k)$ we have
\begin{align}
& \mathrm{D}_{a_j, a_{j'}, b_{j'}}
  f_{\theta, a_0, b_0}(x)
  [u] \\
& = \mathrm{D}_{a_j, a_{j'}, b_{j'}}
    \bigl(a_j \, \sigma(w_j x + b_j)\bigr)
    [u]
  + \mathrm{D}_{a_j, a_{j'}, b_{j'}}
    \bigl(a_{j'} \, \sigma(w_{j'} x + b_{j'})\bigr)
    [u] \\
& = 0;
\end{align}
\item
for all $x \in (p_k, p_{k'})$ we have
\begin{align}
& \mathrm{D}_{a_j, a_{j'}, b_{j'}}
  f_{\theta, a_0, b_0}(x)
  [u] \\
& = \mathrm{D}_{a_j, a_{j'}, b_{j'}}
    \bigl(a_j \, \sigma(w_j x + b_j)\bigr)
    [u]
  + \mathrm{D}_{a_j, a_{j'}, b_{j'}}
    \bigl(a_{j'} \, \sigma(w_{j'} x + b_{j'})\bigr)
    [u] \\
& = -(w_j x + b_j) a_{j'} w_{j'} \\
& < 0;
\end{align}
\item
for all $x \in (p_{k'}, \infty)$ we have
\begin{align}
& \mathrm{D}_{a_j, a_{j'}, b_{j'}}
  f_{\theta, a_0, b_0}(x)
  [u] \\
& = \mathrm{D}_{a_j, a_{j'}, b_{j'}}
    \bigl(a_j \, \sigma(w_j x + b_j)\bigr)
    [u]
  + \mathrm{D}_{a_j, a_{j'}, b_{j'}}
    \bigl(a_{j'} \, \sigma(w_{j'} x + b_{j'})\bigr)
    [u] \\
& = - (w_j x + b_j) a_{j'} w_{j'}
    + (w_{j'} x + b_{j'}) a_{j'} w_j
    + a_{j'} (w_{j'} b_j - w_j b_{j'}) \\
& = 0.
\ifopt\else\qedhere\fi
\end{align}
\end{itemize}
\end{proof}

\begin{figure}[t]
\centering
\begin{tikzpicture}[xscale=\ifopt\else\ifarxiv .8\else .67\fi\fi,yscale=1.5]
\draw         (-5,  0) node[left]  {$0$} -- (3.75,  0);
\draw[very thick]
(-5   ,  1   ) --
(-1   , -2   ) --
( 0.33, -2   ) --
( 1.75, -1.25) --
( 3.75,  1.25);
\node[point,color=blue] at (-4.33, 0) {};
\node[point,color=red]  at (-3   , 0) {};
\draw[thin,dotted]         ( 0.33, 0) node[point,color=red] {} -- (0.33, -2);
\node[point,color=red]  at ( 2.33, 0) {};
\node[point,color=blue] at ( 3.25, 0) {};
\draw[very thick,dotted] (-1, -2) -- (0.33, -2.12) -- (1.45, -1.625);
\draw[->,orange,ultra thick] (1.75, -1.25) -- (1.45, -1.625);
\pic at (-1   , -2   ) {plus};
\pic at ( 1.75, -1.25) {plus};
\end{tikzpicture}
\hspace{1em}
\begin{tikzpicture}[xscale=\ifopt\else\ifarxiv .8\else .67\fi\fi,yscale=.75]
\draw         (-5,  0) node[left]  {$0$} -- (3.75,  0);
\draw[very thick]
(-5   ,  1.5 ) --
(-2   , -1.5 ) --
( 1   , -2.25) --
( 3.75,  3.25);
\draw[thin,dotted]
(-2   , -1.5 ) --
( 0.25, -3.75) --
( 1   , -2.25);
\node[point,color=blue] at (-4   , 0) {};
\draw[thin,dotted]         (-2   , 0) node[point,color=red] {} -- (-2, -1.5 );
\draw[thin,dotted]         ( 1   , 0) node[point,color=red] {} -- ( 1, -2.25);
\node[point,color=blue] at ( 2.75, 0) {};
\draw[violet!67,very thick,dashed]
(-5   ,  3.05 ) --
(-1   , -3.95 ) --
( 3.75,  1.75 );
\draw[->,orange,ultra thick] (-4   ,  0.5 ) -- (-4   ,  1.3 );
\draw[->,orange,ultra thick] (-2   , -1.5 ) -- (-2   , -2.2 );
\draw[->,orange,ultra thick] ( 1   , -2.25) -- ( 1   , -1.55);
\draw[->,orange,ultra thick] ( 2.75,  1.25) -- ( 2.75,  0.55);
\pic at (-2, -1.5 ) {plus};
\pic at ( 1, -2.25) {plus};
\end{tikzpicture}
\caption{Since the focus in this section is on the regularized losses, we indicate the labels of the inputs in the dataset by the colors blue for~$1$, and red for~$-1$.  \textbf{Left:}  Here we illustrate the proof of \cref{l:L1} for the case of property~\ref{l:it:conv} in \cref{l:it}; in this example, one of the kinks of the network function happens to coincide with an input in the dataset.  The kinks $p_k$ and~$p_{k'}$ in the proof are indicated by the green pluses.  The proof identifies a direction for infinitesimally perturbing the network parameters so that the portion of the network function between $p_k$ and~$p_{k'}$ moves as indicated by the orange arrow and the dotted half line.  \textbf{Right:}  Here we illustrate the proof of \cref{l:so} for the case when the biases are not regularized, i.e., for the loss~\hyperref[eq:L1]{$L\mathrm{^\lambda_1}$}; in this example, two of the kinks of the network function happen to coincide with inputs in the dataset.  The two kinks are also indicated by the green pluses because they are the first and last in the intermediate dataset segment between the two red inputs, and so two neurons that contribute to those kinks are perturbed infinitesimally in the proof.  The perturbations across all the intermediate dataset segments (and of the skip connection) work together to move (as indicated by the orange arrows) all pieces of the network function that cross the $y = 0$ line toward the corresponding ones of the switch optimal function~$g_{\tau^\star, \pi^\star}$, which is plotted in violet and dashed.}
\label{f:it:conv.so}
\end{figure}
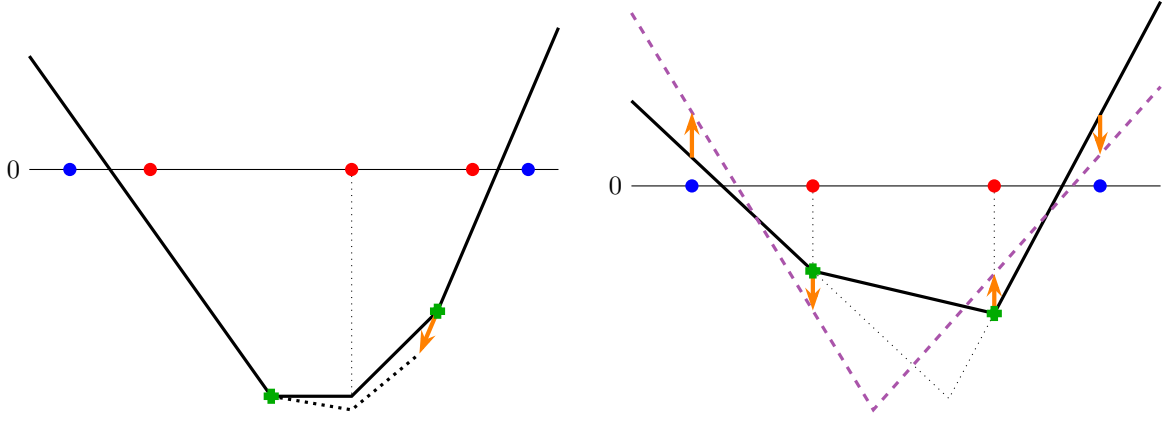

The final lemma has technically the most involved proof.  At its center are constructions, based on the failure of switch optimality, of directions for the network parameters along which the Clarke derivative of the regularized loss is negative.  In contrast to the infinitesimal perturbations that we have seen so far in this work, which are local, here we perturb simultaneously one or two neurons per each intermediate same-label segment of the dataset (and also the skip connection).  The directions are constructed so that, when they are mapped to the pairs of margins incident to all label switches, they point to the unique minimizer of the corresponding optimization problem considered in \cref{l:O1.Ob1}.  For an illustration of a part of the proof, we refer the reader to \Cref{f:it:conv.so} right.

\begin{lem}
\label{l:so}
Suppose the dataset has at least two label switches, i.e., $r \geq 2$, and suppose $\lambda \leq d_{\min} / 24 (r - 1) n (\log n)$ (resp., $\lambda \leq d_{\min} / 24 (r - 1) n (\log n) \sqrt{1 + x_{\max}^2}$).  If $f_{\theta, a_0, b_0}$~has a positive margin, and it is convexity correct and interval turning (resp., single turning), but it is not switch optimal, then $\theta, a_0, b_0$ is not a Clarke stationary point of~\hyperref[eq:L1]{$L\mathrm{^\lambda_1}$} (resp.,~\hyperref[eq:Lb1]{$L\mathrm{^{\lambda, b}_1}$}).
\end{lem}

\begin{proof}
For all $k \in [r]$, let
\begin{align}
\tau_k & \coloneqq y_{\iota(k)}     \, f_{\theta, a_0, b_0}(x_{\iota(k)}) &
\pi_k  & \coloneqq y_{\iota(k) + 1} \, f_{\theta, a_0, b_0}(x_{\iota(k) + 1}).
\end{align}
Since $f_{\theta, a_0, b_0}$~has a positive margin, and it is convexity correct and interval turning (resp., single turning), we have that $g_{\tau, \pi}$~coincides with~$f_{\theta, a_0, b_0}$ except possibly on the open intervals between the first and last kinks of~$f_{\theta, a_0, b_0}$ in intermediate same-label segments of the dataset.  Hence $\tau, \pi$ is feasible for \cref{eq:O1} (resp., \cref{eq:Ob1}) and
\[f_{\theta, a_0, b_0}(x_i) =
  g_{\tau, \pi}(x_i)
  \quad\text{for all } i \in [n].\]

But $f_{\theta, a_0, b_0}$~is not switch optimal, so recalling \cref{l:O1.Ob1} and letting $\tau^\star, \pi^\star$ denote the unique minimizer of \cref{eq:O1} (resp., \cref{eq:Ob1}), we have that $(\tau', \pi') \coloneqq (\tau^\star - \tau, \pi^\star - \pi) \in \mathbb{R}^{2 r}$ is a nonzero vector.

Some more notation that is going to be useful in the remainder of the proof is the following.  For each $k \in [r]$, let us denote the slope and the bias of the $k$-th piece of~$g_{\tau, \pi}$ as
\begin{align}
\alpha_{\tau, \pi}^{(k)} & \coloneqq
-y_{\iota(k)}
\frac{\tau_k + \pi_k}
     {x_{\iota(k) + 1} - x_{\iota(k)}} &
\beta_{\tau, \pi}^{(k)} & \coloneqq
y_{\iota(k)}
\frac{x_{\iota(k) + 1} \tau_k + x_{\iota(k)} \pi_k}
     {x_{\iota(k) + 1} - x_{\iota(k)}},
\end{align}
and for each $k \in [r - 1]$, let us denote the $k$-th kink of~$g_{\tau, \pi}$ as
\[\gamma_{\tau, \pi}^{(k)} \coloneqq
  -\frac{\beta_{\tau, \pi}^{(k + 1)} - \beta_{\tau, \pi}^{(k)}}
        {\alpha_{\tau, \pi}^{(k + 1)} - \alpha_{\tau, \pi}^{(k)}}\]
(which is necessarily either in the open interval between the first and last kinks of~$f_{\theta, a_0, b_0}$ in the $k$-th intermediate same-label segment, or it is the only kink of~$f_{\theta, a_0, b_0}$ in that segment).  Also, for each $i \in [n]$, let $\widehat{k}(i) \in [r]$ be such that $x_i$~is in the left-closed right-open domain of the $\widehat{k}(i)$-th piece of~$g_{\tau, \pi}$.

To proceed, we consider the losses~\hyperref[eq:L1]{$L\mathrm{^\lambda_1}$} and~\hyperref[eq:Lb1]{$L\mathrm{^{\lambda, b}_1}$} as two cases.

\subparagraph{Case \hyperref[eq:L1]{$L\mathrm{^\lambda_1}$}.}

First we define a direction $\theta' = (a'_j, w'_j, b'_j)_{j = 1}^m, a'_0, b'_0$ such that perturbing the network along it has the same effect on~$f_{\theta, a_0, b_0}$ outside the open intervals between the first and last kinks in intermediate same-label segments as perturbing the vector $\tau, \pi$ of margins incident to the label switches in the direction $\tau', \pi'$ has on~$g_{\tau, \pi}$.  To do this, for each $k \in [r - 1]$, we identify as follows either two or one neurons to perturb whose kinks are in the $k$-th intermediate same-label segment; all other neurons are left fixed.

To begin the definition, we set $\theta'_{(0)} \coloneqq (0, 0, 0)_{j = 1}^m$ and
\begin{align}
a'_0 & \coloneqq
\mathrm{D}
\bigl(\alpha_{\tau, \pi}^{(1)}\bigr)
[\tau', \pi'] &
b'_0 & \coloneqq
\mathrm{D}
\bigl(\beta_{\tau, \pi}^{(1)}\bigr)
[\tau', \pi'].
\end{align}
The rest of the definition is inductive, where we consider $k \in [r - 1]$ and assume that, for all $i \in [n]$, we have
\begin{equation}
\mathrm{D}^\circ
\bigl(-y_i \, f_{\theta, a_0, b_0}(x_i)\bigr)
[\theta'_{(k - 1)}, a'_0, b'_0] =
\begin{cases}
\mathrm{D}
\bigl(-y_i \, g_{\tau, \pi}(x_i)\bigr)
[\tau', \pi']
& \text{if } \widehat{k}(i) \leq k, \\[.5ex]
\mathrm{D}
\Bigl(-y_i \bigl(\alpha_{\tau, \pi}^{(k)} x_i + \beta_{\tau, \pi}^{(k)}\bigr)\Bigr)
[\tau', \pi']
& \text{if } \widehat{k}(i) > k.
\end{cases}
\label{eq:i.h}
\end{equation}
The values of the skip connection parameters $a'_0$~and~$b'_0$ above ensure that this inductive hypothesis is satisfied initially, i.e., for $k = 1$.

Let~$p^\dag_k$ and~$p^\ddag_k$ denote respectively the first and last kinks of~$f_{\theta, a_0, b_0}$ in the $k$-th intermediate same-label segment.  We distinguish two subcases depending on whether these coincide.

\subparagraph{Subcase $p^\dag_k < p^\ddag_k$.}

This subcase is the more complex one, and it involves identifying two neurons to perturb, whose kinks are at $p^\dag_k$~and~$p^\ddag_k$.  By $f_{\theta, a_0, b_0}$~having a positive margin and being convexity correct, and by \cref{l:rev,l:bal} and \cref{pr:dd.not.Clarke}, there exist $j^\dag_k, j^\ddag_k \in [m]$ for which we may assume:
\begin{itemize}
\item
$a_{j^\dag_k} > 0$ and $a_{j^\ddag_k} > 0$ (resp., $a_{j^\dag_k} < 0$ and $a_{j^\ddag_k} < 0$) if the $k$-th intermediate dataset segment has label~$-1$ (resp.,~$1$);
\item
$w_{j^\dag_k} = \Bigl\lvert a_{j^\dag_k} \Bigr\rvert$ and $w_{j^\ddag_k} = \Bigl\lvert a_{j^\ddag_k} \Bigr\rvert$;
\item
$-b_{j^\dag_k} \big/ w_{j^\dag_k} = p^\dag_k$ and $-b_{j^\ddag_k} \big/ w_{j^\ddag_k} = p^\ddag_k$.
\end{itemize}
Now let
\begin{align}
\delta^\dag_k & \coloneqq
\frac{\mathrm{D}
      \Bigl(
      \bigl(\alpha_{\tau, \pi}^{(k + 1)} - \alpha_{\tau, \pi}^{(k)}\bigr)
      p^\ddag_k
    + \bigl(\beta_{\tau, \pi}^{(k + 1)} - \beta_{\tau, \pi}^{(k)}\bigr)
      \Bigr)
      [\tau', \pi']}
     {p^\ddag_k - p^\dag_k} \\
\delta^\ddag_k & \coloneqq
-\frac{\mathrm{D}
       \Bigl(
       \bigl(\alpha_{\tau, \pi}^{(k + 1)} - \alpha_{\tau, \pi}^{(k)}\bigr)
       p^\dag_k
     + \bigl(\beta_{\tau, \pi}^{(k + 1)} - \beta_{\tau, \pi}^{(k)}\bigr)
       \Bigr)
       [\tau', \pi']}
      {p^\ddag_k - p^\dag_k},
\end{align}
and let $\theta'_{(k)}$~be obtained from~$\theta'_{(k - 1)}$ by setting
\begin{align}
a'_{j^\dag_k} & \coloneqq
\frac{\delta^\dag_k}
     {2 w_{j^\dag_k}} &
w'_{j^\dag_k} & \coloneqq
\frac{\delta^\dag_k}
     {2 a_{j^\dag_k}} &
b'_{j^\dag_k} & \coloneqq
-\frac{\delta^\dag_k p^\dag_k}
      {2 a_{j^\dag_k}} \\
a'_{j^\ddag_k} & \coloneqq
\frac{\delta^\ddag_k}
     {2 w_{j^\ddag_k}} &
w'_{j^\ddag_k} & \coloneqq
\frac{\delta^\ddag_k}
     {2 a_{j^\ddag_k}} &
b'_{j^\ddag_k} & \coloneqq
-\frac{\delta^\ddag_k p^\ddag_k}
      {2 a_{j^\ddag_k}}.
\end{align}
These perturbations to the two newly identified neurons $j^\dag_k$~and~$j^\ddag_k$ have the following contributions to the directional derivative of $f_{\theta, a_0, b_0}(x)$: for $x \in (-\infty, p^\dag_k)$, zero; for $x \in (p^\ddag_k, \infty)$, subtracting
$\mathrm{D}
 \bigl(\alpha_{\tau, \pi}^{(k)} x + \beta_{\tau, \pi}^{(k)}\bigr)
 [\tau', \pi']$
and adding
$\mathrm{D}
 \bigl(\alpha_{\tau, \pi}^{(k + 1)} x + \beta_{\tau, \pi}^{(k + 1)}\bigr)
 [\tau', \pi']$;
and for $x \in (p^\dag_k, p^\ddag_k)$, the affine interpolation between the former and latter.  Moreover, our definitions give
\[w'_{j^\dag_k}  p^\dag_k  + b'_{j^\dag_k}  = 0
  \quad\text{and}\quad
  w'_{j^\ddag_k} p^\ddag_k + b'_{j^\ddag_k} = 0,\]
so the two spatial one-sided limits of the output directional derivative agree at each of the two selected kinks.  Hence, since there is no training input in $(p^\dag_k, p^\ddag_k)$ because $f_{\theta, a_0, b_0}$~is interval turning, and by applying \cref{l:up.down} if either of the two selected kinks is a training input, it follows that the inductive hypothesis \cref{eq:i.h} is satisfied for $k + 1$ and all $i \in [n]$.

\subparagraph{Subcase $p^\dag_k = p^\ddag_k$.}

Here $f_{\theta, a_0, b_0}$~has only this one kink in the $k$-th intermediate same-label segment, which means that it equals~$g_{\tau, \pi}$ everywhere on that segment.  By $f_{\theta, a_0, b_0}$~having a positive margin and being convexity correct, and by \cref{l:rev,l:bal} and \cref{pr:dd.not.Clarke}, there exists $j^\dag_k \in [m]$ for which we may assume:
\begin{itemize}
\item
$a_{j^\dag_k} > 0$ (resp., $a_{j^\dag_k} < 0$) if the $k$-th intermediate dataset segment has label~$-1$ (resp.,~$1$);
\item
$w_{j^\dag_k} = \Bigl\lvert a_{j^\dag_k} \Bigr\rvert$;
\item
$-b_{j^\dag_k} \big/ w_{j^\dag_k} = p^\dag_k$.
\end{itemize}
Now let $\theta'_{(k)}$~be obtained from~$\theta'_{(k - 1)}$ by setting
\begin{align}
a'_{j^\dag_k} & \coloneqq
\frac{\mathrm{D}
      \bigl(\alpha_{\tau, \pi}^{(k + 1)} - \alpha_{\tau, \pi}^{(k)}\bigr)
      [\tau', \pi']}
     {2 w_{j^\dag_k}} \\
w'_{j^\dag_k} & \coloneqq
\frac{\mathrm{D}
      \bigl(\alpha_{\tau, \pi}^{(k + 1)} - \alpha_{\tau, \pi}^{(k)}\bigr)
      [\tau', \pi']}
     {2 a_{j^\dag_k}} \\
b'_{j^\dag_k} & \coloneqq
\frac{\mathrm{D}
      \bigl(\alpha_{\tau, \pi}^{(k + 1)} - \alpha_{\tau, \pi}^{(k)}\bigr)
      [\tau', \pi'] \,
      p^\dag_k}
     {2 a_{j^\dag_k}} +
\frac{\mathrm{D}
      \bigl(\beta_{\tau, \pi}^{(k + 1)} - \beta_{\tau, \pi}^{(k)}\bigr)
      [\tau', \pi']}
     {a_{j^\dag_k}}.
\end{align}
This perturbation to the single newly identified neuron~$j^\dag_k$ does not change the directional derivative of $f_{\theta, a_0, b_0}(x)$ for $x < p^\dag_k$, whereas for $x > p^\dag_k$ it subtracts
$\mathrm{D}
 \bigl(\alpha_{\tau, \pi}^{(k)} x + \beta_{\tau, \pi}^{(k)}\bigr)
 [\tau', \pi']$
and adds
$\mathrm{D}
 \bigl(\alpha_{\tau, \pi}^{(k + 1)} x + \beta_{\tau, \pi}^{(k + 1)}\bigr)
 [\tau', \pi']$.
Hence, and by \cref{l:up.down}, if $p^\dag_k$~equals a training input $x_i$, then $\mathrm{D}^\circ
 \bigl(-y_i \, f_{\theta, a_0, b_0}(x_i)\bigr)
 [\theta'_{(k)}, a'_0, b'_0]$
and
$\mathrm{D}
 \bigl(-y_i \, g_{\tau, \pi}(x_i)\bigr)
 [\tau', \pi']$
both equal
\[\max
  \Bigl\{
  \mathrm{D}
  \Bigl(-y_i \bigl(\alpha_{\tau, \pi}^{(k)} x_i
                 + \beta _{\tau, \pi}^{(k)}\bigr)\Bigr)
  [\tau', \pi'],
  \mathrm{D}
  \Bigl(-y_i \bigl(\alpha_{\tau, \pi}^{(k + 1)} x_i
                 + \beta _{\tau, \pi}^{(k + 1)}\bigr)\Bigr)
  [\tau', \pi']
  \Bigr\},\]
and so the inductive hypothesis \cref{eq:i.h} is satisfied for $k + 1$ and all $i \in [n]$.

\vspace{2ex}
Having completed the two subcases and the induction, letting $\theta' \coloneqq \theta'_{(r - 1)}$ we have that
\[\mathrm{D}^\circ_{\theta, a_0, b_0}
  \bigl(-y_i \, f_{\theta, a_0, b_0}(x_i)\bigr)
  [\theta', a'_0, b'_0] =
  \mathrm{D}_{\tau, \pi}
  \bigl(-y_i \, g_{\tau, \pi}(x_i)\bigr)
  [\tau', \pi']
  \quad \text{for all } i \in [n],\]
and hence also
\[\mathrm{D}^\circ_{\theta, a_0, b_0}
  \ell\bigl(y_i \, f_{\theta, a_0, b_0}(x_i)\bigr)
  [\theta', a'_0, b'_0] =
  \mathrm{D}_{\tau, \pi}
  \ell\bigl(y_i \, g_{\tau, \pi}(x_i)\bigr)
  [\tau', \pi']
  \quad \text{for all } i \in [n].\]

Also observe that
\begin{align}
& \mathrm{D}_{\theta, a_0, b_0}
  \biggl(\frac{1}{2}
         \sum_{j = 1}^m (a_j^2 + w_j^2)\biggr)
  [\theta', a'_0, b'_0] \\
& = \sum_{k \mid p^\dag_k < p^\ddag_k}
    \Bigl(a_{j^\dag_k}  a'_{j^\dag_k}  + w_{j^\dag_k}  w'_{j^\dag_k}
        + a_{j^\ddag_k} a'_{j^\ddag_k} + w_{j^\ddag_k} w'_{j^\ddag_k}\Bigr)
  + \sum_{k \mid p^\dag_k = p^\ddag_k}
    \Bigl(a_{j^\dag_k} a'_{j^\dag_k} + w_{j^\dag_k} w'_{j^\dag_k}\Bigr) \\
& = \sum_{k = 1}^{r - 1}
    \sgn\Bigl(a_{j^\dag_k}\Bigr)
    \mathrm{D}_{\tau, \pi}
    \bigl(\alpha_{\tau, \pi}^{(k + 1)} - \alpha_{\tau, \pi}^{(k)}\bigr)
    [\tau', \pi'] \\
& = \mathrm{D}_{\tau, \pi}
    \sum_{k = 1}^{r - 1}
    \Biggl(\frac{\tau_k + \pi_k}
                {x_{\iota(k) + 1} - x_{\iota(k)}}
         + \frac{\tau_{k + 1} + \pi_{k + 1}}
                {x_{\iota(k + 1) + 1} - x_{\iota(k + 1)}}\Biggr)
    [\tau', \pi'] \\
& = \mathrm{D}_{\tau, \pi}
    \mathrm{R_1}(g_{\tau, \pi})
    [\tau', \pi'].
\end{align}
Therefore, we have
\[\mathrm{D}^\circ
  L\mathrm{^\lambda_1}(\theta, a_0, b_0)
  [\theta', a'_0, b'_0]
  = \mathrm{D}_{\tau, \pi}
    \biggl(
    \frac{1}{n}
    \sum_{i = 1}^n
    \ell\bigl(y_i \, g_{\tau, \pi}(x_i)\bigr) +
    \lambda \,
    \mathrm{R_1}(g_{\tau, \pi})
    \biggr)
    [\tau', \pi']
  < 0,\]
where the latter inequality is by the strict convexity of the objective and the convexity of the feasible set of \cref{eq:O1} (cf.\ \ref{l:O1.Ob1:O1}).  Thus $\theta, a_0, b_0$ is not Clarke stationary, as required.

\subparagraph{Case \hyperref[eq:Lb1]{$L\mathrm{^{\lambda, b}_1}$}.}

By $f_{\theta, a_0, b_0}$~having a positive margin, and being convexity correct and single turning, and by \cref{l:rev,l:bal.b} and \cref{pr:dd.not.Clarke}, it follows that for all $k \in [r - 1]$ there exists $j^\dag_k \in [m]$ such that, without loss of generality:
\begin{itemize}
\item
$a_{j^\dag_k} > 0$ (resp., $a_{j^\dag_k} < 0$) if the $k$-th intermediate dataset segment has label~$-1$ (resp.,~$1$);
\item
$w_{j^\dag_k} > 0$;
\item
$a_{j^\dag_k}^2 = w_{j^\dag_k}^2 + b_{j^\dag_k}^2$;
\item
$-b_{j^\dag_k} \big/ w_{j^\dag_k} = \gamma_{\tau, \pi}^{(k)}$.
\end{itemize}

Next, for each $k \in [r - 1]$, we define the perturbation of neuron~$j^\dag_k$ exactly as in subcase $p^\dag_k = p^\ddag_k$ above, and define the skip-connection perturbation as before.  We then argue like in the former case to conclude that $\theta, a_0, b_0$ is not Clarke stationary, except that we show the equality of the directional derivatives of the unscaled regularizers as follows:
\begin{align}
& \mathrm{D}_{\theta, a_0, b_0}
  \biggl(\frac{1}{2}
         \sum_{j = 1}^m (a_j^2 + w_j^2 + b_j^2)\biggr)
  [\theta', a'_0, b'_0] \\
& = \sum_{k = 1}^{r - 1}
    \Bigl(a_{j^\dag_k} a'_{j^\dag_k}
        + w_{j^\dag_k} w'_{j^\dag_k}
        + b_{j^\dag_k} b'_{j^\dag_k}\Bigr) \\
& = \sum_{k = 1}^{r - 1}
    \Biggl[
    \Biggl(\frac{a_{j^\dag_k}}{2 w_{j^\dag_k}}
         + \frac{w_{j^\dag_k}}{2 a_{j^\dag_k}}
         - \frac{b_{j^\dag_k}^2}{2 a_{j^\dag_k} w_{j^\dag_k}}\Biggr)
    \mathrm{D}
    \bigl(\alpha_{\tau, \pi}^{(k + 1)} - \alpha_{\tau, \pi}^{(k)}\bigr)
    [\tau', \pi']
  + \frac{b_{j^\dag_k}}{a_{j^\dag_k}} \,
    \mathrm{D}
    \bigl(\beta_{\tau, \pi}^{(k + 1)} - \beta_{\tau, \pi}^{(k)}\bigr)
    [\tau', \pi']
    \Biggr] \\
& = \sum_{k = 1}^{r - 1}
    \Biggl[
    \frac{w_{j^\dag_k}}{a_{j^\dag_k}} \,
    \mathrm{D}
    \bigl(\alpha_{\tau, \pi}^{(k + 1)} - \alpha_{\tau, \pi}^{(k)}\bigr)
    [\tau', \pi']
  + \frac{b_{j^\dag_k}}{a_{j^\dag_k}} \,
    \mathrm{D}
    \bigl(\beta_{\tau, \pi}^{(k + 1)} - \beta_{\tau, \pi}^{(k)}\bigr)
    [\tau', \pi']
    \Biggr] \\
& = \sum_{k = 1}^{r - 1}
    \frac{\mathrm{D}
          \bigl(\alpha_{\tau, \pi}^{(k + 1)} - \alpha_{\tau, \pi}^{(k)}\bigr)
          [\tau', \pi']
        - \gamma_{\tau, \pi}^{(k)} \,
          \mathrm{D}
          \bigl(\beta_{\tau, \pi}^{(k + 1)} - \beta_{\tau, \pi}^{(k)}\bigr)
          [\tau', \pi']}
         {\sgn\Bigl(a_{j^\dag_k}\Bigr)
          \sqrt{1 + \bigl(\gamma_{\tau, \pi}^{(k)}\bigr)^2}} \\
& = \sum_{k = 1}^{r - 1}
    \frac{\bigl(\alpha_{\tau, \pi}^{(k + 1)}
              - \alpha_{\tau, \pi}^{(k)}\bigr)
          \mathrm{D}
          \bigl(\alpha_{\tau, \pi}^{(k + 1)} - \alpha_{\tau, \pi}^{(k)}\bigr)
          [\tau', \pi']
        + \bigl(\beta_{\tau, \pi}^{(k + 1)}
              - \beta_{\tau, \pi}^{(k)}\bigr)
          \mathrm{D}
          \bigl(\beta_{\tau, \pi}^{(k + 1)} - \beta_{\tau, \pi}^{(k)}\bigr)
          [\tau', \pi']}
         {\sqrt{\bigl(\alpha_{\tau, \pi}^{(k + 1)}
                    - \alpha_{\tau, \pi}^{(k)}\bigr)^2
              + \bigl(\beta_{\tau, \pi}^{(k + 1)}
                    - \beta_{\tau, \pi}^{(k)}\bigr)^2}} \\
& = \mathrm{D}
    \Biggl(
    \sum_{k = 1}^{r - 1}
    \sqrt{\bigl(\alpha_{\tau, \pi}^{(k + 1)}
              - \alpha_{\tau, \pi}^{(k)}\bigr)^2
        + \bigl(\beta_{\tau, \pi}^{(k + 1)}
              - \beta_{\tau, \pi}^{(k)}\bigr)^2}
    \Biggr)
    [\tau', \pi'] \\
& = \mathrm{D}_{\tau, \pi}
    \mathrm{R^b_1}(g_{\tau, \pi})
    [\tau', \pi'].
\ifopt\else\qedhere\fi
\end{align}
\end{proof}

\begin{proof}%
\ifopt
\else
[Proof of \cref{th:min.reg.loss}]
\fi
For part~\ref{th:min.reg.loss:L.L1} (resp., part~\ref{th:min.reg.loss:Lb.Lb1}), suppose $f_{\theta, a_0, b_0}$~is a minimizer of~\hyperref[eq:L1]{$L\mathrm{^\lambda_1}$} (resp.,~\hyperref[eq:Lb1]{$L\mathrm{^{\lambda, b}_1}$}) in function space.  Then $\theta, a_0, b_0$ is a Clarke stationary point, and its margin is positive by \cref{c:pos.margin}.  By \cref{pr:dd.not.Clarke}, \cref{l:cc,l:all.P1}, and \cref{l:it,l:L1} (resp., \cref{l:st,l:all.Pb1}), $f_{\theta, a_0, b_0}$~is convexity correct and interval turning (resp., single turning), and so by \cref{l:so}, $f_{\theta, a_0, b_0}$~is also switch optimal.

Conversely, suppose $h$~is switch optimal, convexity correct, and interval turning (resp., single turning), but not a minimizer of~\hyperref[eq:L1]{$L\mathrm{^\lambda_1}$} (resp.,~\hyperref[eq:Lb1]{$L\mathrm{^{\lambda, b}_1}$}) in function space.  By \cref{c:pos.margin,l:attain.r.l}, for some $m \in \mathbb{N}$ and some $\theta, a_0, b_0$ which is a minimizer of~\hyperref[eq:L1]{$L\mathrm{^\lambda_1}$} (resp.,~\hyperref[eq:Lb1]{$L\mathrm{^{\lambda, b}_1}$}) restricted to network width~$m$ and whose margin is positive, we have that
$L\mathrm{^\lambda_1}(\theta, a_0, b_0) <
 \frac{1}{n}
 \sum_{i = 1}^n
 \ell\bigl(y_i \, h(x_i)\bigr) +
 \lambda \,
 \mathrm{R_1}(h)$
(resp.,
$L\mathrm{^{\lambda, b}_1}(\theta, a_0, b_0) <
 \frac{1}{n}
 \sum_{i = 1}^n
 \ell\bigl(y_i \, h(x_i)\bigr) +
 \lambda \,
 \mathrm{R^b_1}(h)$).
Reasoning like above, $f_{\theta, a_0, b_0}$~is switch optimal, convexity correct, and interval turning (resp., single turning).  But then we infer that $f_{\theta, a_0, b_0}$~and~$h$ have equal values on every input $x_i$, and have equal $\mathrm{R_1}$~(resp.,~$\mathrm{R^b_1}$) norms, which is a contradiction.

Since switch optimality and convexity correctness uniquely fix all crossing pieces, a straightforward positive-margin adaptation of the proof of \cref{l:unique.st} shows that the single-turning minimizer is unique in function space.

The remainders of parts~\ref{th:min.reg.loss:L.L1} and~\ref{th:min.reg.loss:Lb.Lb1}, which are for the losses~\hyperref[eq:L]{$L\mathrm{^\lambda}$} and~\hyperref[eq:Lb]{$L\mathrm{^{\lambda, b}}$} (respectively), now follow by \cref{l:RR1.RbRb1}.

For part~\ref{th:min.reg.loss:stat}, suppose $(\theta, a_0, b_0) \in \mathbb{R}^{3 m + 2}$ is a Clarke stationary point of~\hyperref[eq:L1]{$L\mathrm{^\lambda_1}$} (resp.,~\hyperref[eq:Lb1]{$L\mathrm{^{\lambda, b}_1}$}) and its margin is positive.  Reasoning like above, $f_{\theta, a_0, b_0}$~is switch optimal, convexity correct, and interval turning (resp., single turning).  Moreover, by \cref{th:repr} and \cref{l:bal,l:opp} (resp., \cref{l:bal.b,,l:opp,,l:bias}), we have that $\frac{1}{2} \sum_{j = 1}^m (a_j^2 + w_j^2) = \mathrm{R_1}(f_{\theta, a_0, b_0})$ (resp., $\frac{1}{2} \sum_{j = 1}^m (a_j^2 + w_j^2 + b_j^2) = \mathrm{R^b_1}(f_{\theta, a_0, b_0})$), so by part~\ref{th:min.reg.loss:L.L1} (resp.,~\ref{th:min.reg.loss:Lb.Lb1}), we conclude that $\theta, a_0, b_0$ is a minimizer.
\end{proof}

\clearpage
\section{Proofs for limits of normalized approximate Clarke stationary points}

Here we state and prove our final main theoretical result which establishes that, for each of the four variants of the $\ell_2$-regularized logistic loss we consider, limits of margin-normalized approximate Clarke stationary points as the regularization strength~$\lambda$ tends to~$0$ are KKT points of the corresponding interpolator norm minimization problem, provided the approximation tightness~$\zeta$ tends to~$0$ sufficiently fast.  It also shows that, when the biases are not penalized and the skip connection is present, and again provided $\zeta$~tends to~$0$ sufficiently fast, the limits are necessarily interval turning in function space; therefore, recalling \cref{th:inter}, most KKT points cannot be obtained in this way.  Before stating and proving the theorem, we define margin normalization, show a simple proposition about it, and define approximate Clarke stationary points.

Provided the margin involved is positive, let
\begin{align}
\widetilde{\theta} & \coloneqq
\frac{\theta}{\sqrt{M(\theta)}} &
\widetilde{\theta, a_0, b_0} & \coloneqq
\frac{\theta}{\sqrt{M(\theta, a_0, b_0)}},
\frac{a_0}{M(\theta, a_0, b_0)},
\frac{b_0}{M(\theta, a_0, b_0)}
\end{align}
denote scalings of the parameters that normalize the margin.

\begin{prop}
\label{pr:margin}
If $M(\theta, a_0, b_0) > 0$ then:
\[f_{\widetilde{\theta, a_0, b_0}}(x) =
  \frac{f_{\theta, a_0, b_0}(x)}{M(\theta, a_0, b_0)}
  \quad \text{for all } x \in \mathbb{R},
  \qquad \text{and} \qquad
  M\bigl(\widetilde{\theta, a_0, b_0}\bigr) = 1.\]
\end{prop}

\begin{proof}
This is straightforward to check using the definition of the network function in \cref{eq:f}.
\end{proof}

For locally Lipschitz $F \colon \mathbb{R}^N \to \mathbb{R}$ and $\zeta > 0$, we say that $v \in \mathbb{R}^N$ is a $\zeta$-stationary point of~$F$ if and only if
$\zeta \mathbb{B}_2^N \cap
 \partial F(v)
 \neq \emptyset$.

\begin{thm}
\label{th:lim}
Fix a network width~$m$, and suppose $\lambda_k, \zeta_k > 0$ for all $k \in \mathbb{N}$ and $\lambda_k \to 0$ as $k \to \infty$.
\begin{enumerate}[(i)]
\item
\label{th:lim:no.skip}
If $M(\theta^{(k)}) > 0$ and $\theta^{(k)}$~is a $\zeta_k$-stationary point of~$L\mathrm{^{\lambda_k}}$ (resp., $L\mathrm{^{\lambda_k, b}}$) for all $k \in \mathbb{N}$ and $\zeta_k / \bigl(\lambda_k \sqrt{M(\theta^{(k)})}\bigr) \to 0$ as $k \to \infty$, then every limit point of $\bigl\{\widetilde{\theta^{(k)}} \bigm| k \in \mathbb{N}\bigr\}$ is a KKT point of \cref{eq:P} (resp., \cref{eq:Pb}).
\item
\label{th:lim:yes.skip}
If $M\bigl(\theta^{(k)}, a_0^{(k)}, b_0^{(k)}\bigr) \geq 1$ and $\bigl(\theta^{(k)}, a_0^{(k)}, b_0^{(k)}\bigr)$ is a $\zeta_k$-stationary point of~$L\mathrm{^{\lambda_k}_1}$ (resp., $L\mathrm{^{\lambda_k, b}_1}$) for all $k \in \mathbb{N}$ and $\zeta_k / \lambda_k \to 0$ as $k \to \infty$, then every limit point of $\Bigl\{\bigl(\widetilde{\theta^{(k)}, a_0^{(k)}, b_0^{(k)}}\bigr) \Bigm| k \in \mathbb{N}\Bigr\}$ is a KKT point of \cref{eq:P1} (resp., \cref{eq:Pb1}).
\item
\label{th:lim:it}
In part~\ref{th:lim:yes.skip} when biases are not penalized, i.e., for loss~$L\mathrm{^{\lambda_k}_1}$ and \cref{eq:P1}, if \ifarxiv also\else moreover\fi\ $\zeta_k / \Bigl(\lambda_k^F \sqrt{M\bigl(\theta^{(k)}, a_0^{(k)}, b_0^{(k)}\bigr)}\Bigr) \to 0$ as $k \to \infty$ where $F > 1$ depends only on the dataset,%
\footnote{It suffices that, for each $k \in [r - 1]$, we have that $F > 1 + y_{\iota(k)} (x_{\iota(k + 1)} - x_{\iota(k) + 1}) s_{k - 1} s_k / (s_{k - 1} - s_k)$, where $s_{k - 1} \coloneqq 2 y_{\iota(k) + 1} / (x_{\iota(k) + 1} - x_{\iota(k)})$.}
then every limit point of $\Bigl\{\bigl(\widetilde{\theta^{(k)}, a_0^{(k)}, b_0^{(k)}}\bigr) \Bigm| k \in \mathbb{N}\Bigr\}$ is interval turning in function space.
\end{enumerate}
\end{thm}

\begin{proof}
For parts \ref{th:lim:no.skip} and~\ref{th:lim:yes.skip}, we show the claim for \cref{eq:Pb1}, the cases of \cref{eq:P,,eq:P1,,eq:Pb} can be handled analogously.

By considering a convergent subsequence, we may assume that $\bigl(\widetilde{\theta^{(k)}, a_0^{(k)}, b_0^{(k)}}\bigr) \to (\theta^\star, a_0^\star, b_0^\star)$ as $k \to \infty$.  Since each $\widetilde{\theta^{(k)}, a_0^{(k)}, b_0^{(k)}}$ is a feasible point of \cref{eq:Pb1} by \cref{pr:margin} and the feasible set is closed, we have that $\theta^\star, a_0^\star, b_0^\star$ is feasible and thus satisfies MFCQ by \cref{l:feas.MFCQ}.

Since each $\theta^{(k)}, a_0^{(k)}, b_0^{(k)}$ is a $\zeta_k$-stationary point of~\hyperref[eq:Lb1]{$L\mathrm{^{\lambda_k, b}_1}$}, we have
\begin{align}
\emptyset
& \neq \zeta_k \mathbb{B}_2^{3 m + 2}
  \cap \partial L\mathrm{^{\lambda_k, b}_1}
                \bigl(\theta^{(k)}, a_0^{(k)}, b_0^{(k)}\bigr) \\
& =    \zeta_k \mathbb{B}_2^{3 m + 2}
  \cap \Bigl(\partial L\bigl(\theta^{(k)}, a_0^{(k)}, b_0^{(k)}\bigr)
           + \frac{\lambda_k}{2}
             \partial \lVert \theta^{(k)} \rVert_2^2\Bigr) \\
& =    \zeta_k \mathbb{B}_2^{3 m + 2}
  \cap \biggl(\Bigl\{\Bigl(u, \nabla_{a_0}
                              L\bigl(\theta^{(k)}, a_0^{(k)}, b_0^{(k)}\bigr),
                              \nabla_{b_0}
                              L\bigl(\theta^{(k)}, a_0^{(k)}, b_0^{(k)}\bigr)\Bigr) \\
& \qquad\qquad\qquad\qquad\quad
              \Bigm| u \in \partial_\theta
                           L\bigl(\theta^{(k)}, a_0^{(k)}, b_0^{(k)}\bigr)\Bigr\}
           + \frac{\lambda_k}{2}
             \bigl\{\bigl(\nabla_\theta \lVert \theta^{(k)} \rVert_2^2,
                          0, 0\bigr)\bigr\}\biggr) \\
& =    \zeta_k \mathbb{B}_2^{3 m + 2}
  \cap \Bigl\{\Bigl(u
                  + \frac{\lambda_k}{2}
                    \nabla_\theta
                    \lVert \theta^{(k)} \rVert_2^2,
                    \nabla_{a_0}
                    L\bigl(\theta^{(k)}, a_0^{(k)}, b_0^{(k)}\bigr),
                    \nabla_{b_0}
                    L\bigl(\theta^{(k)}, a_0^{(k)}, b_0^{(k)}\bigr)\Bigr) \\
& \qquad\qquad\qquad\qquad
       \Bigm| u \in \partial_\theta L(\theta^{(k)}, a_0^{(k)}, b_0^{(k)})\Bigr\},
\end{align}
so for some decomposition $\zeta_k^2 = (\zeta'_k)^2 + (\zeta''_k)^2$ we have
\begin{align}
\zeta'_k \mathbb{B}_2^2
& \ni \nabla_{a_0, b_0} L\bigl(\theta^{(k)}, a_0^{(k)}, b_0^{(k)}\bigr) \\
& =   \frac{1}{n}
      \sum_{i = 1}^n
      \nabla_{a_0, b_0}
      \ell\bigl(y_i \, f_{\theta^{(k)}, a_0^{(k)}, b_0^{(k)}}(x_i)\bigr) \\
& =   -\frac{1}{n}
       \sum_{i = 1}^n
       \frac{\nabla_{a_0, b_0}
             \bigl(y_i \, f_{\theta^{(k)}, a_0^{(k)}, b_0^{(k)}}(x_i)\bigr)}
            {1 + \exp\bigl(y_i \,
                           f_{\theta^{(k)}, a_0^{(k)}, b_0^{(k)}}(x_i)\bigr)} \\
& =   -\frac{1}{n}
       \sum_{i = 1}^n
       \frac{\nabla_{a_0, b_0}
             \Bigl(y_i \,
               f_{\widetilde{\theta^{(k)}, a_0^{(k)}, b_0^{(k)}}}(x_i)\Bigr)}
            {1 + \exp\bigl(y_i \, f_{\theta^{(k)}, a_0^{(k)}, b_0^{(k)}}(x_i)\bigr)}
\label{eq:zetap}
\end{align}
and also
\begin{align}
\emptyset
& \neq \zeta''_k \mathbb{B}_2^{3 m}
  \cap \Bigl(\partial_\theta L\bigl(\theta^{(k)}, a_0^{(k)}, b_0^{(k)}\bigr)
           + \frac{\lambda_k}{2}
             \{\nabla_\theta \lVert \theta^{(k)} \rVert_2^2\}\Bigr) \\
& =    \zeta''_k \mathbb{B}_2^{3 m}
  \cap \biggl(\frac{1}{n}
              \sum_{i = 1}^n
              \partial_\theta
              \ell\bigl(y_i \, f_{\theta^{(k)}, a_0^{(k)}, b_0^{(k)}}(x_i)\bigr)
            + \frac{\lambda_k}{2}
              \{\nabla_\theta \lVert \theta^{(k)} \rVert_2^2\}\biggr) \\
& =    \zeta''_k \mathbb{B}_2^{3 m}
  \cap \biggl(-\frac{1}{n}
               \sum_{i = 1}^n
               \frac{\partial_\theta
                     \bigl(y_i \, f_{\theta^{(k)}, a_0^{(k)}, b_0^{(k)}}(x_i)\bigr)}
                    {1 + \exp\bigl(y_i \,
                                   f_{\theta^{(k)}, a_0^{(k)}, b_0^{(k)}}(x_i)\bigr)}
            + \frac{\lambda_k}{2}
              \{\nabla_\theta \lVert \theta^{(k)} \rVert_2^2\}\biggr) \\
& =    \zeta''_k \mathbb{B}_2^{3 m}
  \cap \sqrt{M\bigl(\theta^{(k)}, a_0^{(k)}, b_0^{(k)}\bigr)}
       \Biggl(-\frac{1}{n}
               \sum_{i = 1}^n
               \frac{\partial_\theta
                     \Bigl(y_i \,
                       f_{\widetilde{\theta^{(k)}, a_0^{(k)}, b_0^{(k)}}}(x_i)\Bigr)}
                    {1 + \exp\bigl(y_i \,
                               f_{\theta^{(k)}, a_0^{(k)}, b_0^{(k)}}(x_i)\bigr)} \\
& \qquad\qquad\qquad\qquad\qquad\qquad\qquad\quad
            + \frac{\lambda_k}{2}
              \Biggl\{\nabla_\theta
                      \Biggl\lVert
                      \frac{\theta^{(k)}}
                           {\sqrt{M\bigl(\theta^{(k)}, a_0^{(k)}, b_0^{(k)}\bigr)}}
                      \Biggr\rVert_2^2\Biggr\}\Biggr),
\end{align}
which once we recall that $M\bigl(\theta^{(k)}, a_0^{(k)}, b_0^{(k)}\bigr) \geq 1$ gives us
\begin{equation}
\emptyset
\neq \zeta''_k \mathbb{B}_2^{3 m}
\cap \Biggl(-\frac{1}{n}
             \sum_{i = 1}^n
             \frac{\partial_\theta
                   \Bigl(y_i \,
                     f_{\widetilde{\theta^{(k)}, a_0^{(k)}, b_0^{(k)}}}(x_i)\Bigr)}
                  {1 + \exp\bigl(y_i \,
                             f_{\theta^{(k)}, a_0^{(k)}, b_0^{(k)}}(x_i)\bigr)}
          + \frac{\lambda_k}{2}
            \Biggl\{\nabla_\theta
                    \Biggl\lVert
                    \frac{\theta^{(k)}}
                         {\sqrt{M\bigl(\theta^{(k)}, a_0^{(k)}, b_0^{(k)}\bigr)}}
                    \Biggr\rVert_2^2\Biggr\}\Biggr).
\label{eq:zetapp}
\end{equation}
Now putting \cref{eq:zetap,eq:zetapp} together, and defining
\begin{align}
\varepsilon_k & \coloneqq
\frac{\zeta_k}{\lambda_k} &
\mu_i^{(k)}   & \coloneqq
\frac{1}{n \lambda_k
         \Bigl(1 + \exp\bigl(y_i \,
                             f_{\theta^{(k)}, a_0^{(k)}, b_0^{(k)}}(x_i)\bigr)\Bigr)}
\quad\text{for all } i \in [n],
\end{align}
we obtain
\begin{align}
\emptyset
& \neq \varepsilon_k \mathbb{B}_2^{3 m + 2}
  \cap \Biggl(\frac{1}{2}
              \Biggl\{\nabla
                      \Biggl\lVert
                      \frac{\theta^{(k)}}
                           {\sqrt{M\bigl(\theta^{(k)}, a_0^{(k)}, b_0^{(k)}\bigr)}}
                      \Biggr\rVert_2^2\Biggr\}
            - \sum_{i = 1}^n
              \mu_i^{(k)}
              \partial
              \Bigl(y_i \,
                f_{\widetilde{\theta^{(k)}, a_0^{(k)}, b_0^{(k)}}}(x_i)\Bigr)\Biggr),
\label{eq:appr.eq.incl}
\end{align}
so moreover defining
\[\delta_k \coloneqq
  \sum_{i = 1}^n
  \mu_i^{(k)}
  \Bigl(y_i \, f_{\widetilde{\theta^{(k)}, a_0^{(k)}, b_0^{(k)}}}(x_i) - 1\Bigr),\]
we have that $\widetilde{\theta^{(k)}, a_0^{(k)}, b_0^{(k)}}$ is an $\varepsilon_k, \delta_k$-KKT point of \cref{eq:Pb1}.%
\footnote{For the definition of approximate KKT points, we refer the reader to \cref{s:Clarke.KKT.MFCQ}.}

Recalling that $\varepsilon_k \to 0$ as $k \to \infty$, to conclude by \cref{th:appr.KKT} that $\theta^\star, a_0^\star, b_0^\star$ is a KKT point of \cref{eq:Pb1}, it remains to show that $\delta_k \to 0$ as $k \to \infty$.

For each $k \in \mathbb{N}$, reasoning like in the proof of \cref{l:feas.MFCQ}, from \cref{eq:appr.eq.incl} and by \cref{pr:margin} we have
\begin{align}
\sum_{i = 1}^n
\mu_i^{(k)}
& \leq \sum_{i = 1}^n
       \mu_i^{(k)}
       y_i \, f_{\widetilde{\theta^{(k)}, a_0^{(k)}, b_0^{(k)}}}(x_i) \\
& \leq \frac{1}{2}
       \Biggl\lVert
       \frac{\theta^{(k)}}
            {\sqrt{M\bigl(\theta^{(k)}, a_0^{(k)}, b_0^{(k)}\bigr)}}
       \Biggr\rVert_2^2 \\
& \quad
     + \varepsilon_k
       \Biggl\lVert
       \frac{\theta^{(k)}}
            {2 \sqrt{M\bigl(\theta^{(k)}, a_0^{(k)}, b_0^{(k)}\bigr)}},
       \frac{a_0^{(k)}}
            {M\bigl(\theta^{(k)}, a_0^{(k)}, b_0^{(k)}\bigr)},
       \frac{b_0^{(k)}}
            {M\bigl(\theta^{(k)}, a_0^{(k)}, b_0^{(k)}\bigr)}
       \Biggr\rVert_2,
\end{align}
so since $\bigl(\widetilde{\theta^{(k)}, a_0^{(k)}, b_0^{(k)}}\bigr) \to (\theta^\star, a_0^\star, b_0^\star)$ as $k \to \infty$, we have that $\sup_{k \in \mathbb{N}} \mu_i^{(k)} < \infty$ for each $i \in [n]$.  Hence, since also $\lambda_k \to 0$ as $k \to \infty$, we infer that $M\bigl(\theta^{(k)}, a_0^{(k)}, b_0^{(k)}\bigr) \to \infty$ as $k \to \infty$.

Now for all $i, i' \in [n]$ such that $y_i \, f_{\theta^\star, a_0^\star, b_0^\star}(x_i) < y_{i'} \, f_{\theta^\star, a_0^\star, b_0^\star}(x_{i'})$, and all $k \in \mathbb{N}$, we have
\begin{align}
\frac{\mu_{i'}^{(k)}}
     {\mu_i^{(k)}}
& = \frac{1 + \exp\bigl(y_i \,
                    f_{\theta^{(k)}, a_0^{(k)}, b_0^{(k)}}(x_i)\bigr)}
         {1 + \exp\bigl(y_{i'} \,
                    f_{\theta^{(k)}, a_0^{(k)}, b_0^{(k)}}(x_{i'})\bigr)} \\
& < 2 \exp\bigl(y_i \,    f_{\theta^{(k)}, a_0^{(k)}, b_0^{(k)}}(x_i)
              - y_{i'} \, f_{\theta^{(k)}, a_0^{(k)}, b_0^{(k)}}(x_{i'})\bigr) \\
& = 2 \exp\biggl(
          \Bigl(y_i \,
                f_{\widetilde{\theta^{(k)}, a_0^{(k)}, b_0^{(k)}}}(x_i)
              - y_{i'} \,
                f_{\widetilde{\theta^{(k)}, a_0^{(k)}, b_0^{(k)}}}(x_{i'})\Bigr)
          M\bigl(\theta^{(k)}, a_0^{(k)}, b_0^{(k)}\bigr)
          \biggr),
\end{align}
and so $\mu_{i'}^{(k)} \to 0$ as $k \to \infty$.

Therefore, for all $i \in [n]$, we have either $y_i \, f_{\theta^\star, a_0^\star, b_0^\star}(x_i) = 1$ in which case $y_i \, f_{\widetilde{\theta^{(k)}, a_0^{(k)}, b_0^{(k)}}}(x_i) \to 1$ as $k \to \infty$ (and there is necessarily at least one such~$i$), or $y_i \, f_{\theta^\star, a_0^\star, b_0^\star}(x_i) > 1$ in which case $\mu_i^{(k)} \to 0$ as $k \to \infty$.  Hence $\delta_k \to 0$ as $k \to \infty$, as required.

For part~\ref{th:lim:it}, suppose for a contradiction that $f_{\theta^\star, a_0^\star, b_0^\star}$~is not interval turning.  Since $\theta^\star, a_0^\star, b_0^\star$ is a KKT point of \cref{eq:P1}, we have by \cref{th:inter}~\ref{th:inter:KKT} and~\ref{th:inter:P.P1} that $f_{\theta^\star, a_0^\star, b_0^\star}$~is convexity correct.  Hence, by continuity and reasoning as in the proofs of \cref{l:it,l:L1}, there exists $\eta > 0$ such that, for all sufficiently large~$k$, there exists a direction $(\theta', a'_0, b'_0) \in \mathbb{B}_2^{3 m + 2}$ such that:
\begin{itemize}
\item
$\mathrm{D}
 \Bigl(\frac{1}{2}
       \sum_{j = 1}^m
       \bigl((a^{(k)}_j)^2 + (w^{(k)}_j)^2\bigr)\Bigr)
 [\theta', a'_0, b'_0]
 \leq 0$;
\item
$\mathrm{D}^\circ
 \bigl(-y_i \,
        f_{\theta^{(k)}, a_0^{(k)}, b_0^{(k)}}
        (x_i)\bigr)
 [\theta', a'_0, b'_0]
 \leq 0$
for all $i \in [n]$;
\item
$\mathrm{D}
 \bigl(-y_{i^\dag} \,
        f_{\theta^{(k)}, a_0^{(k)}, b_0^{(k)}}
        (x_{i^\dag})\bigr)
 [\theta', a'_0, b'_0]
 \leq -\eta \sqrt{M\bigl(\theta^{(k)}, a_0^{(k)}, b_0^{(k)}\bigr)}$
for some $i^\dag \in [n]$.
\end{itemize}
Then, we have
\begin{align}
-\frac{\zeta_k}
      {\lambda_k^F
       \sqrt{M\bigl(\theta^{(k)}, a_0^{(k)}, b_0^{(k)}\bigr)}}
& \leq \frac{\mathrm{D}^\circ
             L\mathrm{^{\lambda_k}_1}
             \bigl(\theta^{(k)}, a_0^{(k)}, b_0^{(k)}\bigr)
             [\theta', a'_0, b'_0]}
            {\lambda_k^F
             \sqrt{M\bigl(\theta^{(k)}, a_0^{(k)}, b_0^{(k)}\bigr)}} \\
& \leq -\frac{\eta}
             {n
              \lambda_k^F
              \Bigl(1 + \exp\bigl(y_{i^\dag} \,
                                  f_{\theta^{(k)}, a_0^{(k)}, b_0^{(k)}}
                                  (x_{i^\dag})\bigr)\Bigr)},
\end{align}
and so
$\lambda_k^F
 \Bigl(1 + \exp\bigl(y_{i^\dag} \,
                     f_{\theta^{(k)}, a_0^{(k)}, b_0^{(k)}}
                     (x_{i^\dag})\bigr)\Bigr)
 \to \infty$
as $k \to \infty$.

However, for each $i \in [n]$ with $y_i \, f_{\theta^\star, a_0^\star, b_0^\star}(x_i) = 1$, we have
\begin{align}
& \bigl(n \mu_i^{(k)}\bigr)^F
  \lambda_k^F
  \Bigl(1 + \exp\bigl(y_{i^\dag} \,
                      f_{\theta^{(k)}, a_0^{(k)}, b_0^{(k)}}
                      (x_{i^\dag})\bigr)\Bigr) \\
& =   \frac{1 + \exp\bigl(y_{i^\dag} \,
                          f_{\theta^{(k)}, a_0^{(k)}, b_0^{(k)}}
                          (x_{i^\dag})\bigr)}
           {\Bigl(1 + \exp\bigl(y_i \,
                                f_{\theta^{(k)}, a_0^{(k)}, b_0^{(k)}}
                                (x_i)\bigr)\Bigr)^F} \\
& <   2 \frac{\exp\bigl(y_{i^\dag} \,
                        f_{\theta^{(k)}, a_0^{(k)}, b_0^{(k)}}
                        (x_{i^\dag})\bigr)}
             {\exp\bigl(y_i \,
                        f_{\theta^{(k)}, a_0^{(k)}, b_0^{(k)}}
                        (x_i)\bigr)
              ^F} \\
& =   2 \exp\biggl(\Bigl(y_{i^\dag} \,
                         f_{\widetilde{\theta^{(k)}, a_0^{(k)}, b_0^{(k)}}}
                         (x_{i^\dag})
                  - F \, y_i \,
                         f_{\widetilde{\theta^{(k)}, a_0^{(k)}, b_0^{(k)}}}
                         (x_i)\Bigr)
                   M\bigl(\theta^{(k)}, a_0^{(k)}, b_0^{(k)}\bigr)\biggr) \\
& \to 0 \quad\text{as } k \to \infty,
\end{align}
and therefore $\mu_i^{(k)} \to 0$ as $k \to \infty$.  But then $\mu_i^{(k)} \to 0$ as $k \to \infty$ for all $i \in [n]$, which is impossible due to the analogue of \cref{eq:appr.eq.incl} with biases not penalized.
\end{proof}

\clearpage
\ifopt
\else
\section{Further details for and outcomes from the experiments}
\fi
\label{s:exp:app}

\fregloss{\ifopt\else (This is \Cref{f:reg.loss:main}, reproduced here for the reader's convenience.)  \fi}{f:reg.loss}

\newcommand{\SparsityClustering}{medium}

\pgfplotsset{
  sparsity symlog y/.style={
    y coord trafo/.code={\pgfmathparse{sign(##1)*(
      min(abs(##1)/\SparsitySymlogThreshold,1)
      +ln(max(abs(##1)/\SparsitySymlogThreshold,1)))}},
    y coord inv trafo/.code={\pgfmathparse{sign(##1)*\SparsitySymlogThreshold*(
      min(abs(##1),1)+max(exp(abs(##1)-1)-1,0))}},
    ytick={-10,-1,-0.1,0,0.1,1,10},
    yticklabels={{$-10^1$},{$-10^0$},{$-10^{-1}$},{$0$},
                 {$10^{-1}$},{$10^0$},{$10^1$}},
    minor ytick={-9,-8,-7,-6,-5,-4,-3,-2,
      -0.9,-0.8,-0.7,-0.6,-0.5,-0.4,-0.3,-0.2,
       0.2,0.3,0.4,0.5,0.6,0.7,0.8,0.9,
       2,3,4,5,6,7,8,9},
  },
}

\definecolor{clrAW}{HTML}{1B7837}
\definecolor{clrAWB}{HTML}{762A83}


\newcommand{\adamLegend}{%
  \begingroup\small\hspace{4.9em}
  \tikz[baseline=-0.55ex]{\draw[clrAW,solid,line width=0.8pt]
    plot[mark=*,mark size=1.35pt] coordinates {(0,0) (0.65,0)};}
  \ biases not regularized \hspace{2.7em}
  \tikz[baseline=-0.55ex]{\draw[clrAWB,dashed,line width=0.8pt]
    plot[mark=square*,mark size=1.25pt] coordinates {(0,0) (0.65,0)};}
  \ biases regularized
  \endgroup
}

\newcommand{\sparsityPanel}[5]{%
  \begin{tikzpicture}
    \begin{axis}[
      experiment axis, sparsity symlog y, scaled y ticks=false,
      ylabel={\shortstack{kink count (relative excess)\\skip connection #4}},
      title={\twoEquationTitle{#5}},
      /pgfplots/shared bounds/sp/\SparsityClustering/.try,
      /pgfplots/manual bounds/sp/\SparsityClustering/.try,
      #3,
    ]
      \rangeSeries{plot_np_data/panels/sp_\SparsityClustering_#1_F#2.csv}
    {relative_excess_kink_cluster_count_\SparsityClustering_aw}{clrAW}{solid,mark=*,mark size=1.35pt}{sp-#1-#2-aw}
      \rangeSeries{plot_np_data/panels/sp_\SparsityClustering_#1_F#2.csv}
    {relative_excess_kink_cluster_count_\SparsityClustering_awb}{clrAWB}{dashed,mark=square*,mark size=1.25pt}{sp-#1-#2-awb}
      \addplot[black!55,densely dotted,mark=none,forget plot]
        coordinates {(4,0) (12,0)};
    \end{axis}
  \end{tikzpicture}%
}

\newcommand{\unsuccessfulPanel}[3]{%
  \begin{tikzpicture}
    \begin{axis}[
      unbounded coords=discard,filter discard warning=false,
      grid=major,major grid style={draw=black!12,line width=0.2pt},
      axis line style={black!65},tick style={black!65},
      xticklabel style={font=\footnotesize},yticklabel style={font=\footnotesize},
      xlabel style={font=\footnotesize},ylabel style={font=\footnotesize},
      title style={font=\footnotesize,yshift=-1pt},
      every axis plot/.append style={line width=0.75pt},clip mode=individual,
      width=.39\textwidth,height=.2\textheight,
      xlabel={overparameterization factor},ylabel={\shortstack{positive margin\\failure rate}},
      xtick={1,2,3,4,5},xmin=.75,xmax=5.25,scaled y ticks=false,
      /pgfplots/shared bounds/unsuccessful/all/.try,
      /pgfplots/manual bounds/unsuccessful/all/.try,title={#2},#3,
    ]
      \addplot[clrNoskip,solid,mark=*,mark size=1.35pt]
        table[col sep=comma,x=width_factor,y=proportion_unsuccessful_noskip]
        {plot_np_data/panels/unsuccessful_#1.csv};
      \addplot[clrSkip,dashed,mark=square*,mark size=1.25pt]
        table[col sep=comma,x=width_factor,y=proportion_unsuccessful_skip]
        {plot_np_data/panels/unsuccessful_#1.csv};
    \end{axis}
  \end{tikzpicture}%
}

\begin{figure}[t]
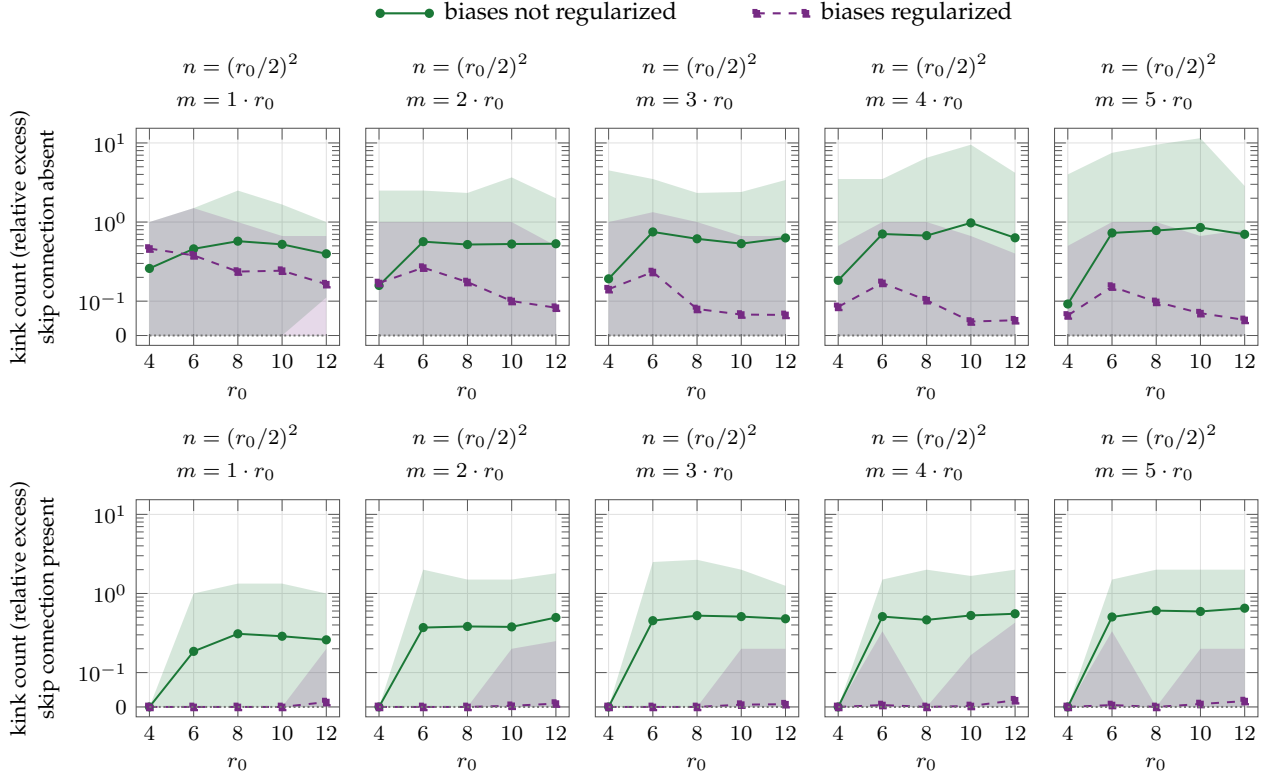

  \centering\adamLegend\par\medskip
  \setlength{\tabcolsep}{0.2pt}\renewcommand{\arraystretch}{0.90}
  \begin{tabular}{@{}ccccc@{}}
    \sparsityPanel{skip0}{1}{}{absent}{1}&
    \sparsityPanel{skip0}{2}{hide row y axis}{absent}{2}&
    \sparsityPanel{skip0}{3}{hide row y axis}{absent}{3}&
    \sparsityPanel{skip0}{4}{hide row y axis}{absent}{4}&
    \sparsityPanel{skip0}{5}{hide row y axis}{absent}{5}\\
    \sparsityPanel{skip1}{1}{}{present}{1}&
    \sparsityPanel{skip1}{2}{hide row y axis}{present}{2}&
    \sparsityPanel{skip1}{3}{hide row y axis}{present}{3}&
    \sparsityPanel{skip1}{4}{hide row y axis}{present}{4}&
    \sparsityPanel{skip1}{5}{hide row y axis}{present}{5}
  \end{tabular}
  \caption{We plot the relative excess number of distinct kinks of the network at the end of training, i.e., $Q / (r - 1) - 1$, where $Q$~is the number of distinct kinks and $r$~is the number of label switches in the training dataset (and so $r - 1$ is the minimal number of distinct kinks of any positive-margin classifier).  Rows indicate whether the skip connection is absent or present, columns give the overparameterization factor $\omega = m / r_0$, and the two curves indicate whether biases are regularized.  For each point, we use the runs, among 60 RNG seeds, that reached a positive margin: the curve gives their mean, and the shaded region spans their smallest and largest outcomes (positive-margin failure rates are reported in \Cref{f:inter.fail}).  The shared vertical axis is logarithmic above~$10^{-1}$ and linear below it.  \textbf{Observations:}  When biases are regularized, training consistently reaches or closely approaches the minimal kink count, especially with a skip connection; without bias regularization, the final networks typically have more kinks, particularly when the skip connection is absent. This is consistent with our theoretical result that bias regularization makes every (global) minimizer sparsest, and that a skip connection rules out suboptimal positive-margin stationary points, as well as with the counterexample showing that nonsparse suboptimal KKT points can occur without a skip connection.}
  \label{f:sparsity}
\end{figure}

\begin{figure}[t]
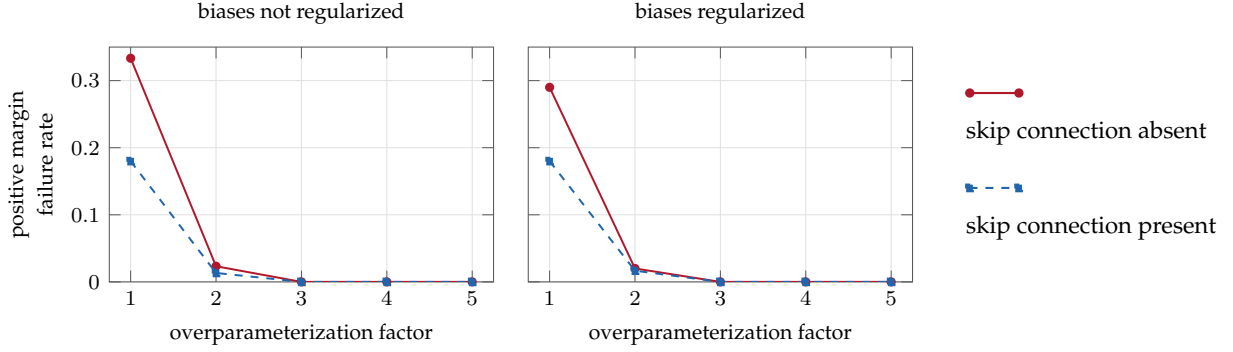

  \centering
  \setlength{\tabcolsep}{3pt}
  \begin{tabular}{@{}cc@{\hspace{2em}}l@{}}
    \unsuccessfulPanel{adam_weights}{biases not regularized}{}&
    \unsuccessfulPanel{adam_weights_biases}{biases regularized}{hide row y axis}&
    \raisebox{15ex}{\begin{tabular}[c]{@{}l@{}}
      \tikz[baseline=-0.55ex]{\draw[clrNoskip,solid,line width=.8pt]
        plot[mark=*,mark size=1.35pt] coordinates {(0,0) (.65,0)};}\\[.5ex]
      {\small skip connection absent}\\[2ex]
      \tikz[baseline=-0.55ex]{\draw[clrSkip,dashed,line width=.8pt]
        plot[mark=square*,mark size=1.25pt] coordinates {(0,0) (.65,0)};}\\[.5ex]
      {\small skip connection present}
    \end{tabular}}
  \end{tabular}
  \caption{We plot the fraction of runs that failed to reach a positive margin on their training dataset, aggregated over 300 runs for each setting, comprising 60 RNG seeds for each of the five numbers of class breaks $r_0 \in \{4, 6, 8, 10, 12\}$.  The plots have the same linear vertical axis.  \textbf{Observations:}  The fractions of failed runs are always (if nonzero) higher when a skip connection is absent, and are significant only for the smallest overparameterization factor of~$1$, which is when the network width is only an additive constant greater than the minimum required for a positive margin.}
  \label{f:inter.fail}
\end{figure}

\paragraph{Datasets.}

We first sample a data distribution, which has a hyperparameter $r_0 \geq 3$, which is the number of class breaks and which we vary in the experiments.  Namely, we sample class breaks $(\beth_k)_{k = 1}^{r_0}$ from the uniform distribution on the interval $(-1, 1)$, and index them so that they are increasing, i.e., $\beth_1 < \dots < \beth_{r_0}$.  Also we write $\beth_0 \coloneqq -1$ and $\beth_{r_0 + 1} \coloneqq 1$.  We then sample the training inputs $(x_i)_{i = 1}^n$ from the uniform distribution on the interval $(-1, 1)$, and assign the labels according to the class breaks, i.e., for each $i \in [n]$, we set $y_i \coloneqq (-1)^k$, where $x_i \in [\beth_k, \beth_{k + 1})$.

Since the training inputs are random, the number~$r$ of label switches in the dataset is at most $r_0$ and possibly smaller.  We set the dataset size as $n \coloneqq (r_0 / 2)^2$, which ensures that, in expectation, $r$~is only a constant below the hyperparameter~$r_0$.

We use rejection sampling to ensure that all training inputs in the first (resp., last) two same-label segments are negative (resp., positive), i.e., that \cref{ass:data} is satisfied; also to ensure that the minimal value of $\beth_{k + 1} - \beth_k$ is at least $1 / (r_0)^2$, so that no class is very unlikely to be populated by a training input; and to ensure that the minimal value of $x_{i + 1} - x_i$ when $y_{i + 1} \neq y_i$ (i.e., at label switches) is at least $2 / n \sqrt{r_0}$, so that reaching a positive margin on the dataset will not require very large network parameters.\footnote{These two lower bounds are multiples of the corresponding asymptotic expectations, by small constants.}

\paragraph{Networks.}

We set the network width as $m \coloneqq \omega \cdot r_0$, where the overparameterization factor~$\omega$ is a hyperparameter which we vary in the experiments.

Since the focus of this work is norm minimization of interpolating networks, we initialize the network in a way that helps to reach a positive margin early in the training.  Namely, following \citet{SafranVL22}, we initialize the output weights small and the hidden weights large: we set~$d_{\min}$ to be the minimal distance between two training inputs with different labels, set~$\varsigma$ to be a constant times $1 / d_{\min}$, sample the output weights $(a_j)_{j = 1}^m$ from the centered normal distribution with variance~$1 / \varsigma^2$, and sample the hidden weights $(w_j)_{j = 1}^m$ from the centered normal distribution with variance~$\varsigma^2$.  Then we initialize the hidden biases $(b_j)_{j = 1}^m$ so that, for each $j \in [m]$, the kink $-b_j / w_j$ is equally likely to be in any of the $n - 1$ intervals $(x_i, x_{i + 1})$ between two consecutive training inputs, and is independently uniformly distributed on that interval.

If a skip connection is present, we initialize its weight~$a_0$ and its bias~$b_0$ both as zero.

\paragraph{Training.}

Depending on whether biases are regularized and whether a skip connection is present, we minimize one of the four variants \hyperref[eq:L]{$L\mathrm{^\lambda}$}, \hyperref[eq:L1]{$L\mathrm{^\lambda_1}$}, \hyperref[eq:Lb]{$L\mathrm{^{\lambda, b}}$}, and~\hyperref[eq:Lb1]{$L\mathrm{^{\lambda, b}_1}$} of the empirical logistic loss with $\ell_2$~regularization of strength~$\lambda$.  We use small~$\lambda$, namely set it to a constant times $1 / m \varsigma^2$, so that before a positive margin is reached, the objective is dominated by the logistic loss term; this also ensured that, in all the experiments we ran, $\lambda$~satisfied the upper bounds in \cref{l:O1.Ob1,th:min.reg.loss}.

We use an in-house implementation in NumPy~\citep{harris2020array} of the Adam optimizer \citep{KingmaB14}.  We leave all of Adam's hyperparameters at their default values, except that before reaching a positive margin the learning rate is $10$~times smaller, and after that $1 - \beta_2$ is $10$~times larger.  To get close to a Clarke stationary point, we run $2 \cdot 10^7$ full-batch iterations per experiment.

\paragraph{Final computations.}

For each experiment, to compare its final value of the regularized loss with the minimum value, we recall that, by \cref{th:min.reg.loss}, the latter equals the minimum of \cref{eq:O1} or \cref{eq:Ob1} (depending on whether biases are regularized) for the same dataset and the same regularization strength~$\lambda$.  Recalling also \cref{l:O1.Ob1}, \Cref{eq:O1,eq:Ob1} have strictly convex objectives on convex feasible sets, and we compute their minima using Clarabel~\citep{GoulartC26} (via CVXPY~\citep{cvxpy16} and using the default settings therein). All runs of Clarabel involved in the experiments to produce the plots in \Cref{f:reg.loss} returned near optimal solutions.  Also we remark that applying a general-purpose convex reformulation~\citep[cf.][]{pilanci2020neural} would provide an alternative computation of the same benchmark rather than an additional learning baseline.

For the sparsity plots in \Cref{f:sparsity}, we compute the number of distinct kinks as the number of clusters of values $-b_j / w_j$. We use SciPy's~\citep{2020SciPy-NMeth} scipy.cluster.hierarchy.fclusterdata function with the ``distance'' criterion and the ``complete'' method with a threshold of $0.02$~times the minimal distance~$d_{\min}$ between two training inputs with different labels. We do not count clusters whose sum of slope changes $a_j \lvert w_j \rvert$ is in the interval $[-0.01,0.01]$.

\paragraph{Hyperparameters and seeds.}

We run experiments for each value combination of the following four hyperparameters (the first two numerical and the last two Boolean): number of class breaks $r_0 \in \{4, 6, 8, 10, 12\}$, overparameterization factor $\omega \in \{1, 2, 3, 4, 5\}$, whether biases are regularized, and whether a skip connection is present.  For each of those 100 combinations, we repeat everything (class breaks sampling, dataset sampling, network initialization, network training, loss baseline computation, and sparsity computation) 60 times for different RNG seeds.%
\footnote{Class breaks sampling and dataset sampling depend only on~$r_0$ and the seed.  Network initialization apart from the skip connection depends additionally on~$\omega$.}
Only runs that succeed in reaching a positive margin contribute to the plots in \Cref{f:reg.loss,f:sparsity}; the fractions of failed runs are plotted in \Cref{f:inter.fail}.

\paragraph{Compute.}

Single experiments (i.e., with all four hyperparameters fixed and for one RNG seed) took under~20min on a single core of an AMD EPYC 9555P 64-core CPU, and all 6000 experiments that we ran to produce the plots in \Cref{f:reg.loss,,f:sparsity,,f:inter.fail} took around~12h using two such CPUs.

\end{document}